\documentclass[10pt,twocolumn,letterpaper]{article}

\usepackage[pagenumbers]{cvpr} % To force page numbers, e.g. for an arXiv version

\usepackage{graphicx}
\usepackage{amsmath}
\usepackage{amssymb}
\usepackage{booktabs}
\usepackage{array}
\usepackage[accsupp]{axessibility}

\usepackage[pagebackref,breaklinks,colorlinks]{hyperref}

\usepackage{amsthm}
\newtheorem{innercustomgeneric}{\customgenericname}
\providecommand{\customgenericname}{}
\newcommand{\newcustomtheorem}[2]{%
  \newenvironment{#1}[1]
  {%
   \renewcommand\customgenericname{#2}%
   \renewcommand\theinnercustomgeneric{##1}%
   \innercustomgeneric
  }
  {\endinnercustomgeneric}
}
\newcustomtheorem{customthm}{Theorem}
\newcustomtheorem{customprop}{Proposition}

\usepackage[capitalize]{cleveref}
\crefname{section}{Sec.}{Secs.}
\Crefname{section}{Section}{Sections}
\Crefname{table}{Table}{Tables}
\crefname{table}{Tab.}{Tabs.}

\def\cvprPaperID{7362} % *** Enter the CVPR Paper ID here
\def\confName{CVPR}
\def\confYear{2022}

\usepackage{amsthm} % For proof environments. Needs to be loaded before sg-ibs-macros due to redefinition of theorem.
\usepackage{sg-ibs-macros}
\usepackage{overpic}
\usepackage{enumitem} %< control spacing in itemize/enumerate/...
\usepackage{overpic} %< add raw math symbols to figures
\usepackage{color}
\usepackage{tabularx}
\usepackage{graphicx}
\usepackage{makecell}
\usepackage{multirow}
\usepackage{amssymb}

\usepackage{multicol}

\newcolumntype{M}[1]{>{\centering\arraybackslash}m{#1}}

\newcommand\recviscol{0.14}
\newcommand\recviscoltwod{0.13}
\definecolor{turquoise}{cmyk}{0.65,0,0.1,0.3}
\definecolor{purple}{rgb}{0.65,0,0.65}
\definecolor{dark_green}{rgb}{0, 0.5, 0}
\definecolor{orange}{rgb}{0.8, 0.6, 0.2}
\definecolor{red}{rgb}{0.8, 0.2, 0.2}
\definecolor{darkred}{rgb}{0.6, 0.1, 0.05}
\definecolor{blueish}{rgb}{0.0, 0.3, .6}
\definecolor{light_gray}{rgb}{0.7, 0.7, .7}
\definecolor{pink}{rgb}{1, 0, 1}
\definecolor{greyblue}{rgb}{0.25, 0.25, 1}

\usepackage{blindtext}

\renewcommand{\paragraph}[1]{\vspace{1em}\noindent\textbf{#1}.}

\begin{document}

\title{DiGS : Divergence guided shape implicit neural representation for unoriented point clouds}

% \author{
%   Yizhak Ben-Shabat \thanks{www.itzikbs.com} \\
%   The Australian National University\\
%   Technion Israel Institute of Technology\\
%     \texttt{yizhak.benshabat@anu.edu.au} \\
%     \and
%   Chamin Hewa Koneputugodage*  \\
%   The Australian National University\\
%     \texttt{chamin.hewa@anu.edu.au} \\
%   \and
%   Stephen Gould \\
%   The Australian National University\\
%     \texttt{stephen.gould@anu.edu.au} \\
% }

\author{
  Yizhak Ben-Shabat$^{1,2,3}$\thanks{Equal contribution} \qquad  Chamin Hewa Koneputugodage$^{2,3}$\footnotemark[1] \qquad Stephen Gould$^{2,3}$\vspace{3mm}\\
  { $^{1}$Technion Israel Institute of Technology\quad $^{2}$The Australian National University \quad}\\
  { $^{3}$Australian Centre for Robotic Vision}\\
  {\tt \small \{yizhak.benshabat, chamin.hewa, stephen.gould\}@anu.edu.au}
}

\maketitle
\begin{abstract}
Shape implicit neural representations (INRs) have recently shown to be effective in shape analysis and reconstruction tasks. Existing INRs require point coordinates to learn the implicit level sets of the shape. When a normal vector is available for each point, a higher fidelity representation can be learned, however normal vectors are often not provided as raw data.
Furthermore, the method's initialization has been shown to play a crucial role for surface reconstruction. 
In this paper, we propose a divergence guided shape representation learning approach that does not require normal vectors as input. We show that incorporating a soft constraint on the divergence of the distance function favours smooth solutions that reliably orients gradients to match the unknown normal at each point, in some cases even better than approaches that use ground truth normal vectors directly. Additionally, we introduce a novel geometric initialization method for sinusoidal INRs that further improves convergence to the desired solution.
We evaluate the effectiveness of our approach on the task of surface reconstruction and shape space learning and show SOTA performance compared to other unoriented methods.\\
Code and model parameters available at our project page \url{https://chumbyte.github.io/DiGS-Site/}.
\end{abstract}

\section{Introduction}
\label{Sec:intro}

Reconstructing surfaces from 3D point samples is a well studied problem in computer vision and computer graphics. Recently, neural networks have been used to learn an implicit neural representation (INR) that can be used to reconstruct the underlying surface \cite{park2019deepsdf, mescheder2019occupancy, peng2020convolutional_occ, chen2019learning, atzmon2020sal, Atzmon_2020_CVPR_SALD, urbach2020dpdist,gropp2020igr, sitzmann2020siren}, which we refer to as a shape INR. These methods are often supervised using estimated or known volumetric implicit representations in a regression setting \cite{mescheder2019occupancy, park2019deepsdf, chen2019learning} or using surface 3D points with or without extra 3D supervision (\eg normal data) \cite{atzmon2020sal, Atzmon_2020_CVPR_SALD, gropp2020igr, sitzmann2020siren} to regress the function directly. Due to the difficulty of optimizing a regression model to fit high fidelity surfaces, most methods use normal data. However, raw 3D point clouds from scans are typically unoriented, and thus do not have normal vectors, therefore a pre-processing estimation stage is required. While while some methods allow normal supervision to be absent~\cite{gropp2020igr,sitzmann2020siren}, we show that without it their performance drops significantly. While there have been significant advances in normal estimation algorithms \cite{guerrero2018pcpnet, ben2019nestinet, ben2020deepfit, lenssen2020deepiterative}, all yield unoriented and noisy predictions. This poses a great challenge for existing normal-based shape INR learning approaches. 

\begin{figure}[t]
     \centering
     \includegraphics[width=0.95\linewidth]{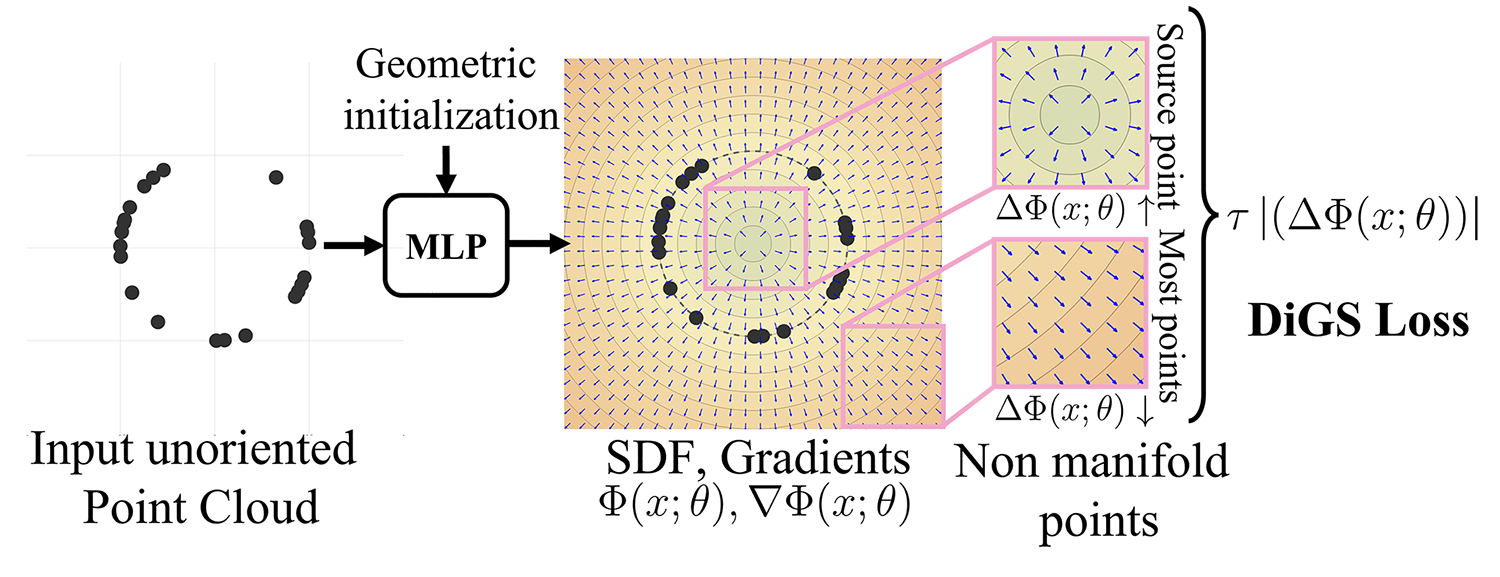}
     \caption{Toy 2D example. Given an unoriented input point cloud sampled on a shape, we train a shape implicit neural representation using a geometric initialization and a divergence penalty loss which arises from the observation that in most locations the divergence of the signed distance function should be low. }
    \label{fig:teaser}
\end{figure}

In this work, we introduce a divergence guided shape implicit neural representation learning approach (DiGS INR), which uses raw 3D point data for supervision without any pre-processing stages. Our approach is motivated by the observation, illustrated in \figref{fig:teaser}, that the gradient vector field of the signed distance function (produced by the network) has low divergence nearly everywhere. We incorporate this geometric prior as a soft constraint in the loss function and anneal it as training progresses. This requires a network architecture that has continuous second derivatives, such as SIRENs \cite{sitzmann2020siren} (which we use), in contrast to many previous works that use ReLU multi-layer perceptrons (MLPs). Additionally, we propose two novel geometrically driven initialization methods for this architecture to set the initial zero level set (i.e., the estimated shape's surface) to be approximately spherical, one of which explicitly maintains high frequencies in a controllable fashion. 

We test DiGS on the task of surface reconstruction and shape space learning, and use shapes from the Surface Reconstruction Benchmark (SRB) \cite{berger2013benchmark}, DFAUST~\cite{bogo2017dynamic}, and ShapeNet \cite{chang2015shapenet} datasets, which consist of shapes with challenging properties, \eg, topological complexity, nonuniform sampling, registration misalignment, missing data and different feature sizes (levels of detail). We show that our method suffers very little, if at all, without normal supervision and performs on par and sometimes better than state-of-the-art methods that use normals. Importantly, it performs significantly better than other methods that work on unoriented point clouds. 

%Furthermore, we observe that DiGS reconstructions avoid the need for ad hoc post-processing of ghost geometries and are often smoother than other methods, which is a desirable trait when dealing with planar and other low curvature surfaces, yet still reconstructs sharp boundaries when present. To quantify this trait, we introduce a new evaluation metric that measures the reconstructed surface roughness. For shape space learning we report on the DFAUST dataset \cite{bogo2017dynamic} and show that unseen geometries can be reconstructed given scan sequences of human motion.  

The main contributions of this paper are: 
\begin{itemize}[topsep=0pt]
    \denselist
    \item Introducing divergence guided shape INR learning which incorporates soft second order derivative constraints to guide the INR learning process.
    \item Deriving two novel geometric initialization methods for sinusoidal-based shape INR networks that lead to better representations. 
    % \item Learning from raw point cloud data without ground truth normal vector or distance to manifold information. 
    % \item A new evaluation metric for surface reconstruction that quantifies surface roughness. 
\end{itemize}

% \clearpage
\section{Related-work}
\label{Sec:related-work}

% \subsection{Surface reconstruction}
% \label{Sec:related-work:surf-recon}
\textbf{Surface reconstruction.}
Reconstructing surfaces from point clouds is a difficult problem that has been studied for many years. Its main challenges include (1) nonuniform point sampling, (2) noisy point positions and normals due to sampling inaccuracy, scan registration misalignment and normal estimation errors, and (3) missing data in parts of the surface due to occlusions. Reconstruction methods attempt to overcome these challenges and infer the unknown underlying surface. Classical combinatorical approaches for reconstruction include  triangulation methods \cite{cazals2006delaunay, kolluri2004spectral}, alpha shapes \cite{bernardini1999ball} and Voronoi diagrams \cite{amenta2001power}. Classical implicit function based approaches include piecewise polynomial functions \cite{nagai2009smoothing, ohtake2005sparse}, summing radial basis functions \cite{carr2001reconstruction}, and solving a Poisson equation to find a global indicator function \cite{kazhdan2006poisson, kazhdan2013screened_poisson}. For a more in depth survey of classical surface reconstruction methods we refer the reader to Berger \etal~\cite{berger2017survey}.
Recently, and more relevant to this work, learning based approaches with neural networks have been proposed. These include parameteric methods that learn mappings from a parametric space to a region of the surface \cite{groueix2018papier, williams2019DGP}, methods that learn an implicit volumetric function on regular grids \cite{ jiang2020sdfdiff, tatarchenko2017octree}, methods that learn implicit neural representations (INRs) for the shape and recently a differentiable point-to-mesh optimisation layer \cite{peng2021shape}.

% \subsection{Shape implicit neural representations (INRs)}
% \label{Sec:related-work:shapeINR}
\textbf{Shape implicit neural representations (INRs).}
Neural networks have shown to be very efficient in representing shapes as implicit functions~\cite{park2019deepsdf, mescheder2019occupancy, gropp2020igr, atzmon2020sal, Atzmon_2020_CVPR_SALD, michalkiewicz2019implicit, tancik2020fourfeat, sitzmann2020siren, williams2020neural_splines, yifan2021geometry, lipman2021phase}, and scenes \cite{peng2020convolutional_occ, sitzmann2019srns, jiang2020local, azinovic2021neural}. These networks are trained to output either a signed distance function \cite{park2019deepsdf, atzmon2020sal, Atzmon_2020_CVPR_SALD, gropp2020igr, michalkiewicz2019implicit, williams2020neural_splines, sitzmann2020siren, sitzmann2019srns, jiang2020local, yifan2021geometry, azinovic2021neural},  an occupancy function \cite{mescheder2019occupancy, peng2020convolutional_occ, tancik2020fourfeat}, or recently a representation that unifies the two \cite{lipman2021phase}. 

Our method uses a signed distance function (SDF) to represent the shape. SDFs were popularised for shape INRs by DeepSDF \cite{park2019deepsdf}. However, for off surface points they require the ground truth for whether they are inside/outside the shape (i.e. the sign of the SDF), which is unrealistic. SAL \cite{atzmon2020sal} improves upon this by using a sign agnostic training method that requires no extra 3D supervision, and SALD \cite{Atzmon_2020_CVPR_SALD} shows that using ground truth normals for points on the surface greatly improves performance. IGR \cite{gropp2020igr} introduce an important  eikonal equation based loss term to regularise the learnt function towards being an SDF, which can be used with or without normals. NSP \cite{williams2020neural_splines} frame the task as a kernel regression problem where the kernel chosen is equivalent to the limit of infinitely wide, shallow ReLU networks. FFN \cite{tancik2020fourfeat} and SIREN \cite{sitzmann2020siren} demonstrate that INRs using conventional MLPs bias toward low-frequency solutions, and thus explicitly introduce high frequencies into their architecture. FFN (which learn an occupancy implicit function) use a ReLU MLP with a Fourier feature layer, while SIREN uses periodic (specific sine) activation functions. The benefit of the latter is that second derivatives (and in fact all orders) are defined and continuous, which allows for higher order supervision as we use in this paper. On the other hand, DeepSDF, SAL use ReLU for activation, so the same cannot be done with those methods.

Ground truth normal information is not usually available in real world, raw scan data, and must be noisily estimated. Note that all these methods, except SAL (and DeepSDF which uses other unrealistic ground truth information), report results with ground truth normal information. Technically IGR and SIREN can operate without such information by dropping the corresponding loss term in their overall loss, but our experimental results in Section~\ref{Sec:results} show that performance significantly drops in their absence. On the other hand, we show that without normal information our method does not reduce performance as significantly, and in fact does on par with methods that use normals.

Concurrent work include IDF \cite{yifan2021geometry} and PHASE \cite{lipman2021phase}. IDFs extend SIRENs to explicitly decompose learning the high and low frequencies of a shape, and combine by having the local high frequency as a displacement in the low frequency SIREN's normal direction. PHASE introduce a loss that learns a density function whose limit is an occupancy, and whose log transform is a SDF. They also introduce a version without ground truth normal supervision, PHASE+FF, which uses the Fourier features of FFN \cite{tancik2020fourfeat}.

% \subsection{Initializing Shape INRs}
% \label{Sec:related-work:init-shapeINR}
\textbf{Initializing shape INRs.}
As finding good shape implicit representations is hard due to the ill-posed nature of the task~\cite{berger2017survey}, many methods introduce initialization or regularization to bias toward favourable solutions~\cite{atzmon2020sal,Atzmon_2020_CVPR_SALD,gropp2020igr}. SAL~\cite{atzmon2020sal} introduce a geometric initialization for ReLU MLPs, which carefully chooses the initial weights (or the distribution of the initial weights) in particular layers to make the initial function be approximately the signed distance function to a $r$-radius sphere. This initialization has shown to have favourable properties for reconstruction, such as in-object and out-object sign consistency.

For SIRENs, to allow for training deeper networks for general INRs, Sitzmann \etal~\cite{sitzmann2020siren} proposed an initialization that preserves the distribution of activations through its layers. However for shape INRs, this initialization lacks a geometric grounding and sometimes causes unwanted ghost geometries. In this paper, we propose two novel initializations for SIRENs that extend the notion of geometric initialization to periodic activations and show that initializing with high frequencies is important for capturing fine detail.

\section{Outline of Divergence Guided Shape INRs}\label{sec:overall-method}
We use a smooth-to-sharp approach that keeps the gradient vector field stay highly consistent during training (see \figref{fig:training}). In particular, it allows us to learn a good implicit representation without normal information. 
The proposed training procedure consists of the following four steps, which will be detailed in subsequent sections: 
% \noindent The training procedure consists of the following steps: 
\begin{itemize}[topsep=0pt]
    \denselist
    \item \textbf{Geometric initialization}. Initialize to a sphere, biasing the function to start with an SDF that is positive away from the object and negative in the centre of the object's bounding box, while keeping the model's ability to have high frequencies (in a controllable manner).
    \item \textbf{High divergence phase.} Guide the model towards a smooth reconstruction of the coarse shape. Importantly, this prevents the model from prematurely fitting to fine details.
    \item \textbf{Annealing divergence phase.} Slowly allow fine details to emerge while still learning a function that has smoothly changing normals. 
    \item \textbf{Low divergence phase.} Allow very fine details such as sharp corners to emerge, and for the function to interpolate the original data (point cloud samples) as much as possible (subject to also minimising the Eikonal term).
\end{itemize}
A high weight on the divergence loss produces very smooth SDF functions, leading to oversmoothed reconstructions. However learning such smooth representations can be done quickly and robustly. In the supplemental material, we provide a video showing the effects of this procedure on the reconstructions at different iteration steps. 
% We use $(t_1, t_2, t_3)=(0, 0.5, 0.75)t$, here $t$ is the total number of iterations. \\
We divide the total number of iterations in 50\%, 25\% and 25\% for the high, annealing and low divergence phases, respectively.
\\
% In Section~\ref{Sec:approach_init} we introduce our geometric and multi-frequency geometric (MFGI) initalizations, and discuss how to initialize to a good starting SDF while also keeping high frequencies ready for learning fine detail.
% Then, in Section~\ref{Sec:approach_divergence} we discussed our divergence loss and how it can help regularize towards simpler solutions.

\begin{figure}
        \centering
        % trim={<left> <lower> <right> <upper>}
        \rotatebox[origin=lB]{90}{\ \ \ \textbf{Our DiGS}}
        \includegraphics[scale=0.2]{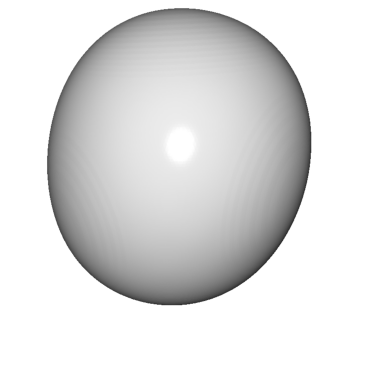}
        \includegraphics[scale=0.2]{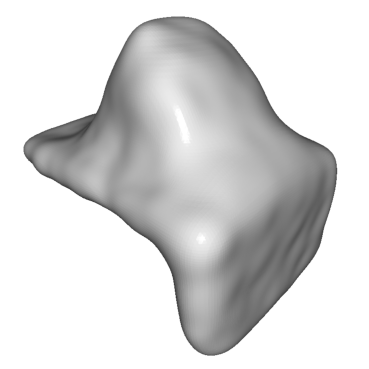}
        \includegraphics[scale=0.2]{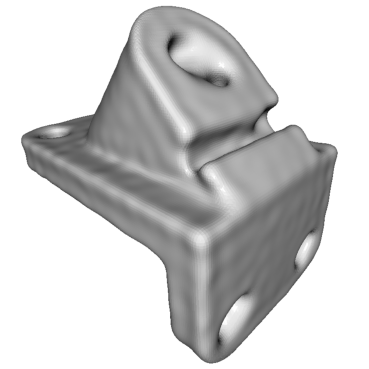}
        \includegraphics[scale=0.2]{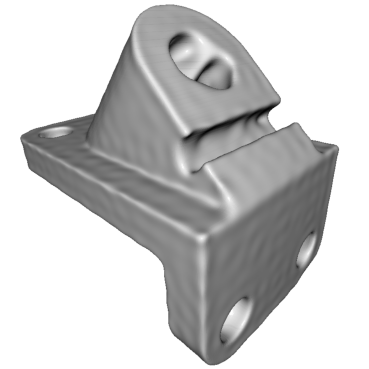}\\
        \rotatebox[origin=lB]{90}{\,SIREN wo n}
        \includegraphics[scale=0.2]{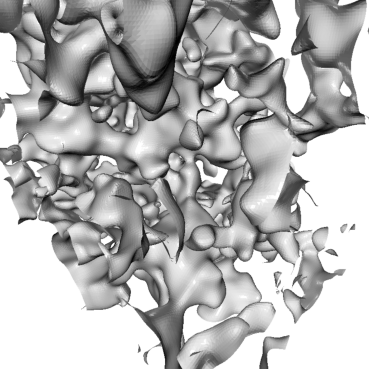}
        \includegraphics[scale=0.2]{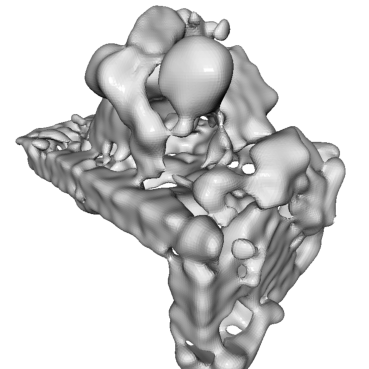}
        \includegraphics[scale=0.2]{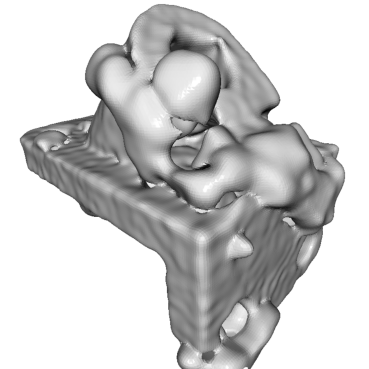}
        \includegraphics[scale=0.2]{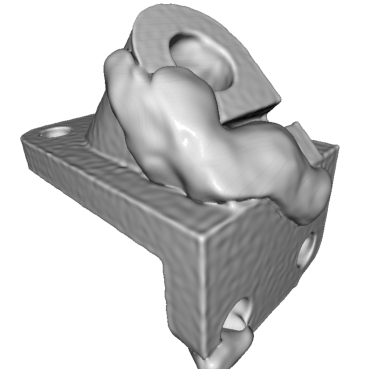}
        \caption{Results at four different training iterations (progressing from left to right) for DiGS (top) and SIREN wo n (bottom).}
        \label{fig:training}
\end{figure}

\section{Geometric initialization for SIRENs}
\label{Sec:approach_init}
We now detail our two geometrically motivated initializations for SIRENs, shown in \figref{fig:geometric_init} for 2D. 

\textbf{Initialization to a sphere.}
A key component of our method is a geometrically meaningful initialization for the parameters of SIRENs such that the initial signed distance function is approximately an $r$-radius sphere. 

Let us consider a SIREN specified by
\begin{align}
    \Phi(x; \theta) &= \mathbf{w}_n^T\left(\phi_{n-1}\circ\phi_{n-2}\circ\ldots\circ\phi_0\right)(x)+ b_n, \notag \\
    \phi_i(x_i) &=\sin\left(\mathbf{W}_ix_i+\mathbf{b}_i\right)
    \label{Eq:SIREN}
\end{align}
where $\phi_i:\mathbb{R}^{M_i}\to\mathbb{R}^{N_i}$ is the $i^{th}$ layer of the network, with input $x_i\in \mathbb{R}^{M_i}$ (so $x_0=x$) and parameters $\theta = \bigl\{ \mathbf{w}_n, b_n, \mathbf{W}_{n-1}, \mathbf{b}_{n-1}, \ldots \mathbf{W}_1, \mathbf{b}_1 \bigr\}$ where $\mathbf{W}_i\in\mathbb{R}^{N_i\times M_i}$,  $\mathbf{b}_i\in\mathbb{R}^{N_i}$, $\mathbf{w}_n\in \reals^{M_n}$, $b_n\in \reals$.
Rather than approximating a norm using smooth SIRENs, we instead approximate the more tractable signed squared norm \cite{willians2021drps} and
apply the following function to the output of the SIREN:
\begin{equation}
    \nu(d) = \sign{d}\sqrt{|d| + \varepsilon}.
\end{equation}
Thus, we develop an initialization of the network's parameters, $\theta = \theta_0$, such that $\nu(\Phi(x; \theta_0)) \approx \norm{x}{2}$ for $x$ within the unit ball. Note translating and scaling to the unit ball is standard for point clouds, so we are only interested in the INR within this region. We then manually minus $r$ from this to initialize to a an $r$-radius sphere.

The following proposition shows this for a single hidden layer SIREN, where we approximate $z\mapsto z^2$ using a translated sine wave (see supplemental for proof). 

\begin{proposition}
\label{prop:single_layer}
Let $\Phi$ be a single hidden layer SIREN ($n=1$ in \eqnref{Eq:SIREN}) of dimension $M_n$ and let $x$ be a point within the unit ball. Set  $\mathbf{W}_{n-1}=\frac{\pi}{2}I$, $\mathbf{b}_{n-1}=\frac{\pi}{2}\mathbf{1}$, $\mathbf{w}_n=-\mathbf{1}$ and $b_n=M_n$. Then, $\nu(\Phi(x)) \approx \norm{x}{2}$.
\end{proposition}

To extend this to networks with an arbitrary number of layers, we design layers $\phi_i$ that preserve the norm on expectation w.r.t.\ the weights of the layer: $\mathbb{E}[\norm{\phi_i(x_i)}{2}]=\norm{x_i}{2}$, so $\|x_{i+1}\|_2\approx \|x_i\|_2$. We do this by sampling entries for the weight matrix uniformly in the range $[-c^i_{wr}, c^i_{wr}]$ (defined below), which makes the rows approximately orthonormal. This also preserves the distribution type of the activations through all layers (see Sitzmann \etal~\cite{sitzmann2020siren}). Applying \propref{prop:single_layer} to the last two layers yields the following proposition (see \figref{fig:geometric_init_illustration} for a illustration of this).

\begin{proposition}
\label{prop:multilayer_init}
Let 
$\Phi$ be a $n$-hidden layer SIREN (\eqnref{Eq:SIREN}) 
that maps from $\mathbb{R}^{M_0}\to\mathbb{R}$ 
and $\|x\|_2\leq 1$. Set $\mathbf{W}_i\sim\mathcal{U}\left(-c^i_{wr},c^i_{wr}\right)$, 
$c^i_{wr}=\sqrt{\frac{3}{M_{i+1}}}$,
$\mathbf{b}_i=\mathbf{0}$ for $0\leq i \leq n-2$ and $\mathbf{W}_{n-1}=\frac{\pi}{2}I$, $\mathbf{b}_{n-1}=\frac{\pi}{2} \mathbf{1}$, $\mathbf{w}_n=-\mathbf{1}$ and $b_n=M_{n}$. Then $\nu(\Phi(x))\approx \|x\|_2$.
\end{proposition}

We perturb all constant parameters in \propref{prop:multilayer_init} with small Gaussian noise to facilitate learning.

However, as this initialization keeps all activations within the first period of $\sin$, in practice SIRENs initialized this way will not generate high frequency output. Sitzmann \etal~ \cite{sitzmann2020siren} also notes this problem, and specifically scales the weight matrix of the first layer by $\omega_0=30$ to hit up to 30 periods, thus giving multiple frequencies through the network. To overcome this problem, we proposed the multi frequency geometric initialization (MFGI). 

\begin{figure}[]
     \centering
     \begin{tabular}{c c c}

         \raisebox{-.5\height}{\includegraphics[width=0.25\linewidth, trim=30pt 30pt 7pt 37pt, clip]{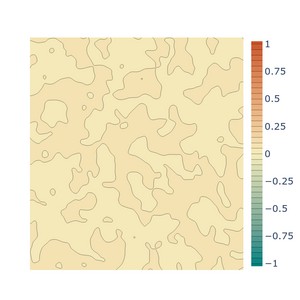} }
         &

         \raisebox{-.5\height}{\includegraphics[width=0.25\linewidth, trim=30pt 30pt 7pt 37pt, clip]{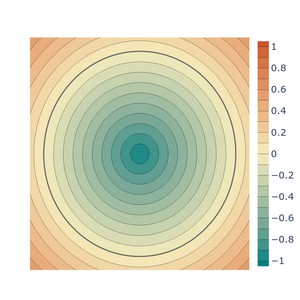} }
         &
        
         \raisebox{-.5\height}{\includegraphics[width=0.25\linewidth, trim=30pt 30pt 7pt 37pt, clip]{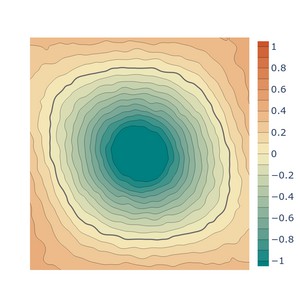} }
         \\

         SIREN \cite{sitzmann2020siren} &  Our Geometric & Our MFGI \\
     \end{tabular}
        \caption{Visualization of the SDF of the proposed geometric initialization and multi-frequency geometric initialization (MFGI) for SIRENs in 2D compared to Sitzmann \etal~\cite{sitzmann2020siren}.}
        \label{fig:geometric_init}
\end{figure}

\begin{figure}[tb]
     \centering
         \includegraphics[width=0.98\linewidth]{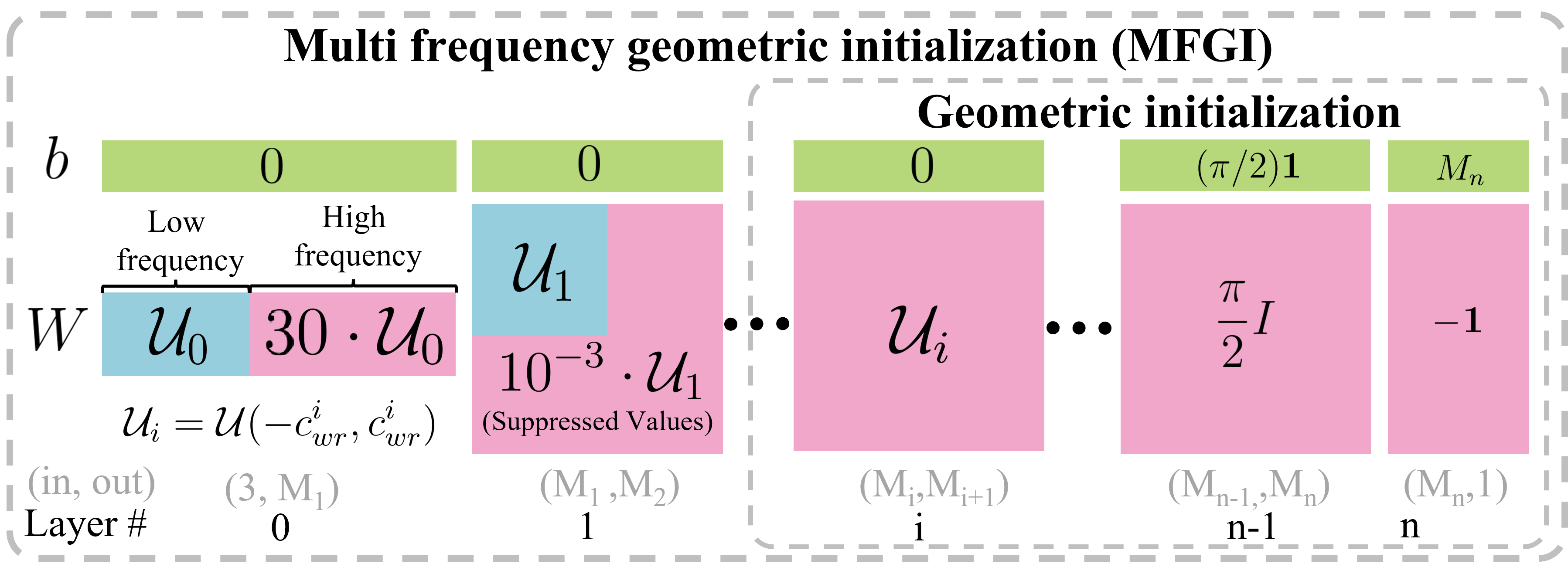} 
        \caption{Illustration of our geometric initialization and MFGI.}
        \label{fig:geometric_init_illustration}
\end{figure}

\textbf{Multi frequency geometric initialization (MFGI).}
We initialize using the geometric initialization of  \propref{prop:multilayer_init} and introduce high frequencies into the first layer in a controlled manner (see \figref{fig:geometric_init_illustration} for an illustration of the method). Specifically, we keep the initialization for the first $k_r$ rows of the weight matrix $\mathbf{W}_0 = \left[w^0_{ij}\right]\in\reals^{N_0 \times M_0}$ as per \propref{prop:multilayer_init}, and for the last $N_0 - k_r$ rows we scale the initialized values by $n_p=30$. This means that the output after the first layer, $x_1\in\reals^{N_0}$, has its first $k_r$ elements hitting only one period and its remaining $N_0 - k_r$  elements hitting up to $n_p$ periods. Then so that our geometric initialization still works, we scale down any part of the next weight matrix, $\mathbf{W}_1 =\left[w^1_{ij}\right]\in\reals^{N_1 \times M_1}$, that would multiply into the multi-period part of the vector by a  factor $s = 10^{-3}$. This can be summarised as: 
\begin{equation}
        %  \resizebox{1.\linewidth}{!}{$
% \begin{array}{c c}
\begin{split}
            & w^0_{ij} \sim \begin{cases} 
                \mathcal{U}\left(-c^0_{wr},c^0_{wr}\right) &  {\scriptstyle 0 \leq j \leq k_r}  \\
                \mathcal{U}\left(-n_p\, c^0_{wr},n_p\, c^0_{wr}\right) & \text{otherwise} \\%{\scriptstyle 1 + k_r < j < N_0}  \\ 
            \end{cases}
\\
    %  &
        & w^1_{ij} \sim \begin{cases} 
                \mathcal{U}\left(-c^1_{wr},c^1_{wr}\right) & {\scriptstyle 0 \leq j \leq k_r, 0 \leq i \leq k_r}\\
                \mathcal{U}\left(-s\, c^1_{wr},s\, c^1_{wr}\right) & \text{otherwise}  \\ 
            \end{cases}
    \\
% \end{array}
% $}
\end{split}.
\end{equation} 

\noindent In our experiments we found that $k_r \approx \frac{1}{4}N_0$ is sufficient.

\textbf{Visualisation of the initalizations.} 
\figref{fig:geometric_init} shows the two initializations introduced compared to the initialization proposed in SIREN \cite{sitzmann2020siren} for a 4-layer SIREN with 128 nodes in each layer. Notice that SIREN's initialization has both SDF and gradient vector norm of almost zero everywhere. On the other hand, our geometric initialization has sphere-like level sets, which provide smooth and desirable eikonal and divergence terms (see supplemental for visualization of these components). Our MFGI initialization is a noisy version of our geometric initialization, where the amount of noise can be tuned by $n_p$, $k_r$ and $s$.

% \textbf{Multi frequency weight regularization:}
% To further encourage the network to utilise the higher frequencies as training progresses, we introduce a weight regularization term that normalizes the distribution of the weights in the first layer. Specifically we encourage high entropy: 
% \begin{equation}
%     L_{w} = \sum_{i=1}^{N_0}\sum_{j=1}^{M_0}p(w^0_{ij})\log\left(p(w^0_{ij})\right)
% \end{equation}
% where $p(w^0_{ij})$ is estimated using the Parzen-window method \cite{parzen1962estimation} with a Gaussian kernel. 

%%%%%%%%%%%%%%%%%%%%%%%%%%%%%%%%%%%%%%%%%%%%%%%%%%%%%%%%%%%%%%%%%%

\section{DiGS Loss}

\label{Sec:approach_divergence}
\subsection{Existing loss function components and setup}\label{subsec:existing-losses}
Prominent neural implicit representation methods~\cite{atzmon2020sal, Atzmon_2020_CVPR_SALD, gropp2020igr, park2019deepsdf, sitzmann2020siren} train a neural network with parameters $\theta$ to output a SDF $\Phi(x; \theta)$ to the surface of an underlying (unknown) shape for every given point $x \in \reals^3$ . They require several loss functions when training to constrain the learned function in certain ways: (A) \textit{manifold constraint:} points on the surface manifold should be on the function's zero level set \cite{atzmon2020sal, Atzmon_2020_CVPR_SALD, gropp2020igr, park2019deepsdf, sitzmann2020siren}, (B) \textit{normal constraint:} the gradients of points on the surface manifold should match the ground truth normal if such supervision is available \cite{Atzmon_2020_CVPR_SALD, gropp2020igr, sitzmann2020siren} (C) either (C1) \textit{non-manifold constraint:} points off the surface manifold should match their ground truth SDF or SDF magnitude \cite{atzmon2020sal, Atzmon_2020_CVPR_SALD, park2019deepsdf} or (C2) \textit{Eikonal constraint:} all points should have a unit gradient~\cite{gropp2020igr, sitzmann2020siren}, (D) \textit{non-manifold penalisation constraint:} off surface points should not have zero SDF \cite{sitzmann2020siren}. Note that (A) and (C) are necessary, and (B) and (D) are used for improving results.

Our method uses the same network architecture and loss functions of SIREN \cite{sitzmann2020siren}, which we provide for completeness. Given domain $\Omega$ and surface manifold $\Omega_0$, they define the above loss functions as
\begingroup
\allowdisplaybreaks % allow page break within this group for this align env
\begin{align}
    L_{A} &= \int_{\Omega_0} \norm{\Phi(x ;\theta)}{2} dx \label{eq:L-mnfld} \\
    L_{B} &= \int_{\Omega_0} \left(1 - \left\langle \nabla_x \Phi(x; \theta) , n_{\text{GT}}(x) \right\rangle \right) dx \label{eq:L-n} \\
    L_{C2} &= \int_\Omega \abs{\norm{\nabla_x \Phi(x; \theta)}{2} - 1} dx \label{eq:eikonal}\\
    L_{D} &= \int_{\Omega \setminus \Omega_0} \exp{\left(-\alpha\abs{\Phi(x; \theta)}\right)}, \quad \alpha \gg 1 \label{eq:L-inter}.
\end{align}%
\endgroup

In summary, the final loss for SIREN is given by a weighted sum of all of the above terms:
\begin{equation}
L_{SIREN} = \lambda_{A} L_{A} + \lambda_{B} L_{B} +  \lambda_{C2} L_{C2} + \lambda_{D} L_{D} 
\end{equation}
with $\left(\lambda_{A}, \lambda_{B}, \lambda_{C2}, \lambda_{D} \right) = \left(3000, 100, 50, 100 \right)$ as given by Sitzmann et al. \cite{sitzmann2020siren}.

Supervision for (B) and (C1) are either unlikely to have in practice or costly to approximate well. Gropp \etal~\cite{gropp2020igr} showed however that (C1) can be replaced by (C2), a popular constraint in PDE theory called the Eikonal equation, which does not require extra data. In fact, given continuous (and sufficiently well behaved) boundary constraints (A), solving for a viscous solution to the PDE defined by the Eikonal equation will be unique (i.e., only (A) and (C2) are necessary). The difficulty in practice, however, is that constraint (A) is only defined at discrete points, resulting in infinitely many solutions, hence the requirement of other constraints to guide to favourable solutions (especially (B)).

We observe and deal with another problem (that we highlight in Section~\ref{subsec:div-as-reg}): our loss functions (and thus constraints) can only be evaluated at finite, discrete points within the domain. As a result the quality of our solution depends highly on the sampling density of our method, and how effectively our losses constrain the surrounding region. While we cannot increase the number of discrete points we get as supervision for (A), (B) greatly constrains the function space by adding higher order information (but requires extra supervision). We now do the same for (C) without extra supervision, and show that it can even remove the need for (B). Note that to benchmark against no normal supervision (B), we can also define
\begin{equation}
L_{SIREN\ wo\ n} = \lambda_{A} L_{A} + \lambda_{C2} L_{C2} + \lambda_{D} L_{D}.
\end{equation}
% however this does not perform well in practice (see \secref{Sec:results}).

\subsection{Second order unsupervised constraints} \label{subsec:second-order-loss}
We turn to second order information to further constrain the SDF around each sample on $\Omega \setminus \Omega_0$. Given the gradient vector field $\nabla \Phi$, we can compute its curl, $\nabla\times \nabla \Phi$, and divergence, $\Delta \Phi := \nabla \cdot \nabla \Phi$ (note that this is also known as the Laplacian of the underlying scalar field $\Phi$)\cite{marsden2003vector}. The curl of any gradient vector field is zero everywhere so it does not give us any information. However, we can observe that for a ground truth SDF, the magnitude of the divergence is very low at most areas of the domain 
% (see ~\figref{fig:appx:gt-vis}, and ~\secref{appx:sec:div-theory}
(see supplemental material 
for discussion and intuition). We can thus impose a penalty on the magnitude of the divergence, \ie,
% Need to be consistent with why L_div is not at manifold points!
\begin{equation}\label{eq:L-div}
L_{div} = \int_{\Omega \setminus \Omega_0} \!\! \abs{\Delta \Phi(x; \theta)}dx = \int_{\Omega \setminus \Omega_0} \!\! \abs{\nabla_x \cdot \nabla_x \Phi(x; \theta)} dx.
\end{equation}
This leads to our proposed DiGS loss, given by: 
\begin{equation}
L_{DiGS} = L_{SIREN\ wo\ n} + \tau\lambda_{div} L_{div}
\end{equation}
where we use $\lambda_{div}=100$ and $\tau$ is an annealing factor we discuss in Section~\ref{sec:overall-method}.

Note that our approach can only be used with architectures that have activation functions with nonzero second derivative such as SIRENs. ReLU based networks will not be affected by the divergence constraint.

% \textbf{Second order supervised constraint:}
% Normal vectors are often not available during training and are therefore estimated using some local approximation method \cite{cazals2005estimating, guerrero2018pcpnet, ben2020deepfit, ben2019nestinet}. Some methods also provide an approximation of the principal curvatures \cite{cazals2005estimating, ben2020deepfit}. The mean curvatures $\kappa_{mean} = \frac{1}{2}(\kappa_1 + \kappa_2)$ provides second order information that can be utilised as supervision for learning shape representation. We propose a supervised variation to DiGS that penalizes points on the surface for having a different mean curvature than the divergence of the vector field using 
% $
% L_{curv} = \int_{ \Omega_0} \abs{ \Delta \Phi(x; \theta)} - 2\abs{\kappa_{mean}}dx
% $
% as an additional loss term. 

% Note that normal and curvature estimations are noisy and highly depend on local neighboring points support size, therefore adding these supervisory signals does not guarantee improved performance (see supplemental for full details on this variation). 

\subsection{Minimising divergence as regularisation}\label{subsec:div-as-reg}
We now provide some justification of the loss in Equation~\ref{eq:L-div} by showing that it is equivalent to regularising the learnt function. We can quantify the complexity of our learnt function by using the Dirichlet Energy. The Dirichlet Energy of a function $\Phi$ over a space $\Omega$ gives a notion for how smooth or variable the function is \cite{bronstein2017geometric} (where lower implies smoother), defined by
\begin{equation}
    E[\Phi]=\frac{1}{2}\int_\Omega\|\nabla\Phi(x)\|_2^2 \, dx.
\end{equation}
To minimize this with respect to our constraints (A)--(D), it suffices to find the function satisfying our conditions whose magnitude of the divergence is as small as possible, i.e., minimizing our divergence term $L_{div}$ in Equation~\ref{eq:L-div} (see
% Section~\ref{appx:sec:div-theory} 
supplemental material
for proof).

Why not explicitly minimize the Dirichlet energy as a loss function? In fact, doing so would be redundant, due to our Eikonal term (\ref{eq:eikonal}): if the gradient is constrained to have unit norm on the sampled points, then the Dirichlet energy, as far as can be determined from at those sampled points, is already determined. However we argue that adding our divergence term does a much better job of reducing the variability of our learned function over the entire space, rather than just having the Eikonal term. The reason for this is that it constrains the local region more due to its second order nature. This can also be motivated by the divergence theorem~\cite{morse1954methods} (see
% Section~\ref{appx:sec:div-theory}
supplemental for a thorough discussion). We use the following toy problem to demonstrate this. 

Consider the problem of learning the SDF to the line $y=0$ where below the line is negative, i.e. $\Phi((x,y))=y$. We train a 2-layer SIREN with point constraints (A) and Eikonal term (C2), and then add our divergence constraint for comparison. For point constraints we sample 10 points on the lines $y=-1$, $y=0$ and $y=1$ (for $x\in\{0, 1\}$), and for Eikonal and divergence constraints we sample on a $n\times n$ grid. We repeat the experiment for $n\in\{20, 200\}$ and evaluate our learned functions on a finer $m\times m$ grid ($m=1000$). We perform 20 repetitions of these four experiments with different random initializations.

The results can be seen in Table~\ref{tab:toyDEproblem}, where we report the mean Dirichlet energy $\bar{E}=\overline{\|\nabla f\|_2^2}$, the mean gradient norm $\overline{\|\nabla f\|_2}$ and its standard deviation $\sigma(\|\nabla f\|_2)$. We also visualize the learnt SDFs for one of the repetitions in \figref{fig:toyDEProblemSmall}. The results show that adding the divergence constraint both reduces the mean Dirichlet energy, indicating that the learned function is less variable, and makes the function more faithful to the Eikonal constraint. This can also be seen in the visualization, where the level set contours are less variable and have more consistent spacing. More visualizations are provided in the supplemental.

\begin{table}[tb]\scriptsize
    % \centering
    % \begin{minipage}{0.6\textwidth}
        \centering
        % \resizebox{\textwidth}{!}{
        \setlength{\tabcolsep}{2pt}
        \begin{tabular}{l c | c c c }
        \toprule
            Loss        & grid size & $\bar{E}=\overline{\|\nabla f\|_2^2}$ & $\overline{\|\nabla f\|_2}$ & $\sigma(\|\nabla f\|_2)$ \\
         \hline
            A+C2        & 20x20     & 4.32 $\pm$ 3.07 & 1.63 $\pm$ 0.51 & 1.05 $\pm$ 0.56  \\
            A+C2 + Div  & 20x20     & 3.37 $\pm$ 3.89 & 1.46 $\pm$ 0.76 & 0.69 $\pm$ 0.43  \\
            A+C2        & 200x200   & 3.58 $\pm$ 3.55 & 1.51 $\pm$ 0.54 & 0.85 $\pm$ 0.52 \\
            A+C2 + Div  & 200x200   & \textbf{1.37 $\pm$ 1.35} & \textbf{1.04 $\pm$ 0.39} & 0.22 $\pm$ 0.33 \\
         \bottomrule
        \end{tabular}
        % }
        \caption{Results on the toy problem. Comparing mean Dirichlet energy and mean gradient norm. The divergence term significantly improves performance.}
        \label{tab:toyDEproblem}
    % \end{minipage}\hfill
    % \begin{minipage}{0.3\textwidth}
    %     \centering
    %     \includegraphics[scale=0.1, trim=20pt 20pt 20pt 20pt, clip]{assets/figures/theory-vis/2021-10-23_17-55-52_0_200x200.pdf}
    %     \includegraphics[scale=0.1, trim=20pt 20pt 20pt 20pt, clip]{assets/figures/theory-vis/2021-10-23_17-55-52_0_200x200_with_div.pdf}\\
    %     \caption{Learned SDF contour lines for the toy problem on a 200x200 grid. With(right) and without (left) a divergence loss term.}
    %     \label{fig:toyDEProblemSmall}
    % \end{minipage}
\end{table}

\begin{figure}
        \centering
        \includegraphics[scale=0.35, clip]{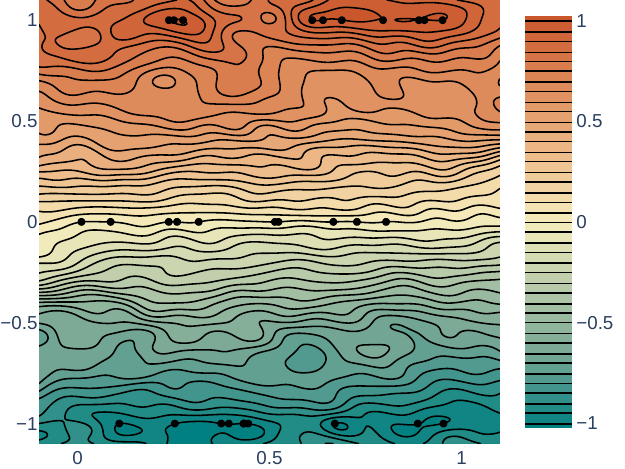}
        \includegraphics[scale=0.35, clip]{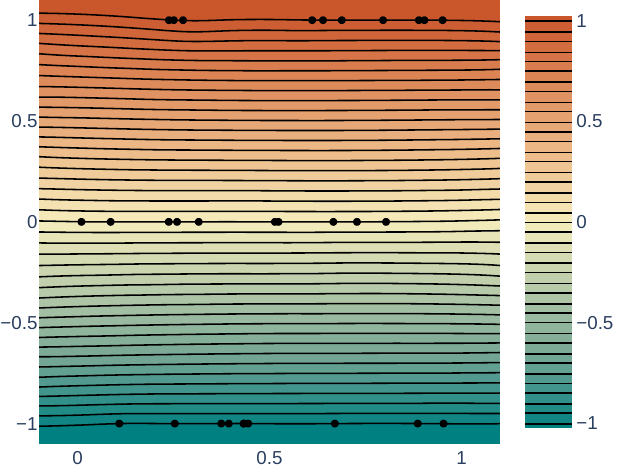}\\
        \caption{SDF contour lines for the toy problem on a 200x200 grid. With (right) and without (left) the divergence loss term.}
        \label{fig:toyDEProblemSmall}
\end{figure}
%%%%%%%%%%%%%%%%%%%%%%%%%%%%%%%%%%%%%%%%%%%%%%%%%%%%%%%%%%%%%%%%%%%%%%%%%%%%%%%%%%%

\section{Experiments}
\label{Sec:results}
We evaluate our method on the tasks of surface reconstruction and shape space learning. We test the former on the Surface Reconstruction Benchmark (SRB) \cite{berger2013benchmark}, ShapeNet \cite{chang2015shapenet} and on a scene from Sitzmann~et~al.~\cite{sitzmann2020siren}, and the latter on DFaust \cite{bogo2017dynamic}. In both tasks, we extract the mesh representing the zero level set of the shape INR using the \textit{marching cubes} algorithm~\cite{lorensen1987marching}. We follow the same mesh generation procedure as in IGR \cite{gropp2020igr} and use a grid whose shortest axis is $512$ elements, tightly fitted onto the shape (adapting the grid range to the input scan bounding box). Unless otherwise specified, we use the same architecture (5 layers, 256 units) and hyperparameters as SIREN~\cite{sitzmann2020siren}. In the supplemental material, we provide full implementation details, extended results, ablations, additional experiments and additional visualizations. 

% This includes 2D reconstruction experiment for visualisations of the loss function for different approaches, , ... .

\textbf{Evaluation metrics.} To compare between two point sets we use the Chamfer ($d_C$) and Hausdorff ($d_H$) distances. 
We follow IGR~\cite{gropp2020igr} and sample 1M points on the reconstruction and compare to the GT and scan point clouds.
For ShapeNet we follow NSP~\cite{williams2020neural_splines} and report the squared Chamfer and Intersection over Union (IoU). See the supplemental for definitions and further details.   

\subsection{Surface reconstruction}
\label{sec:results:surface_recon}
The task of surface reconstruction of unoriented point clouds is specified as follows. Given an input point cloud $\mathcal{X} \subset \mathbb{R}^3$ find  the surface $\mathcal{S}$ from which $\mathcal{X}$ was sampled. The point set is usually acquired by a 3D scanner which introduces various types of data corruptions, e.g., noise, occlusions, nonuniform sampling. We evaluate the performance of our method on the Surface Reconstruction Benchmark dataset \cite{berger2013benchmark} and on ShapeNet \cite{chang2015shapenet}.

We compare our results to recent prominent methods, most of which have shown to outperform classical methods, including SAL \cite{atzmon2020sal}, IGR \cite{gropp2020igr}, SIREN \cite{sitzmann2020siren}, DGP \cite{williams2019DGP} and NSP \cite{williams2020neural_splines}. 
Note that apart from SAL, these methods report results using normal vector supervision, therefore for fair comparison we evaluate on SIREN and IGR without the normal vector loss term (note that the same cannot be done with DGP and NSP). 
Furthermore we compare to results reported from the concurrent work PHASE \cite{lipman2021phase}, who evaluate on SRB. While PHASE uses normal information, they also provide a version without normals that uses Fourier features, PHASE+FF, and a baseline for IGR without normal information, IGR+FF.

\textbf{Surface Reconstruction Benchmark (SRB).}
We use the simulated scan and ground truth data provided in Williams~et~al. \cite{williams2019DGP} (freely available for academic use). The dataset contains five shapes, each with it own challenging traits, e.g., complex topology, high level of detail, missing data, and different feature sizes. In \tabref{tab:results:surf-recon-summary} we report the average Chamfer and Hausdorff distance for the shapes, and since the shapes have varying difficulties, we also report the mean deviation from the best performing method for each shape.
It shows that our method is consistently better than SoTA methods on both metrics when normal information is not available.
% Furthermore,  our method (without normals) is competitive with other methods that use normal information in training. 
% We give extended results, including results with normal supervision, in the supplementary.
Qualitative results for a subset of the dataset are shown in \figref{fig:surface_recon_qual_results}. In the absence of normal information, SIREN and IGR struggle to converge to the correct zero level set and produce undesired artifacts (ghost geometries). 
DiGS, on the other hand, is able to remove such artifacts. When compared to the version with normal supervision added there is not much change other than DiGS being slightly smoother. In fact, DiGS manages to get similar results to the methods that use normal supervision.
% Additionally, despite the lack of normal vectors in training, our results show the expected smoothing effects, such as removing undesired local perturbations in planar regions. 
For extended results (results per shape, results with normal supervision and more metrics), visualization of all shapes and a video showing the shapes from multiple angles, see the supplemental material. 

\begin{table}[]\scriptsize
    \centering
    \begin{tabular}{l c c c c}
         \toprule          Method    & $d_{C}$ & $d_H$ & $\Delta d_C$ & $\Delta d_H$  \\
            \hline
            IGR wo n        & 1.38 & 16.33 & 1.20 & 12.84 \\
            SIREN wo n      & 0.42 & \phantom{1}7.67  & 0.23 & \phantom{1}4.18  \\
            SAL \cite{atzmon2020sal}    & 0.36 & \phantom{1}7.47  & 0.18 & \phantom{1}3.99  \\
            IGR+FF \cite{lipman2021phase}  & 0.96 & 11.06 & 0.78 & \phantom{1}7.58  \\
            PHASE+FF \cite{lipman2021phase}  & 0.22 & \phantom{1}4.96  & 0.04 & \phantom{1}1.48  \\
            Our DiGS            & \textbf{0.19} & \textbf{\phantom{1}3.52}  & \textbf{0.00} & \textbf{\phantom{1}0.04}  \\
         \bottomrule
    \end{tabular}
    \caption{Results on the Surface Reconstruction Benchmark~\cite{berger2013benchmark}.
    We compare against other methods without normal information, for all methods see the supplemental.
    We report the mean Chamfer $d_{C}$ and Hausdorff distance $d_H$ to the GT scans and their mean deviation from the best performing method ($\Delta d_C$ and $\Delta d_H$).}
    \label{tab:results:surf-recon-summary}
\end{table}
\begin{table}[]\scriptsize\setlength{\tabcolsep}{3pt}
    \centering
    \begin{tabular}{l c c c c c c}
         \toprule
          & \multicolumn{3}{c}{squared Chamfer $\downarrow$} & \multicolumn{3}{c}{IoU $\uparrow$}\\
          Method    & mean & median & std & mean & median & std \\
            \hline
            SPSR \cite{kazhdan2013screened_poisson} 		& 2.22e-4 & 1.70e-4 & 1.76e-4 & 0.6340 & 0.6728 & 0.1577 \\
            IGR \cite{gropp2020igr}					& 5.12e-4 & 1.13e-4 & 2.15e-3 & 0.8102 & 0.8480 & 0.1519 \\
            SIREN \cite{sitzmann2020siren}					& 1.03e-4 & 5.28e-5 & 1.93e-4 & 0.8268 & 0.9097 & 0.2329 \\
            FFN \cite{tancik2020fourfeat}		& 9.12e-5 & 8.65e-5 & 3.36e-5 & 0.8218 & 0.8396 & 0.0989 \\
            % Biharmonic RBF          & 1.11e-4 & 8.97e-5 & 7.06e-5 & 0.8247 & 0.8642 & 0.1350 \\
            % SVR                     & 1.14e-4 & 1.04e-4 & 5.99e-5 & 0.7625 & 0.7819 & 0.1300 \\
            NSP	\cite{williams2020neural_splines}					& \textbf{5.36e-5} & 4.06e-5 & 3.64e-5 & 0.8973 & 0.9230 & 0.0871 \\
            DiGS + n 				& 2.74e-4 & \textbf{2.32e-5} & 9.90e-4 & \textbf{0.9200} & \textbf{0.9774} & 0.1992 \\
            \hline
            SIREN wo n				& 3.08e-4 & 2.58e-4 & 3.26e-4 & 0.3085 & 0.2952 & 0.2014 \\
            SAL \cite{atzmon2020sal}	& 1.14e-3 & 2.11e-4 & 3.63e-3 & 0.4030 & 0.3944 & 0.2722\\
            Our DiGS					& \textbf{1.32e-4} & \textbf{2.55e-5} & 4.73e-4 & \textbf{0.9390} & \textbf{0.9764} & 0.1262\\
         \bottomrule
    \end{tabular}
    \caption{Surface Reconstruction results on ShapeNet~\cite{chang2015shapenet}. Methods above the line use ground truth normal information, and methods below do not. The mean, median and standard deviation of the squared Chamfer distance and IoU of all 260 shapes are reported.}
    \label{tab:results:shapenet-summary}
\end{table}
\begin{figure} \scriptsize
     \centering
     \setlength\tabcolsep{2pt} % default value: 6pt
    %  \scriptsize
     \begin{tabular}{c c c | c c c}
     
         \includegraphics[width=\recviscol\linewidth]{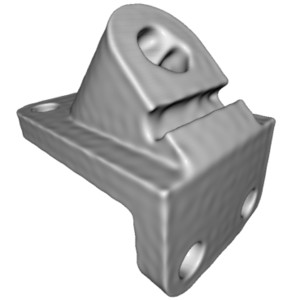}
         &
         \includegraphics[width=\recviscol\linewidth]{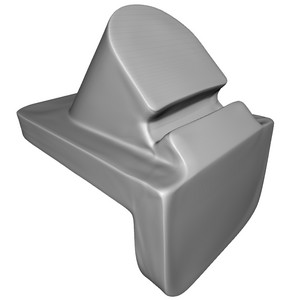}
         &
         \includegraphics[width=\recviscol\linewidth]{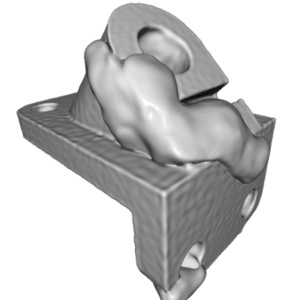}
         & 
         \includegraphics[width=\recviscol\linewidth]{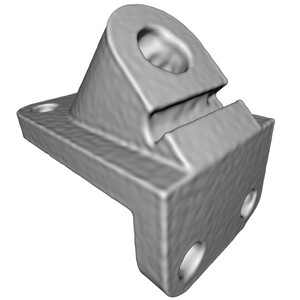}
         & 
         \includegraphics[width=\recviscol\linewidth]{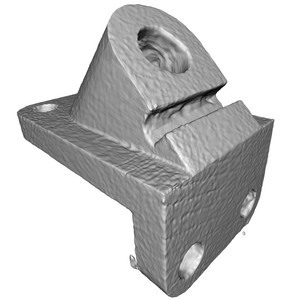}
         & 
         \includegraphics[width=\recviscol\linewidth]{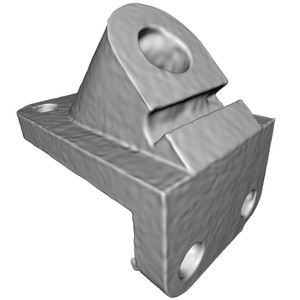}
          \\

         \includegraphics[width=\recviscol\linewidth]{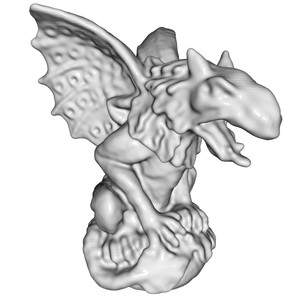} 
         &
         \includegraphics[width=\recviscol\linewidth]{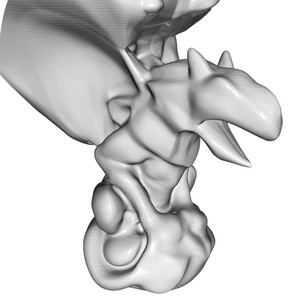} 
         &
         \includegraphics[width=\recviscol\linewidth]{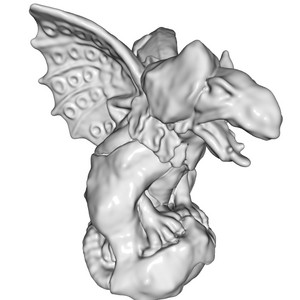} 
         & 
         \includegraphics[width=\recviscol\linewidth]{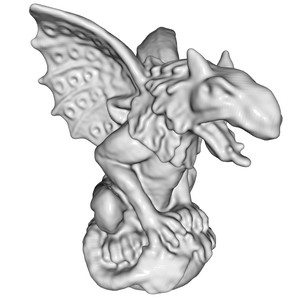}
         & 
         \includegraphics[width=\recviscol\linewidth]{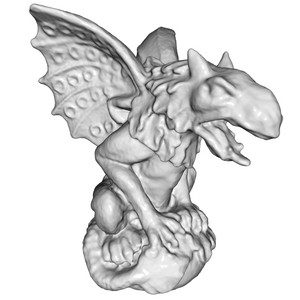}
         & 
         \includegraphics[width=\recviscol\linewidth]{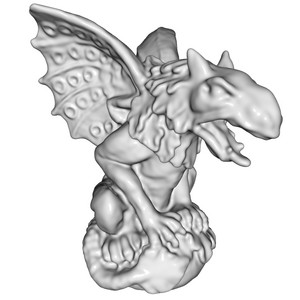} 
         \\
         
         \textbf{Our DiGS} & IGR wo n& SIREN wo n  & \textbf{Our DiGS} + n & IGR & SIREN \\
        %  \multicolumn{3}{c}{without normals} & \multicolumn{2}{c}{ with normals}
     \end{tabular}
        \caption{Qualitative results of surface reconstruction on the anchor and gargoyle shapes from the Surface Reconstruction Benchmark \cite{berger2013benchmark} compared to state of the art approaches (IGR, SIREN) that use normal vectors as ground truth.  }
        \label{fig:surface_recon_qual_results}
\end{figure}
\begin{figure}
    \centering
     \setlength\tabcolsep{5pt}
     \renewcommand{\arraystretch}{0.0}
    %  \resizebox{\textwidth}{!}{
        % trim={<left> <lower> <right> <upper>}
        \begin{tabular}{c c c c } 
        \includegraphics[scale=0.25, trim=32pt 61pt 32pt 69pt, clip]{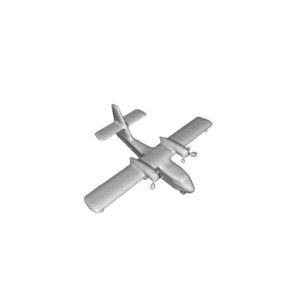} &
        \includegraphics[scale=0.25, trim=32pt 61pt 32pt 69pt, clip]{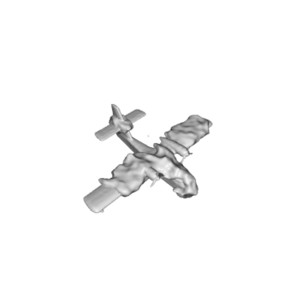} &
        \includegraphics[scale=0.25, trim=32pt 61pt 32pt 69pt, clip]{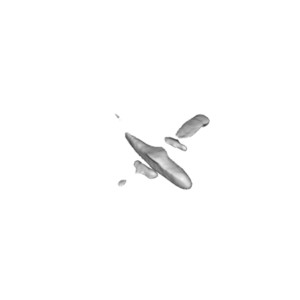} &
        \includegraphics[scale=0.25, trim=32pt 61pt 32pt 69pt, clip]{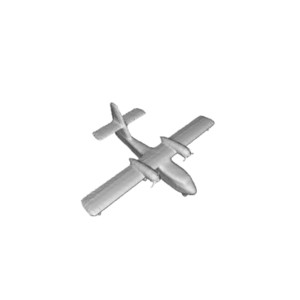}\\
        \includegraphics[scale=0.25, trim=32pt 20pt 32pt 49pt, clip]{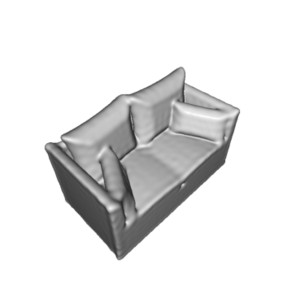} &
        \includegraphics[scale=0.25, trim=32pt 20pt 32pt 49pt, clip]{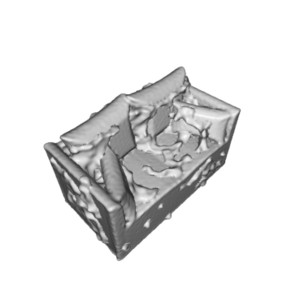} &
        \includegraphics[scale=0.25, trim=32pt 20pt 32pt 49pt, clip]{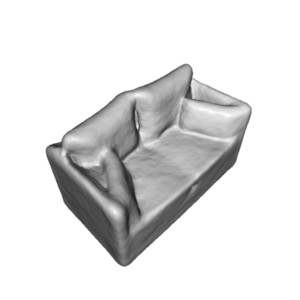} &
        \includegraphics[scale=0.25, trim=32pt 20pt 32pt 49pt, clip]{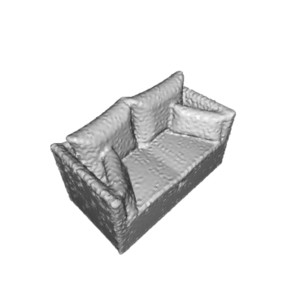}\\
        \textbf{Our DiGS} & SIREN wo n & SAL & NSP
        \end{tabular}
        \caption{Results for surface reconstruction on ShapeNet~\cite{chang2015shapenet}.}
        \label{fig:shapenet-results}
\end{figure}

\textbf{ShapeNet.}
The ShapeNet dataset contains 3D CAD models of a diverse range of objects. These shapes often have internal structure, inconsistent normals, and non-manifold meshes. We use the preprocessing and split of Williams~et~al.~\cite{williams2020neural_splines}, who evaluate on 20 shapes in each of 13 categories, and preprocess to make the normals consistent and internal structure to be manifold meshes. 
The results are reported in \tabref{tab:results:shapenet-summary} and we provide visualizations in \figref{fig:shapenet-results}. When adding normal supervision to DiGS, our method has a much better median squared Chamfer distance and median IoU among the 260 shapes.  We attribute this to our divergence term smoothing out the space, which does well on shapes without much internal structure. However for some shapes with internal structure (\eg, a loudspeaker with components inside or a sofa with structural beams) we get significant internal ghost geometry. As a result our mean squared Chamfer distance, while competitive with the other methods, is worse than methods such as SIREN. Note, however, that ths does not affect the IoU as much, our mean IoU is still better than other methods. When comparing without normals, DiGS has similar medians on both metrics to when normal supervision is added, however it has better means. We attribute this to having fewer internal ghost geometries when not attempting to fit normal vectors at internal points. 
Note that DiGS outperforms other methods that do not use normal supervision on both metrics. The improvement in IoU is significant, which we attribute to other methods failing to contain ghost geometry and being inconsistent with what is inside/outside the shape. On the other hand, DiGS appropriately deals with both of these challenges by its structured training procedure.

\textbf{Scene reconstruction.}
Qualitative results for scene reconstruction are presented in \figref{fig:scene_recon_qual_results}. Here we use eight layers with 512 units and train on the scene from Sitzmann~et~al.~\cite{sitzmann2020siren} which includes 10M oriented points. 
% for 30K iteration, sampling 5K random points in each iteration. 
We train SIREN with and without the normal constraint, and the proposed DiGS method.  In this experiment we did not use a geometric initialization because, unlike the shape reconstruction task, the target surface is vastly different than a sphere, and instead only increase the frequencies. When training SIREN without normal supervision we observe many ghost geometries (SIREN wo n). On the other hand, DiGS is able to reconstruct the scene without these, though it produces a very smooth result. This is desirable in most planar regions, e.g., ceiling, floor and table, however, this trait has the drawback of smoothing out fine details, e.g., sofa legs and picture frame. 
\begin{figure}
     \centering
     \begin{tabular}{c c c }
         \includegraphics[width=0.25\linewidth]{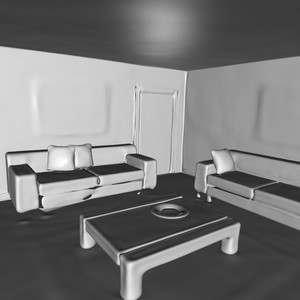} 
         &
         \includegraphics[width=0.25\linewidth]{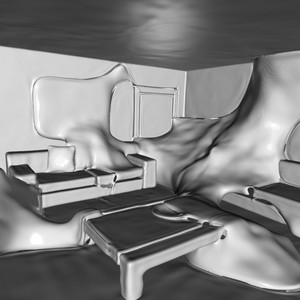}
         & 
         \includegraphics[width=0.25\linewidth]{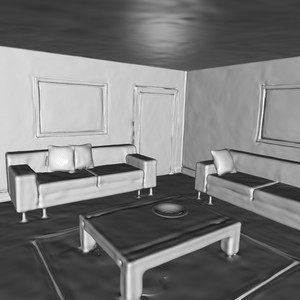} \\
         \textbf{Our DiGS} & SIREN wo n  & SIREN \\
     \end{tabular}
        \caption{Qualitative results of scene reconstruction. SIREN (right) provides high level of detail, however, when normal vectors are not available (center) the proposed divergence constraint (left) demonstrates a significant improvement.}
        \label{fig:scene_recon_qual_results}
\end{figure}
\begin{table}[]\scriptsize\setlength{\tabcolsep}{3pt}
    \centering
    % \resizebox{\textwidth}{!}{
    \begin{tabular}{c c c c c c c c c}
        \toprule
             &\multicolumn{2}{c}{$d_{\vec{C}}$(reg, recon)} &\multicolumn{2}{c}{$d_{\vec{C}}$(recon, reg)} &\multicolumn{2}{c}{$d_{\vec{C}}$(scan, recon)} 
             &\multicolumn{2}{c}{$d_{\vec{C}}$(recon, scan)}\\
             & Mean & Median & Mean & Median & Mean & Median & Mean & Median\\ 
         \hline
            % SAL (SALD reports)* & 0.418 & 0.328  & 0.344 & 0.256 & 0.429 & 0.246 & - & -\\
            % IGR (SALD reports)* & 0.276 & 0.187  & 3.806 & 3.627 & 0.241 & 0.110 & - & -\\
            % SALD (SALD reports)* & 0.428 & 0.346  & 0.489 & 0.362 & 0.397 & 0.279 & - & -\\
        % \hline
            IGR \cite{gropp2020igr}  & 1.053 & 0.509  & 4.916 & 0.540 & 1.054 & 0.509 & 4.916 & 0.540 \\
            DiGS + n          & \textbf{0.568} & \textbf{0.458}  & \textbf{1.834} & \textbf{0.461} & \textbf{0.568} & \textbf{0.458} & \textbf{1.834} & \textbf{0.461} \\ 
            \hline
            IGR wo n        & 3.745 & 2.689 & \textbf{12.149} & \textbf{9.027} & 3.745 & 2.687 & \textbf{12.147} & \textbf{9.026} \\ 
            Our DiGS            & \textbf{0.856}  & \textbf{0.707} & 12.318 & 9.202 & \textbf{0.856} & \textbf{0.707} & 12.319 & 9.204 \\ 
        \bottomrule
    \end{tabular}
    % }
    \caption{Quantitative results on DFaust~\cite{bogo2017dynamic}. We compare the mean and median of the one-sided Chamfer distances (reported as $\times 10^2$) between the ground truth registration meshes (reg), reconstructions (recon) and raw input scans (scan). 
    % * Metrics reported by \cite{Atzmon_2020_CVPR_SALD}. Code was not provided, but on inspection of code from \cite{atzmon2020sal} we noticed anomalies in their implementation of the Chamfer distance. 
    % Results reported were obtained from our own implementation to be released upon acceptance. 
}
    \label{tab:dfaust-shapespace}
\end{table}
\begin{figure} \scriptsize
     \centering 
     \setlength\tabcolsep{0pt}
     \renewcommand{\arraystretch}{0.0}
    %  \resizebox{\textwidth}{!}{
        % trim={<left> <lower> <right> <upper>}
        \begin{tabular}{c c c c c} 
        \includegraphics[width=0.19\linewidth, trim=94pt 67pt 62pt 65pt, clip]{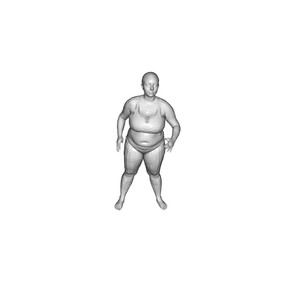}&
        \includegraphics[width=0.19\linewidth, trim=94pt 67pt 62pt 65pt, clip]{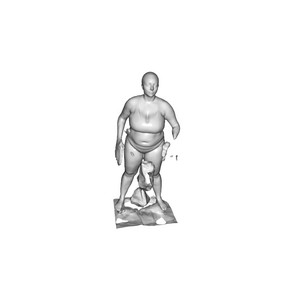}&
        \includegraphics[width=0.19\linewidth, trim=94pt 67pt 62pt 65pt, clip]{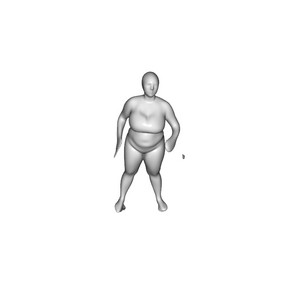}&
        \includegraphics[width=0.19\linewidth, trim=94pt 67pt 62pt 65pt, clip]{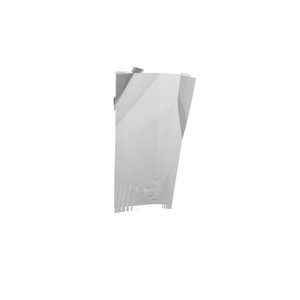}&
        \includegraphics[width=0.19\linewidth, trim=94pt 67pt 62pt 65pt, clip]{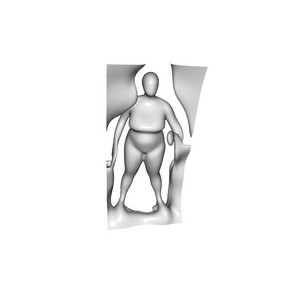}\\
        \includegraphics[width=0.19\linewidth, trim=94pt 67pt 62pt 65pt, clip]{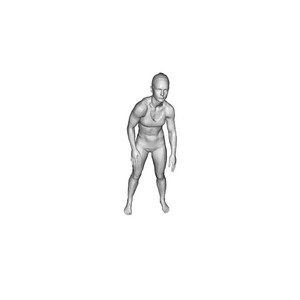}&
        \includegraphics[width=0.19\linewidth, trim=94pt 67pt 62pt 65pt, clip]{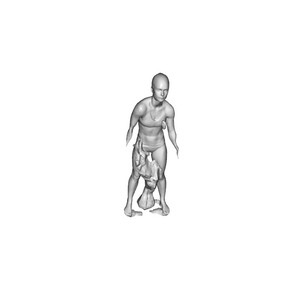}&
        \includegraphics[width=0.19\linewidth, trim=94pt 67pt 62pt 65pt, clip]{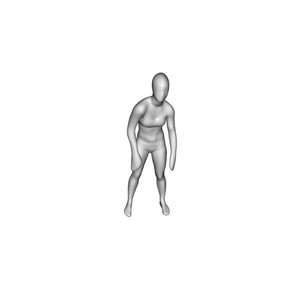}&
        \includegraphics[width=0.19\linewidth, trim=94pt 67pt 62pt 65pt, clip]{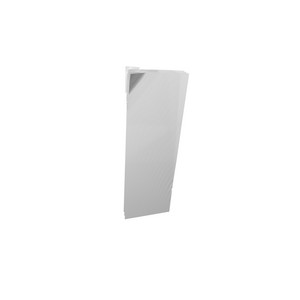}&
        \includegraphics[width=0.19\linewidth, trim=94pt 67pt 62pt 65pt, clip]{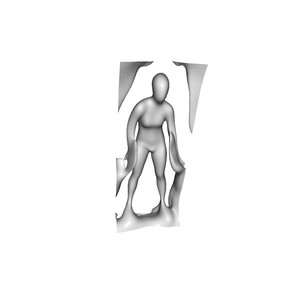}\\
        Ground Truth & IGR & \textbf{Our DiGS} + n & IGR wo n & \textbf{Our DiGS} \\
     \end{tabular}
    % }
    \caption{Qualitative results for the shape space experiment on the DFaust dataset \cite{bogo2017dynamic}. Each row is a single pose of a different human from the test set.}
    \label{fig:reduced_dfaust_shapespace_comparisons}
\end{figure}

\subsection{Shape Space Learning}
Shape space experiments requires training a single model to learn to represent multiple shapes from a class of related shapes. We use the DFAUST dataset \cite{bogo2017dynamic} which consists of ${\sim}40k$ raw human scans of ten humans at different time points during multiple types of activities. The scans are noisy and often incomplete (contains missing object surfaces). The dataset also provides ground truth registrations for each scan. Following IGR's \cite{gropp2020igr} setup, we use their 75\%/25\% train-test split, and we use DiGS as an autodecoder \cite{park2019deepsdf}. Thus at training time, multiple scans of the humans are learnt with a separate latent vector for each pose, and at test time, unseen poses of those same humans are reconstructed by jointly estimating a suitable latent vector. Note this is a much harder problem than surface reconstruction, the model has to learn to be able to fit to multiple shapes given different (learnable) latent codes, stretching the test of the model's capacity.

Table~\ref{tab:dfaust-shapespace} shows the quantitative results of one-sided Chamfer distances\footnote{For IGR we use the code and weights provided by \cite{gropp2020igr}. Note that some methods report the squared Chamfer~\cite{atzmon2020sal}. After correspondence with the authors of SAL and SALD~\cite{atzmon2020sal, Atzmon_2020_CVPR_SALD}, we were unable to reproduce their results and therefore these baseline are omitted.}.
In particular, for registration \& scan to reconstruction DiGS+n has a lower mean than IGR, showing it fits better to the ground truth and input surfaces and does not miss ground truth regions. For reconstruction to registration \& scan the mean distances increase for both IGR and DiGS+n, indicating that both models create ghost geometry, however DiGS+n creates much less. Both of these are demonstrated in Figure~\ref{fig:reduced_dfaust_shapespace_comparisons}: despite IGR sometimes having more detail (e.g., the face) it has multiple ghost geometries and significant missing parts (e.g., forearms). DiGS+n also achieves slightly smaller median distances.

For the methods without normals supervision, DiGS clearly outperforms IGR wo n, the latter is not able to converge. This can be seen in Figure~\ref{fig:reduced_dfaust_shapespace_comparisons}, where the reconstructions do not even resemble the initial humans. DiGS on the other hand captures the human shapes, but is oversmoothed and has large ghost geometries framing the humans. Looking at Table~\ref{tab:dfaust-shapespace} the means and medians for registration \& scan to reconstruction are still quite small, showing that despite oversmoothing the detail DiGS learns to fit to the unseen test surfaces quite well. Due to the ghost geometry, for the reconstruction to registration \& scan distances DiGS has very large values, similar to IGR wo n.

% \textbf{Human Shape Reconstruction}
% We report results on the DFaust dataset \cite{bogo2017dynamic}, which contains high resolution, raw human scans of 10 humans doing multiple activities. We provide more details on the dataset and results on shape space learning, where reconstruction is generalised to unseen point clouds \cite{atzmon2020sal,gropp2020igr,Atzmon_2020_CVPR_SALD}, in the supplemental material. \figref{fig:dfaust_comparisons} shows surface reconstruction results on a single human shape. Both IGR and DiGS use a geometrically guided initialization, and thus have better initial shapes during training. However, SIREN is able to outperform IGR in one-sided Chamfer distance (but not two-sided due to ghost geometries). DiGS+n on the other hand performs better than both, and further does not have any ghost geometries. When comparing the methods without normals, SIREN is not able to learn a meaningful shape, while IGR yields a smoothed shape. DiGS manages to learn the shape quickly and with much more details, but also has ghost geometries. 

% \input{assets/figures/dfaust_qualitative}

\subsection{Limitations}
DiGS is mainly limited in two aspects: (1) capturing very thin structures, e.g., the left sofa's legs in \figref{fig:scene_recon_qual_results}, and (2) smoothing effects, e.g., the pictures on the walls in \figref{fig:scene_recon_qual_results}. This is expected since these regions contain only very few points and without the normal vector information, uncovering the underlying surface is more challenging.

\section{Conclusion}
\label{Sec:Conclusions}
% We introduce DiGS, a Divergence Guided Shape implicit neural representation approach for raw unoriented point clouds without any pre-processing. Additionally, we derive a geometrically motivated initialization for sinusoidal representation networks to approximate the signed distance function to a sphere while preserving high frequencies. Finally, we demonstrate that DiGS has the ability to reconstruct shapes of high fidelity with some limitations of smoothing and thin features reconstruction. 

We introduce DiGS, a divergence guided shape implicit neural representation approach for raw unoriented point clouds without any pre-processing. Additionally, we derive a geometrically motivated initialization for sinusoidal representation networks while preserving high frequencies. Finally, we demonstrate that DiGS has the ability to reconstruct shapes of high fidelity with some limitations of smoothing and thin features reconstruction. We report state of the art results compared to other methods that do not use normal supervision, and show that our method is comparative to methods that do use such supervision.

All existing methods, including ours, struggle with missing data and thin features. Future work can explore extensions to use local self similarity to deal better with these regions. Point cloud density is also an important factor in the representation power and additional work can be done to mitigate density effects. 
% Additionally, computing the divergence is time consuming and it would be valuable to explore potential speedups as done by Tretschk \etal~\cite{tretschk2021non}.
Lastly, shape representation networks are highly dependant on the initialization and further work should be done to explore this direction.

% \section{Broader impact}
% \label{Sec:broader_impact}
\textbf{Potential societal impact.}
The proposed DiGS approach enables accurate representation of 3D shapes from 3D point cloud data in a deep learning framework. Many down stream tasks may be enabled by DiGS, including avatar creation and computer aided design (CAD). These applications may be leveraged for negative and positive outcomes. For example, DiGS may be extended for generative tasks and enable novel approaches for shape generation. This has potential misuses including digital impersonation without consent and unauthorized reproduction of mechanical designs. This ties into DeepFakes which were discussed in depth in a recent review on neural rendering \cite{tewari2020state}.  
\noindent\textbf{Acknowledgements}: This work was funded by The European Union’s Horizon 2020 research and innovation programme under the Marie Sklodowska-Curie grant agreement No 893465. 

%%%%%%%%% REFERENCES
\nocite{*}
{
    % \clearpage
    \small
    \bibliographystyle{ieee_fullname}
    \bibliography{long,references}

\begin{thebibliography}{10}\itemsep=-1pt

\bibitem{amenta2001power}
Nina Amenta, Sunghee Choi, and Ravi~Krishna Kolluri.
\newblock The power crust, unions of balls, and the medial axis transform.
\newblock {\em Computational Geometry}, 19:127--153, 2001.

\bibitem{atzmon2020sal}
Matan Atzmon and Yaron Lipman.
\newblock {SAL}: {Sign} {Agnostic} {Learning} of shapes from raw data.
\newblock In {\em Proc.~of the {IEEE} Conference on Computer Vision and Pattern
  Recognition ({CVPR})}, pages 2565--2574, 2020.

\bibitem{Atzmon_2020_CVPR_SALD}
Matan Atzmon and Yaron Lipman.
\newblock {SALD}: {Sign} {Agnostic} {Learning} with {Derivatives}.
\newblock In {\em Proc.~of the International Conference on Learning
  Representations ({ICLR})}, 2021.

\bibitem{azinovic2021neural}
Dejan Azinovic, Ricardo Martin-Brualla, Dan~B Goldman, Matthias Nie{\ss}ner,
  and Justus Thies.
\newblock Neural {RGB-D} surface reconstruction.
\newblock {\em arXiv preprint arXiv:2104.04532}, 2021.

\bibitem{ben2020deepfit}
Yizhak Ben-Shabat and Stephen Gould.
\newblock {DeepFit}: {3D} surface fitting via neural network weighted least
  squares.
\newblock In {\em Proc.~of the European Conference on Computer Vision
  ({ECCV})}, pages 20--34. Springer, 2020.

\bibitem{ben2019nestinet}
Yizhak Ben-Shabat, Michael Lindenbaum, and Anath Fischer.
\newblock Nesti-net: Normal estimation for unstructured {3D} point clouds using
  convolutional neural networks.
\newblock In {\em Proc.~of the {IEEE} Conference on Computer Vision and Pattern
  Recognition ({CVPR})}, pages 10112--10120, 2019.

\bibitem{berger2013benchmark}
Matthew Berger, Joshua~A Levine, Luis~Gustavo Nonato, Gabriel Taubin, and
  Claudio~T Silva.
\newblock A benchmark for surface reconstruction.
\newblock {\em {ACM} Trans.~on Graphics ({ToG})}, 32:1--17, 2013.

\bibitem{berger2017survey}
Matthew Berger, Andrea Tagliasacchi, Lee~M Seversky, Pierre Alliez, Gael
  Guennebaud, Joshua~A Levine, Andrei Sharf, and Claudio~T Silva.
\newblock A survey of surface reconstruction from point clouds.
\newblock In {\em Computer Graphics Forum}, volume~36, pages 301--329. Wiley
  Online Library, 2017.

\bibitem{bernardini1999ball}
Fausto Bernardini, Joshua Mittleman, Holly Rushmeier, Claudio Silva, and
  Gabriel Taubin.
\newblock The ball-pivoting algorithm for surface reconstruction.
\newblock {\em {IEEE} Trans.~on {V}isualization and {C}omputer {G}raphics},
  5:349--359, 1999.

\bibitem{bogo2017dynamic}
Federica Bogo, Javier Romero, Gerard Pons-Moll, and Michael~J Black.
\newblock Dynamic {FAUST}: Registering human bodies in motion.
\newblock In {\em Proc.~of the {IEEE} Conference on Computer Vision and Pattern
  Recognition ({CVPR})}, pages 6233--6242, 2017.

\bibitem{bronstein2017geometric}
Michael~M Bronstein, Joan Bruna, Yann LeCun, Arthur Szlam, and Pierre
  Vandergheynst.
\newblock Geometric deep learning: Going beyond {E}uclidean data.
\newblock {\em {IEEE} Signal Processing Magazine}, 34:18--42, 2017.

\bibitem{carr2001reconstruction}
Jonathan~C Carr, Richard~K Beatson, Jon~B Cherrie, Tim~J Mitchell, W~Richard
  Fright, Bruce~C McCallum, and Tim~R Evans.
\newblock Reconstruction and representation of {3D} objects with radial basis
  functions.
\newblock In {\em Proc.~of the Conference on Computer Graphics and Interactive
  Techniques}, pages 67--76, 2001.

\bibitem{cazals2006delaunay}
Fr{\'e}d{\'e}ric Cazals and Joachim Giesen.
\newblock Delaunay triangulation based surface reconstruction.
\newblock In {\em Effective Computational Geometry for Curves and Surfaces},
  pages 231--276. Springer, 2006.

\bibitem{cazals2005estimating}
Fr{\'e}d{\'e}ric Cazals and Marc Pouget.
\newblock Estimating differential quantities using polynomial fitting of
  osculating jets.
\newblock {\em Computer Aided Geometric Design}, 22:121--146, 2005.

\bibitem{chang2015shapenet}
Angel~X Chang, Thomas Funkhouser, Leonidas Guibas, Pat Hanrahan, Qixing Huang,
  Zimo Li, Silvio Savarese, Manolis Savva, Shuran Song, Hao Su, et~al.
\newblock {ShapeNet}: An information-rich {3D} model repository.
\newblock {\em arXiv preprint arXiv:1512.03012}, 2015.

\bibitem{chen2019learning}
Zhiqin Chen and Hao Zhang.
\newblock Learning implicit fields for generative shape modeling.
\newblock In {\em Proc.~of the {IEEE} Conference on Computer Vision and Pattern
  Recognition ({CVPR})}, pages 5939--5948, 2019.

\bibitem{gropp2020igr}
Amos Gropp, Lior Yariv, Niv Haim, Matan Atzmon, and Yaron Lipman.
\newblock {Implicit} {Geometric} {Regularization} for learning shapes.
\newblock In {\em Proc.~of the International Conference on Machine Learning
  ({ICML})}, volume 119, pages 3789--3799. PMLR, 2020.

\bibitem{groueix2018papier}
Thibault Groueix, Matthew Fisher, Vladimir~G Kim, Bryan~C Russell, and Mathieu
  Aubry.
\newblock A papier-m{\^a}ch{\'e} approach to learning {3D} surface generation.
\newblock In {\em Proc.~of the {IEEE} Conference on Computer Vision and Pattern
  Recognition ({CVPR})}, pages 216--224, 2018.

\bibitem{guerrero2018pcpnet}
Paul Guerrero, Yanir Kleiman, Maks Ovsjanikov, and Niloy~J Mitra.
\newblock {PCPNet}: Learning local shape properties from raw point clouds.
\newblock In {\em Computer Graphics Forum}, volume~37, pages 75--85. Wiley
  Online Library, 2018.

\bibitem{jiang2020local}
Chiyu Jiang, Avneesh Sud, Ameesh Makadia, Jingwei Huang, Matthias Nie{\ss}ner,
  Thomas Funkhouser, et~al.
\newblock Local implicit grid representations for {3D} scenes.
\newblock In {\em Proc.~of the {IEEE} Conference on Computer Vision and Pattern
  Recognition ({CVPR})}, pages 6001--6010, 2020.

\bibitem{jiang2020sdfdiff}
Yue Jiang, Dantong Ji, Zhizhong Han, and Matthias Zwicker.
\newblock {SDFDiff}: Differentiable rendering of signed distance fields for
  {3D} shape optimization.
\newblock In {\em Proc.~of the {IEEE} Conference on Computer Vision and Pattern
  Recognition ({CVPR})}, pages 1251--1261, 2020.

\bibitem{kazhdan2006poisson}
Michael Kazhdan, Matthew Bolitho, and Hugues Hoppe.
\newblock Poisson surface reconstruction.
\newblock In {\em Proc.~of the {Eurographics} Symposium on Geometry Processing
  ({SGP})}, volume~7, 2006.

\bibitem{kazhdan2013screened_poisson}
Michael Kazhdan and Hugues Hoppe.
\newblock Screened poisson surface reconstruction.
\newblock {\em {ACM} Trans.~on Graphics ({ToG})}, 32:1--13, 2013.

\bibitem{kolluri2004spectral}
Ravikrishna Kolluri, Jonathan~Richard Shewchuk, and James~F O'Brien.
\newblock Spectral surface reconstruction from noisy point clouds.
\newblock In {\em Proc.~of the {Eurographics} Symposium on Geometry Processing
  ({SGP})}, pages 11--21, 2004.

\bibitem{lenssen2020deepiterative}
Jan~Eric Lenssen, Christian Osendorfer, and Jonathan Masci.
\newblock Deep iterative surface normal estimation.
\newblock In {\em Proc.~of the {IEEE} Conference on Computer Vision and Pattern
  Recognition ({CVPR})}, pages 11247--11256, 2020.

\bibitem{lipman2021phase}
Yaron Lipman.
\newblock Phase transitions, distance functions, and implicit neural
  representations.
\newblock {\em arXiv preprint arXiv:2106.07689}, 2021.

\bibitem{lorensen1987marching}
William~E Lorensen and Harvey~E Cline.
\newblock Marching cubes: A high resolution {3D} surface construction
  algorithm.
\newblock In {\em Proc.~of the Intl. Conference on Computer Graphics and
  Interactive Techniques ({SIGGRAPH})}, volume~21, pages 163--169. ACM New
  York, NY, USA, 1987.

\bibitem{marsden2003vector}
Jerrold~E Marsden and Anthony Tromba.
\newblock {\em Vector calculus}.
\newblock Macmillan, 2003.

\bibitem{mescheder2019occupancy}
Lars Mescheder, Michael Oechsle, Michael Niemeyer, Sebastian Nowozin, and
  Andreas Geiger.
\newblock Occupancy networks: Learning {3D} reconstruction in function space.
\newblock In {\em Proc.~of the {IEEE} Conference on Computer Vision and Pattern
  Recognition ({CVPR})}, pages 4460--4470, 2019.

\bibitem{michalkiewicz2019implicit}
Mateusz Michalkiewicz, Jhony~K Pontes, Dominic Jack, Mahsa Baktashmotlagh, and
  Anders Eriksson.
\newblock Implicit surface representations as layers in neural networks.
\newblock In {\em Proc.~of the International Conference on Computer Vision
  ({ICCV})}, pages 4743--4752, 2019.

\bibitem{morse1954methods}
Philip~M Morse and Herman Feshbach.
\newblock Methods of theoretical physics.
\newblock {\em American Journal of Physics}, 22:410--413, 1954.

\bibitem{nagai2009smoothing}
Yukie Nagai, Yutaka Ohtake, and Hiromasa Suzuki.
\newblock Smoothing of partition of unity implicit surfaces for noise robust
  surface reconstruction.
\newblock In {\em Computer Graphics Forum}, volume~28, pages 1339--1348. Wiley
  Online Library, 2009.

\bibitem{ohtake2005sparse}
Yutaka Ohtake, Alexander Belyaev, and Marc Alexa.
\newblock Sparse low-degree implicit surfaces with applications to high quality
  rendering, feature extraction, and smoothing.
\newblock In {\em Proc.~of the {Eurographics} Symposium on Geometry Processing
  ({SGP})}, pages 149--158. Citeseer, 2005.

\bibitem{park2019deepsdf}
Jeong~Joon Park, Peter Florence, Julian Straub, Richard Newcombe, and Steven
  Lovegrove.
\newblock {DeepSDF}: Learning continuous signed distance functions for shape
  representation.
\newblock In {\em Proc.~of the {IEEE} Conference on Computer Vision and Pattern
  Recognition ({CVPR})}, pages 165--174, 2019.

\bibitem{peng2021shape}
Songyou Peng, Chiyu Jiang, Yiyi Liao, Michael Niemeyer, Marc Pollefeys, and
  Andreas Geiger.
\newblock {S}hape {A}s {P}oints: A differentiable poisson solver.
\newblock {\em arXiv preprint arXiv:2106.03452}, 2021.

\bibitem{peng2020convolutional_occ}
Songyou Peng, Michael Niemeyer, Lars Mescheder, Marc Pollefeys, and Andreas
  Geiger.
\newblock Convolutional occupancy networks.
\newblock In {\em Proc.~of the European Conference on Computer Vision
  ({ECCV})}, 2020.

\bibitem{sitzmann2020siren}
Vincent Sitzmann, Julien Martel, Alexander Bergman, David Lindell, and Gordon
  Wetzstein.
\newblock Implicit neural representations with periodic activation functions.
\newblock {\em Advances in Neural Information Processing Systems ({NeurIPS})},
  33, 2020.

\bibitem{sitzmann2019srns}
Vincent Sitzmann, Michael Zollh{\"o}fer, and Gordon Wetzstein.
\newblock {S}cene {R}epresentation {N}etworks: Continuous 3d-structure-aware
  neural scene representations.
\newblock In {\em Advances in Neural Information Processing Systems
  ({NeurIPS})}, 2019.

\bibitem{tancik2020fourfeat}
Matthew Tancik, Pratul~P. Srinivasan, Ben Mildenhall, Sara Fridovich-Keil,
  Nithin Raghavan, Utkarsh Singhal, Ravi Ramamoorthi, Jonathan~T. Barron, and
  Ren Ng.
\newblock Fourier features let networks learn high frequency functions in low
  dimensional domains.
\newblock {\em Advances in Neural Information Processing Systems ({NeurIPS})},
  2020.

\bibitem{tatarchenko2017octree}
Maxim Tatarchenko, Alexey Dosovitskiy, and Thomas Brox.
\newblock Octree generating networks: Efficient convolutional architectures for
  high-resolution {3D} outputs.
\newblock In {\em Proc.~of the International Conference on Computer Vision
  ({ICCV})}, pages 2088--2096, 2017.

\bibitem{tewari2020state}
Ayush Tewari, Ohad Fried, Justus Thies, Vincent Sitzmann, Stephen Lombardi,
  Kalyan Sunkavalli, Ricardo Martin-Brualla, Tomas Simon, Jason Saragih,
  Matthias Nie{\ss}ner, et~al.
\newblock State of the art on neural rendering.
\newblock In {\em Computer Graphics Forum}, volume~39, pages 701--727. Wiley
  Online Library, 2020.

\bibitem{urbach2020dpdist}
Dahlia Urbach, Yizhak Ben-Shabat, and Michael Lindenbaum.
\newblock {DPDist}: Comparing point clouds using deep point cloud distance.
\newblock In {\em Proc.~of the European Conference on Computer Vision
  ({ECCV})}, pages 545--560. Springer, 2020.

\bibitem{willians2021drps}
Christopher Williams and Kenneth Duru.
\newblock Provably stable full-spectrum dispersion relation preserving schemes.
\newblock {\em arXiv preprint arXiv:2110.04957}, 2021.

\bibitem{williams2019DGP}
Francis Williams, Teseo Schneider, Claudio Silva, Denis Zorin, Joan Bruna, and
  Daniele Panozzo.
\newblock {Deep} {Geometric} {Prior} for surface reconstruction.
\newblock In {\em Proc.~of the {IEEE} Conference on Computer Vision and Pattern
  Recognition ({CVPR})}, pages 10130--10139, 2019.

\bibitem{williams2020neural_splines}
Francis Williams, Matthew Trager, Joan Bruna, and Denis Zorin.
\newblock {Neural} {Splines}: Fitting {3D} surfaces with infinitely-wide neural
  networks.
\newblock In {\em Proc.~of the {IEEE} Conference on Computer Vision and Pattern
  Recognition ({CVPR})}, pages 9949--9958, 2021.

\bibitem{yifan2021geometry}
Wang Yifan, Lukas Rahmann, and Olga Sorkine-Hornung.
\newblock Geometry-consistent neural shape representation with implicit
  displacement fields.
\newblock {\em arXiv preprint arXiv:2106.05187}, 2021.

\end{thebibliography}
}

% --- supplementary material
\newpage
\setcounter{section}{0}
\renewcommand{\thesection}{\Alph{section}}

\twocolumn[\section*{\centering DiGS : Divergence guided shape implicit neural representation for unoriented point clouds}\vspace{0.5in}]

\section{Supplementary Material}
\label{Sec:appendix}
Here we provide supplementary material for the proposed divergence guided shape implicit neural representation. In \secref{appx:sec:theory} we discuss SDF properties, theory behind the proposed divergence  constraint, and a second order supervised constraint. In \secref{sec:appx:approach_init} we provide proofs and additional visualuzations for the proposed geometric initializations, as well as visualizations of our overall training procedure.  In \secref{sec:appx:additoinal_results} we provide additional experimental results. Finally, we provide a high resolution video showcasing the performance of our method in multiple scenarios at \href{https://drive.google.com/file/d/1CuP8KN93JpzWY-597g0945r5VCIEOi3D/view?usp=sharing}{URL}. 

%%%%%%%%%%%%%%%%%%%%%%%%%%%%%%%%%%%%%%%%%%%%%%%%%%%%%%%%%%%%%%%%%%%%%%%%%%%%%%%%%%

\subsection{SDF learning theory and the divergence term}
\label{appx:sec:theory}
\subsubsection{SDF properties}
As explained in \secref{Sec:approach_divergence}, current losses enforce two properties of SDFs, the function should be zero on the surface and the gradient of the norm should be one everywhere. In particular they are looking at two scalar fields related to the SDF function: the function itself, and the gradient norm field. On the other hand we also considered two more scalar fields, the divergence and curl of the gradient vector field of the SDF function. In \figref{fig:appx:gt-vis} we visualise the ground truth value of these four fields for three 2D shapes. We can see straight away that the curl is zero everywhere (as explained in \secref{Sec:approach_divergence}), and the gradient norm is one everywhere. The divergence on the other hand has very low magnitude in most areas, spiking sharply at points such as centres, skeletons or corners of the shape, and diffusing quickly from there. 

\begin{figure*}
     \centering
     \resizebox{\textwidth}{!}{
     \begin{tabular}{c c c c c}
         \rotatebox[origin=c]{90}{Circle}
         &
         \raisebox{-0.5\height}{\includegraphics[scale=0.2, trim=20pt 20pt 20pt 20pt, clip]{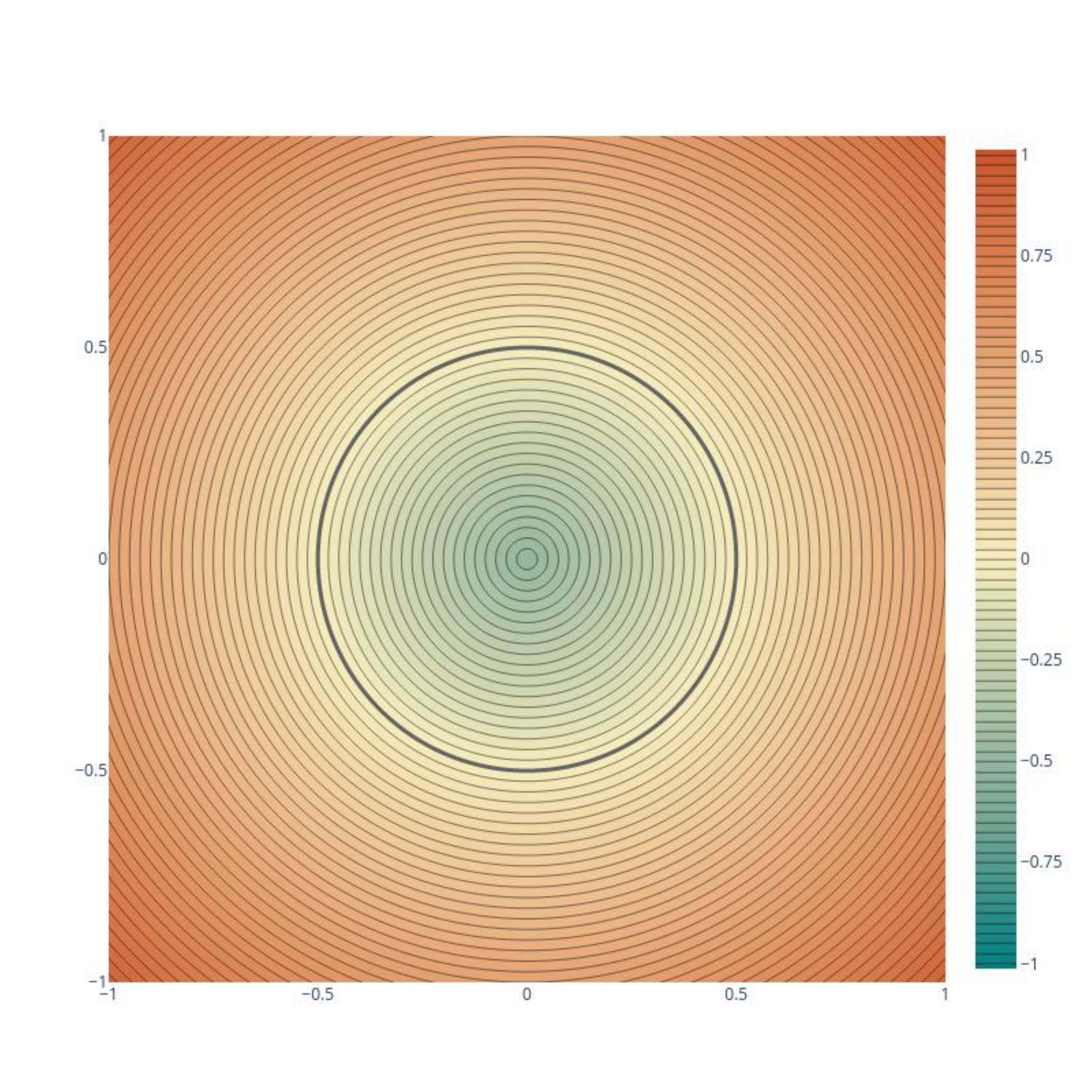}}
         &
         \raisebox{-0.5\height}{\includegraphics[scale=0.2, trim=20pt 20pt 20pt 20pt, clip]{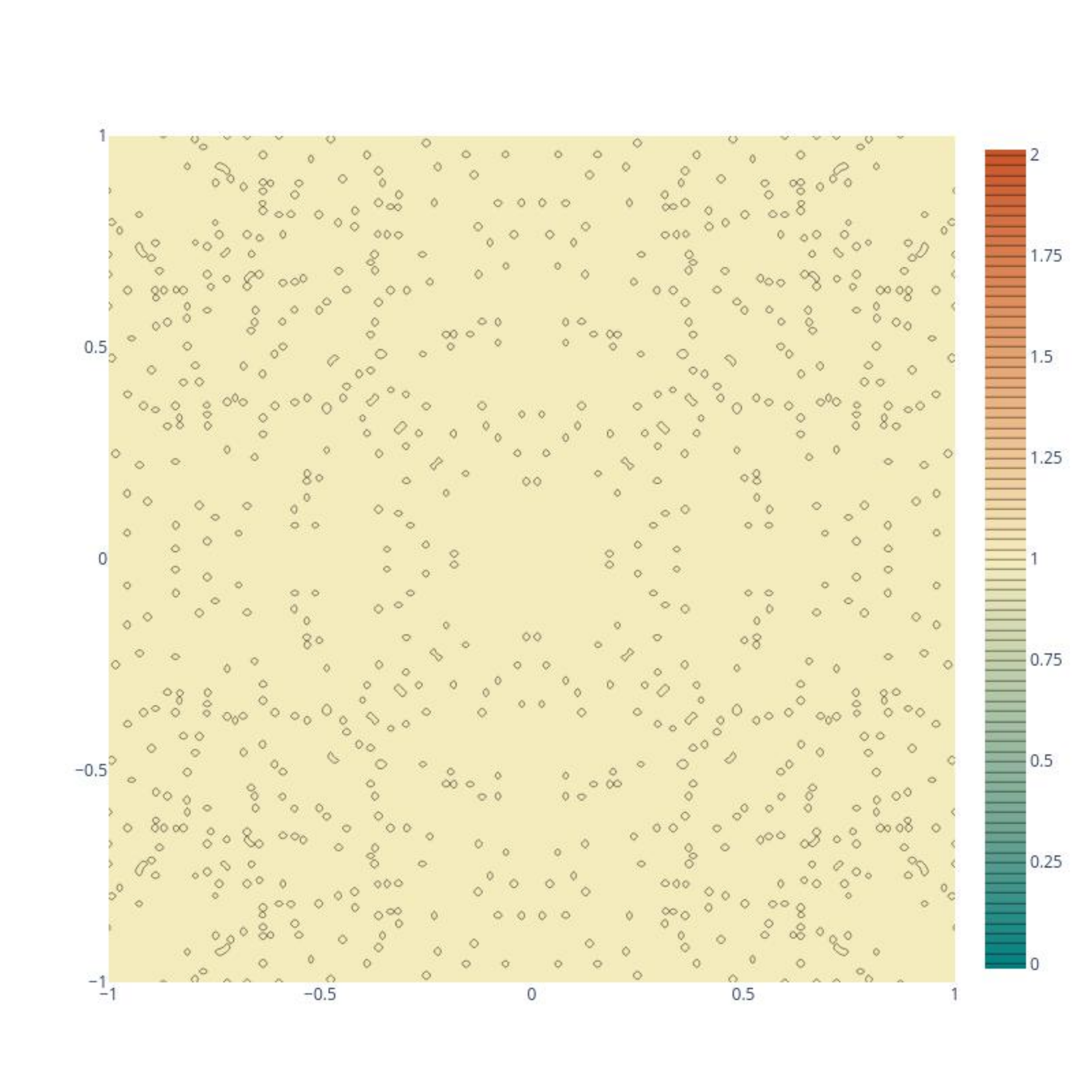}}
         &
         \raisebox{-0.5\height}{\includegraphics[scale=0.2, trim=20pt 20pt 20pt 20pt, clip]{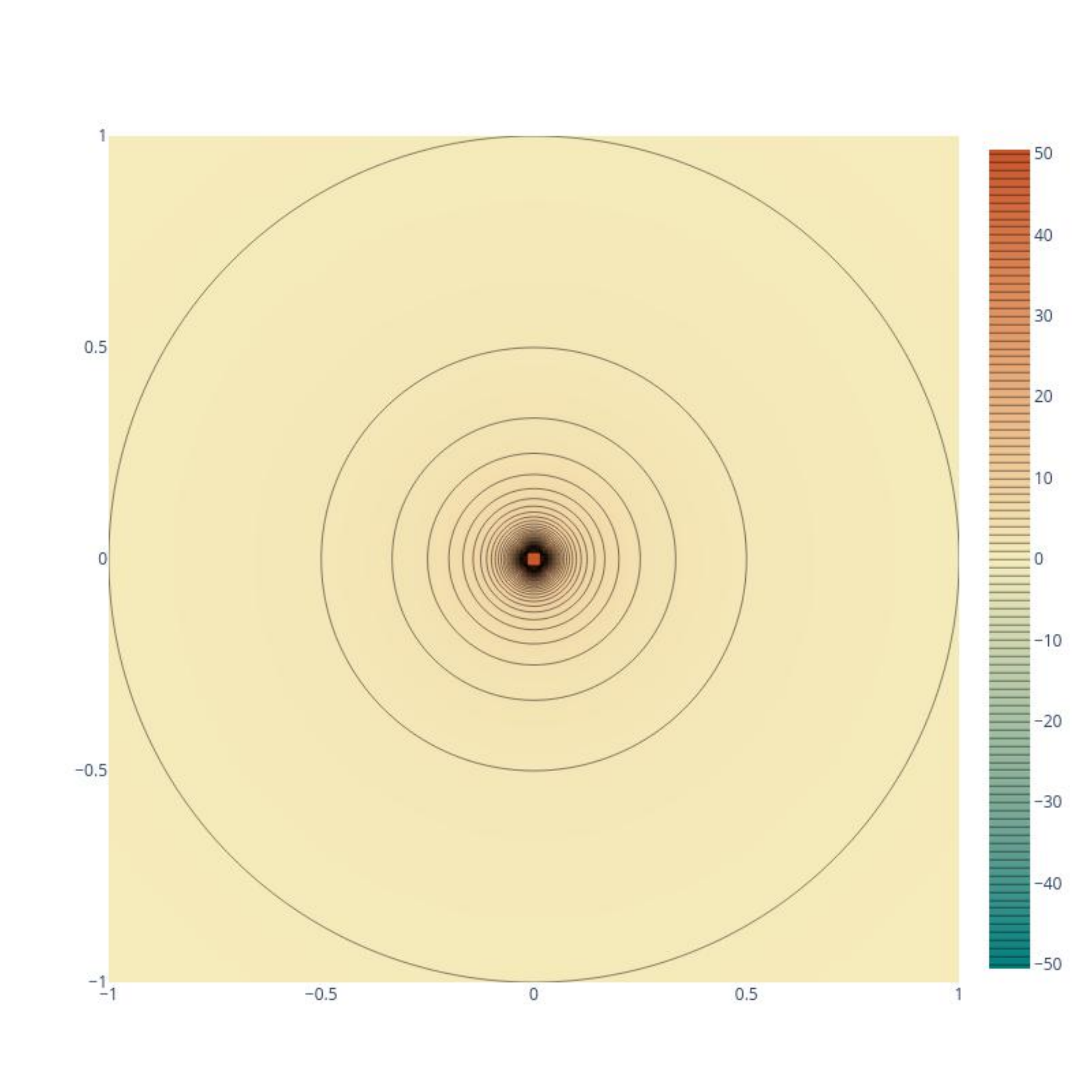}}
         &
         \raisebox{-0.5\height}{\includegraphics[scale=0.2, trim=20pt 20pt 20pt 20pt, clip]{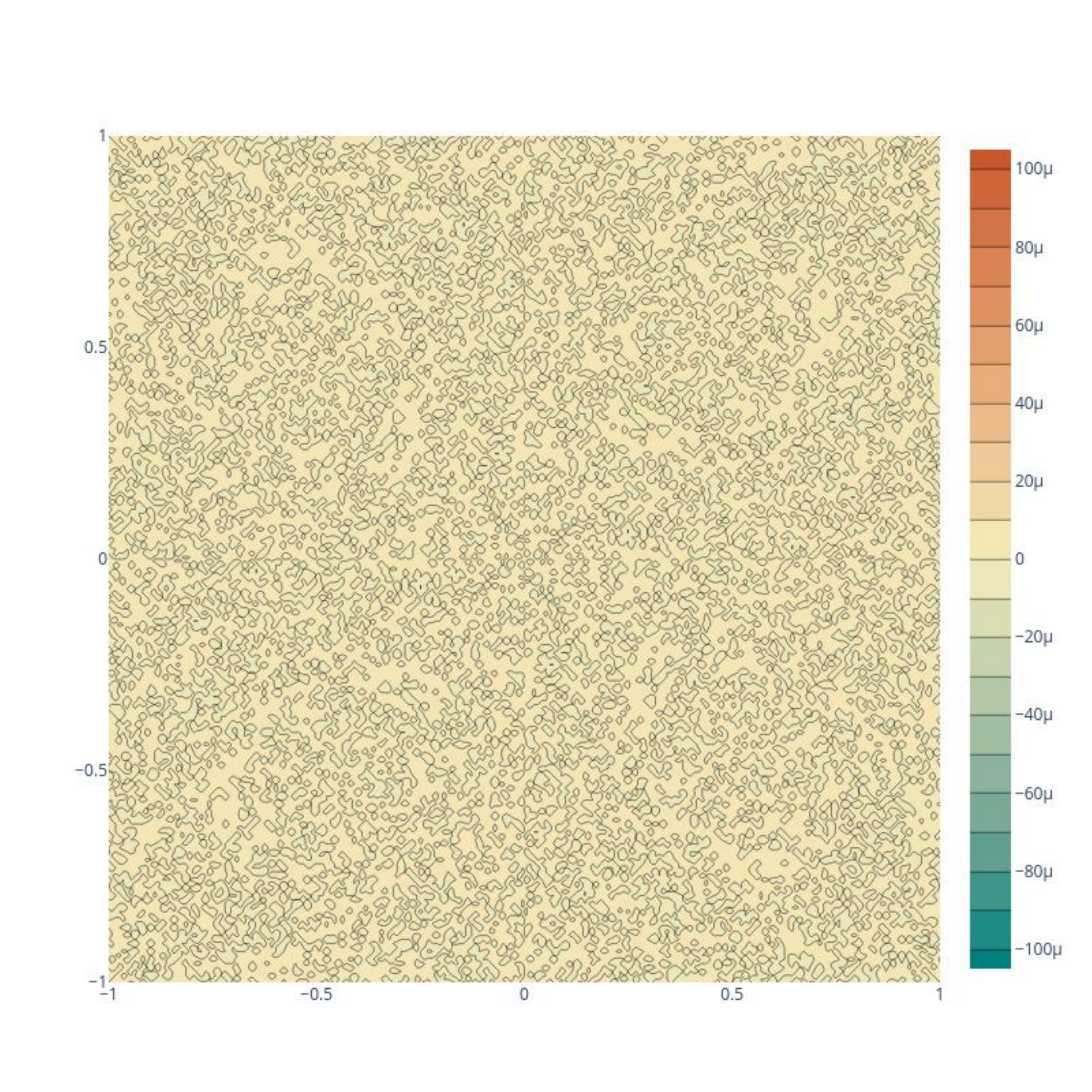}}\\
         \rule{0pt}{0.05ex} &&& \\  
         \rotatebox[origin=c]{90}{L}
         &
         \raisebox{-0.5\height}{\includegraphics[scale=0.2, trim=20pt 20pt 20pt 20pt, clip]{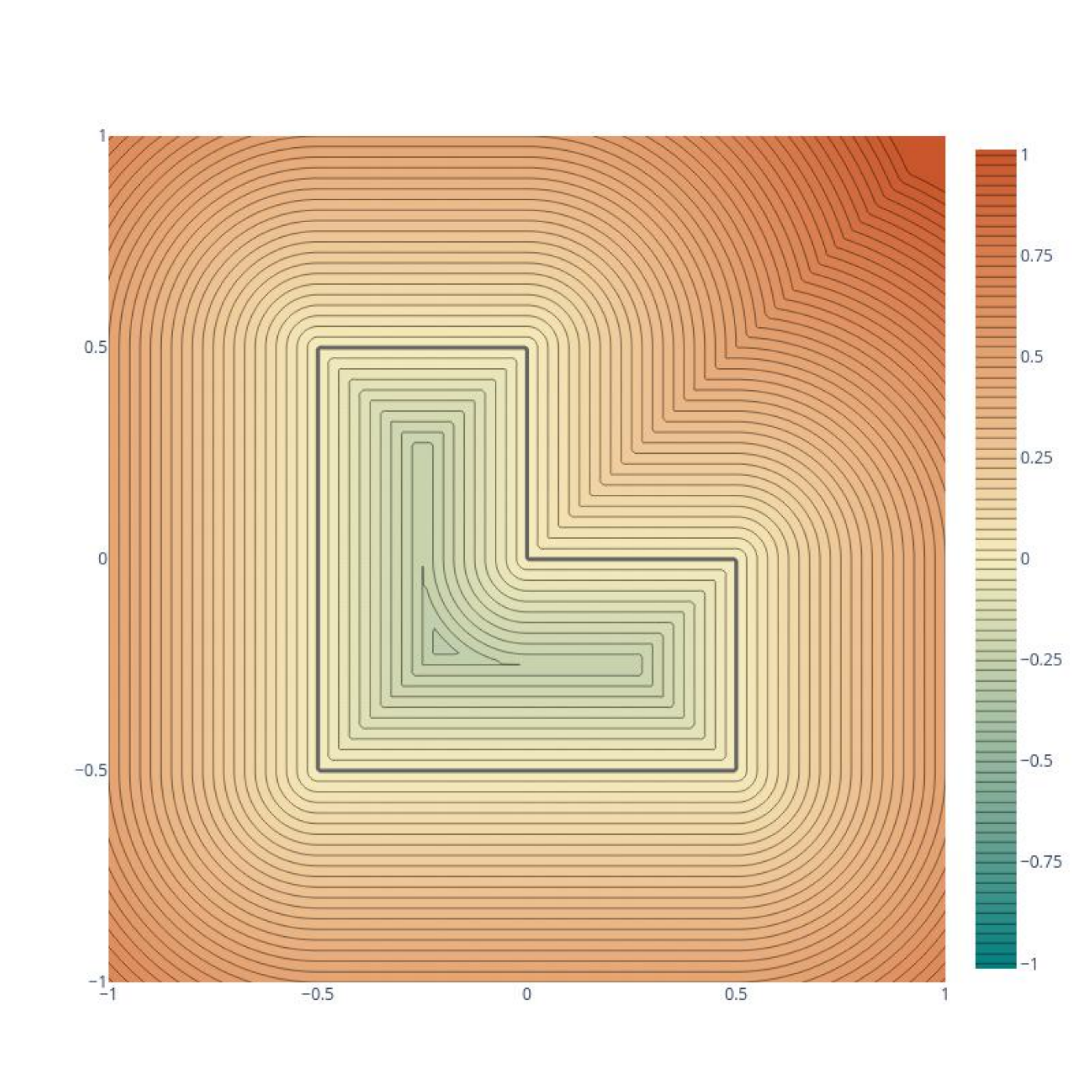}}
         &
         \raisebox{-0.5\height}{\includegraphics[scale=0.2, trim=20pt 20pt 20pt 20pt, clip]{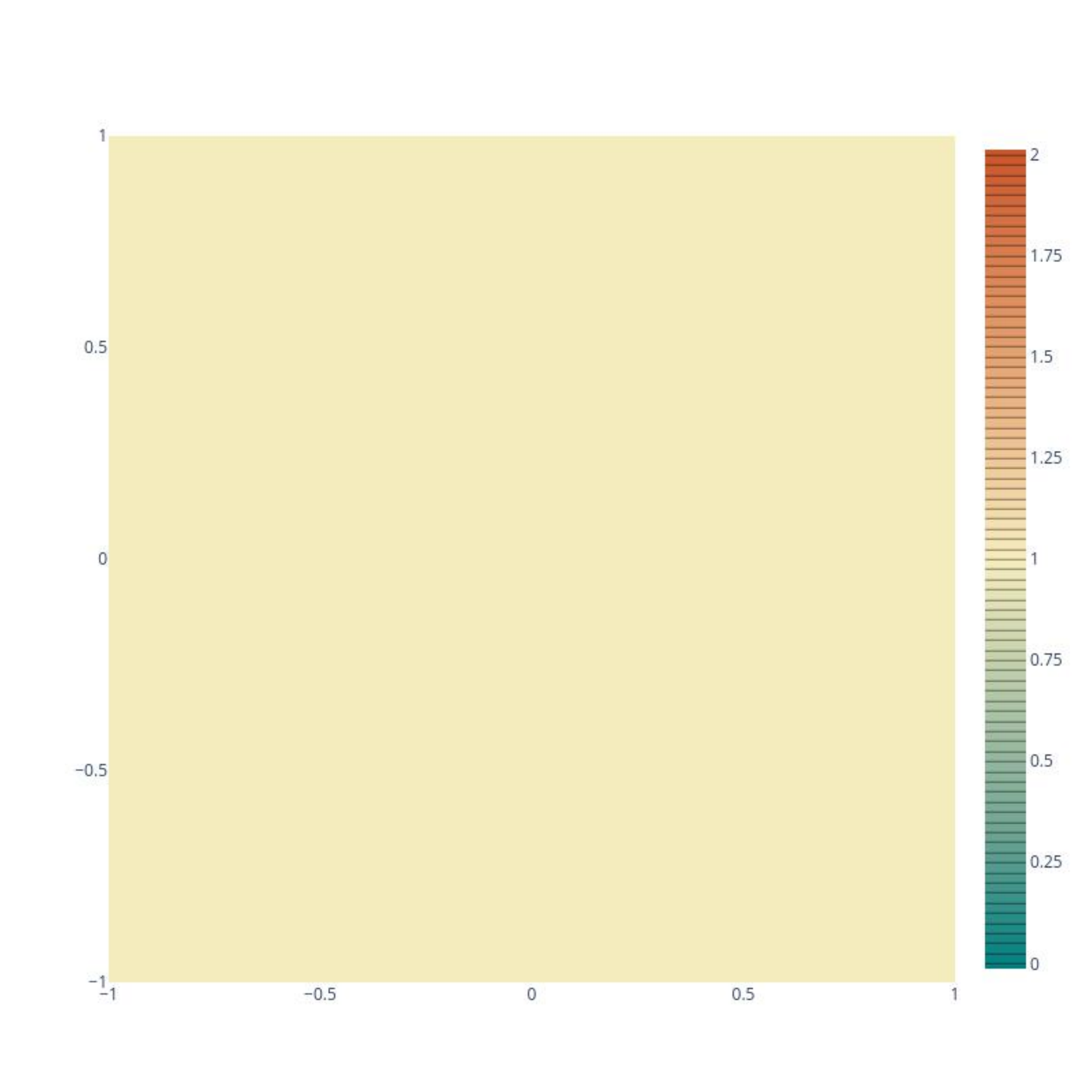}}
         &
         \raisebox{-0.5\height}{\includegraphics[scale=0.2, trim=20pt 20pt 20pt 20pt, clip]{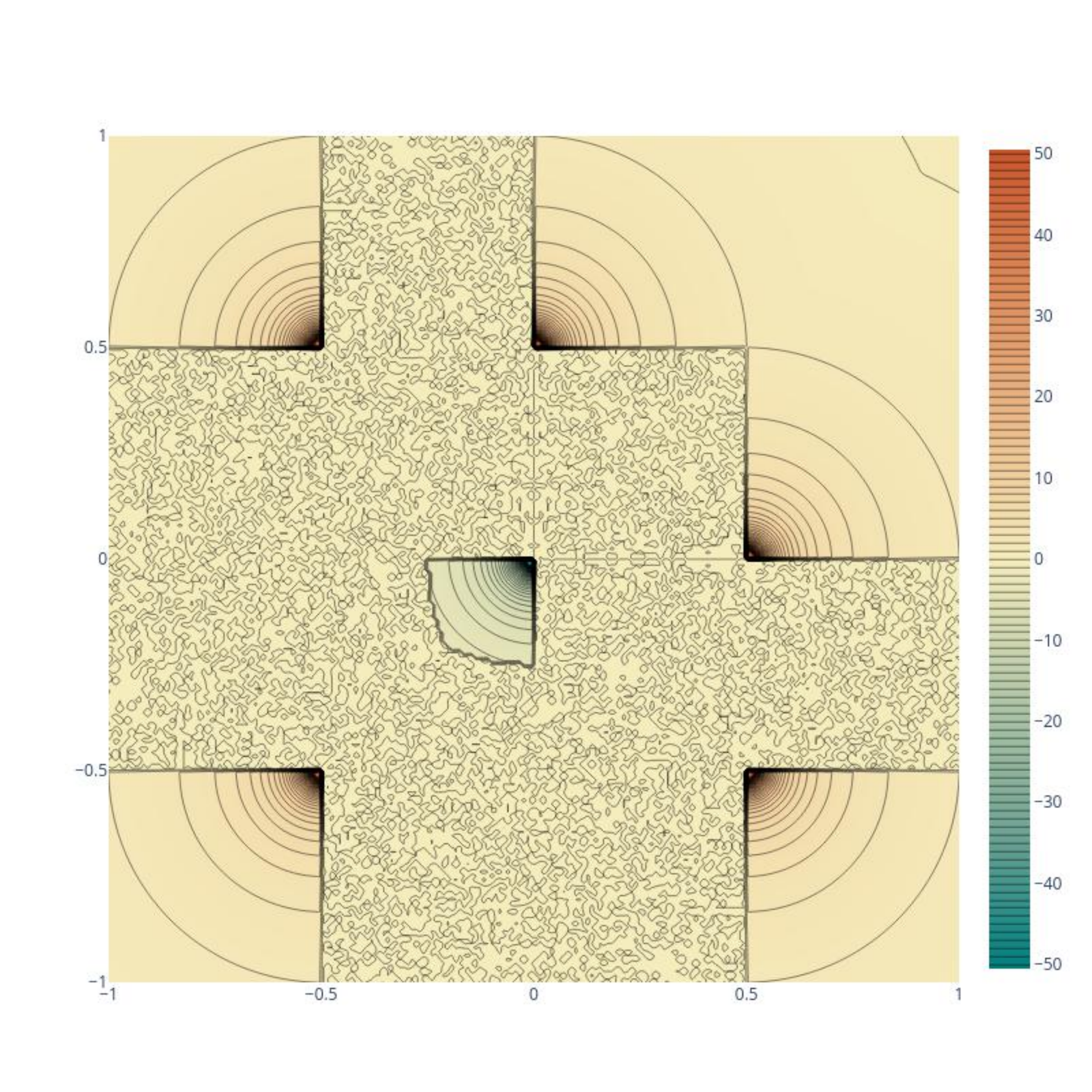}}
         &
         \raisebox{-0.5\height}{\includegraphics[scale=0.2, trim=20pt 20pt 20pt 20pt, clip]{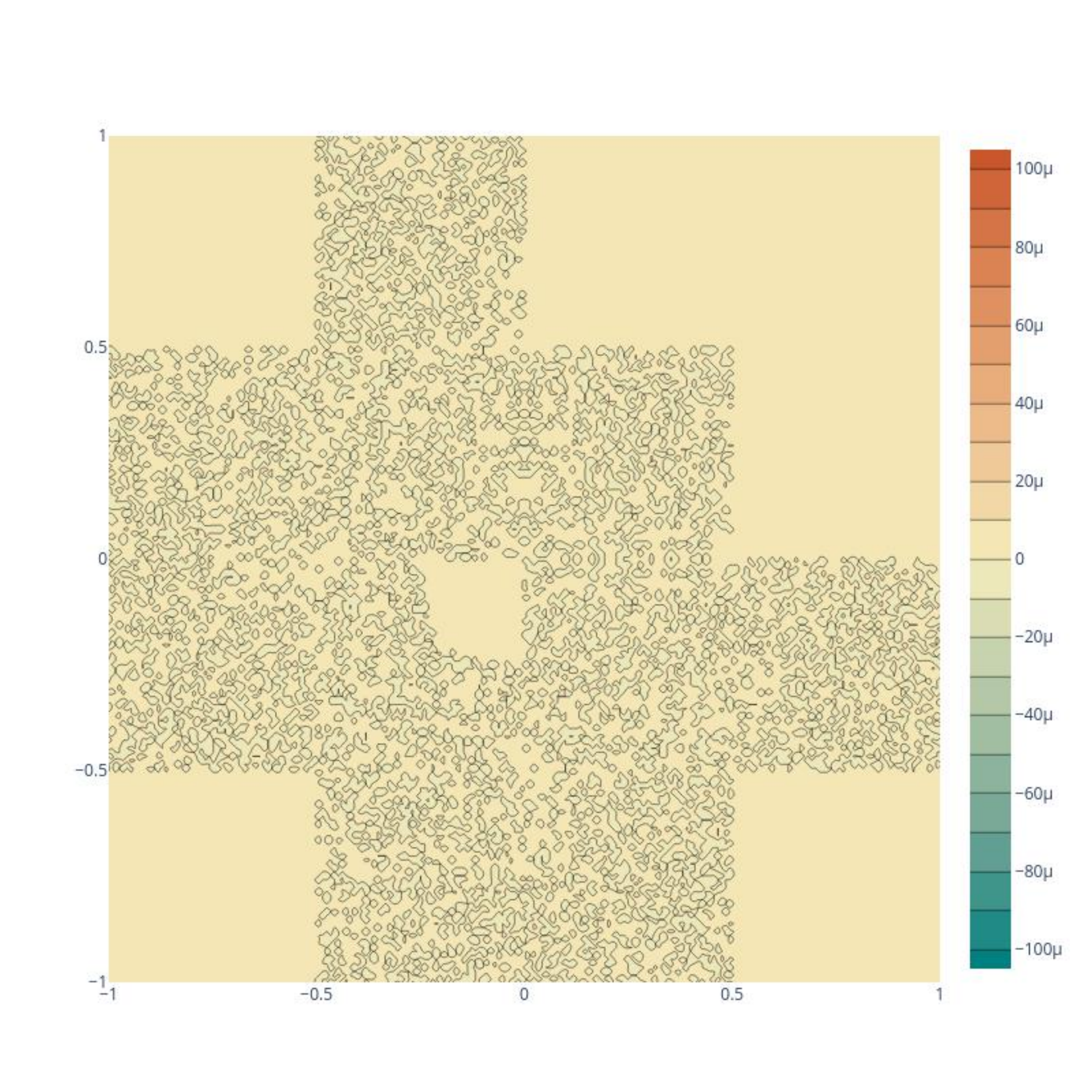}}\\
         \rule{0pt}{0.05ex} &&& \\         
         
         \rotatebox[origin=c]{90}{Snowflake}
         &
         \raisebox{-0.5\height}{\includegraphics[scale=0.2, trim=20pt 20pt 20pt 20pt, clip]{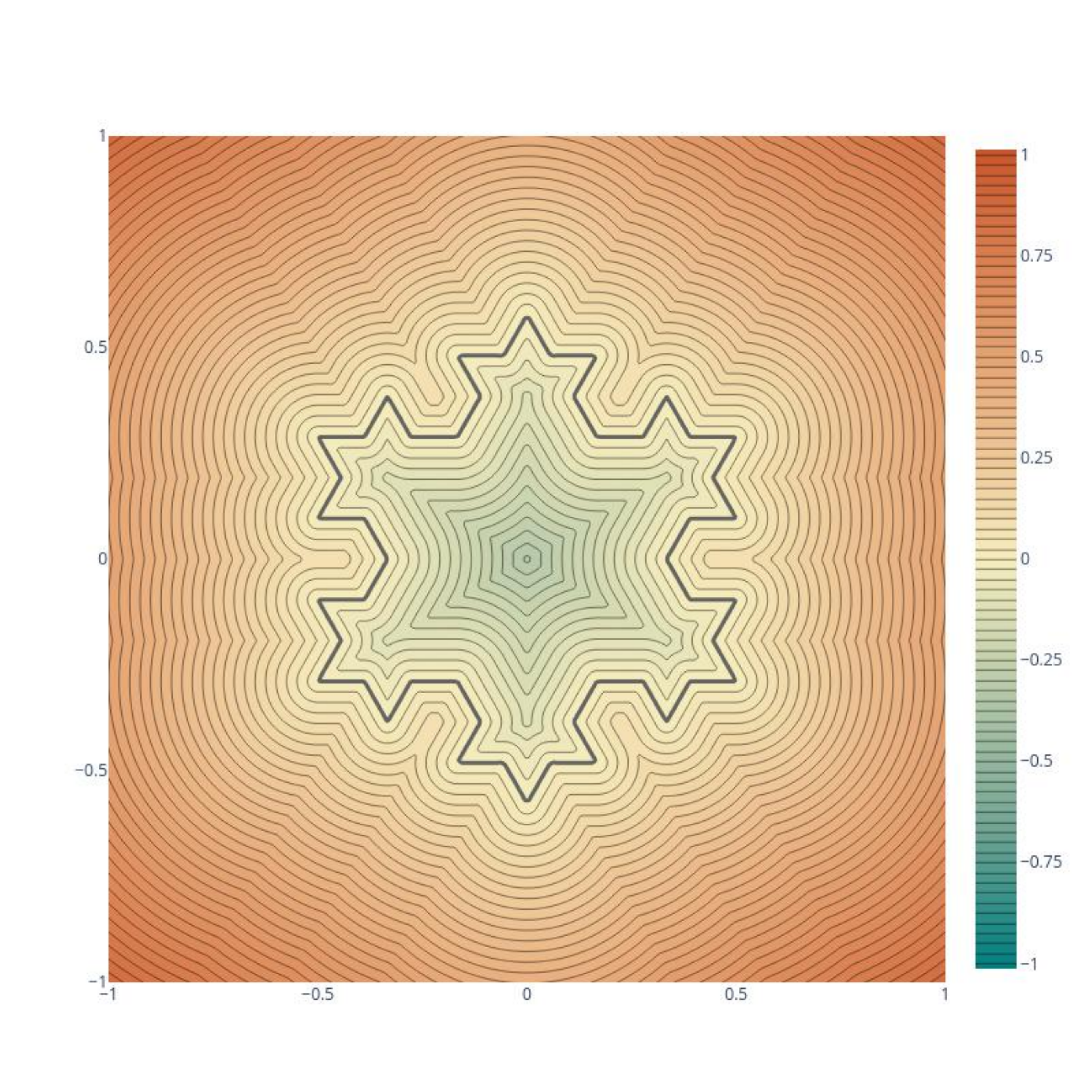}}
         &
         \raisebox{-0.5\height}{\includegraphics[scale=0.2, trim=20pt 20pt 20pt 20pt, clip]{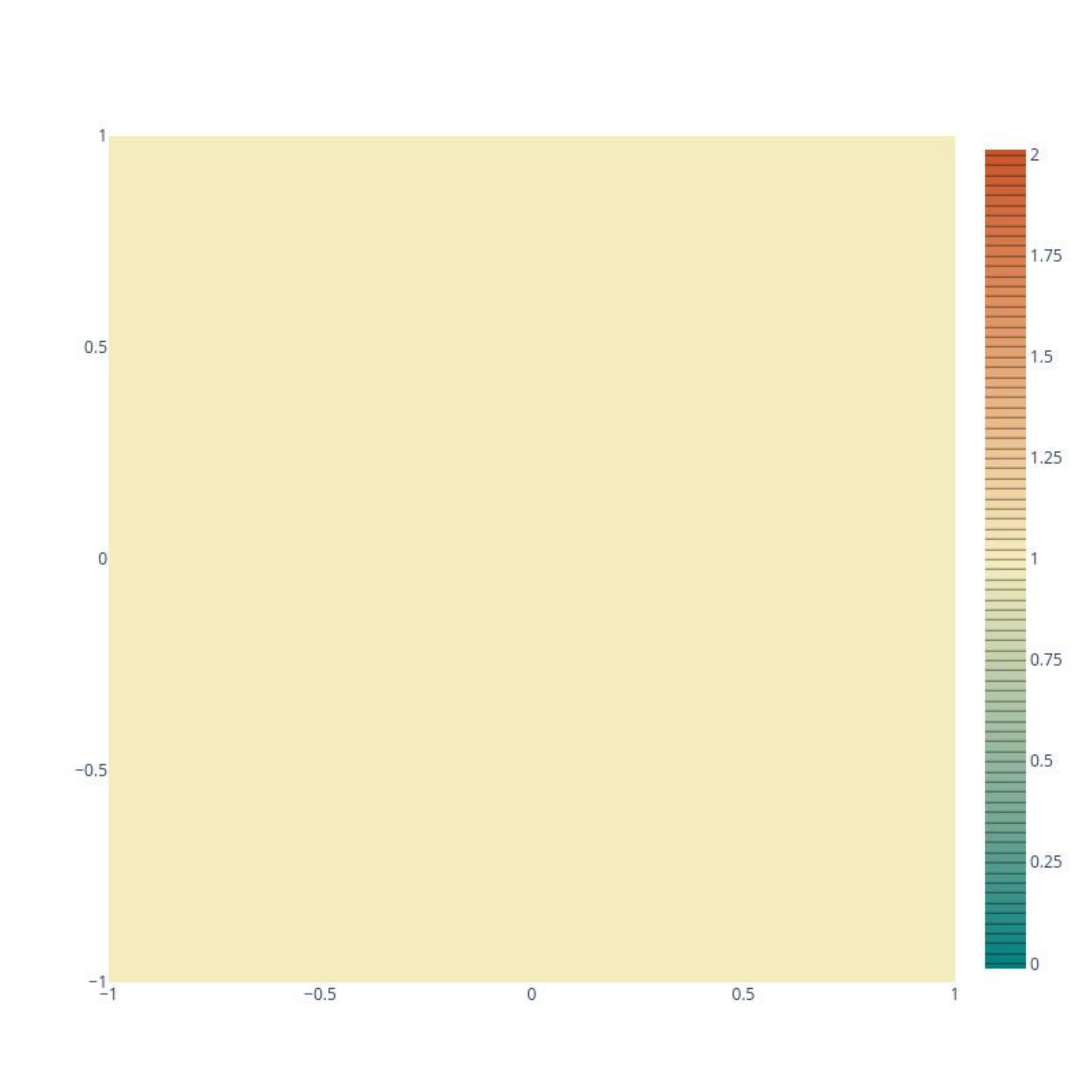}}
         &
         \raisebox{-0.5\height}{\includegraphics[scale=0.2, trim=20pt 20pt 20pt 20pt, clip]{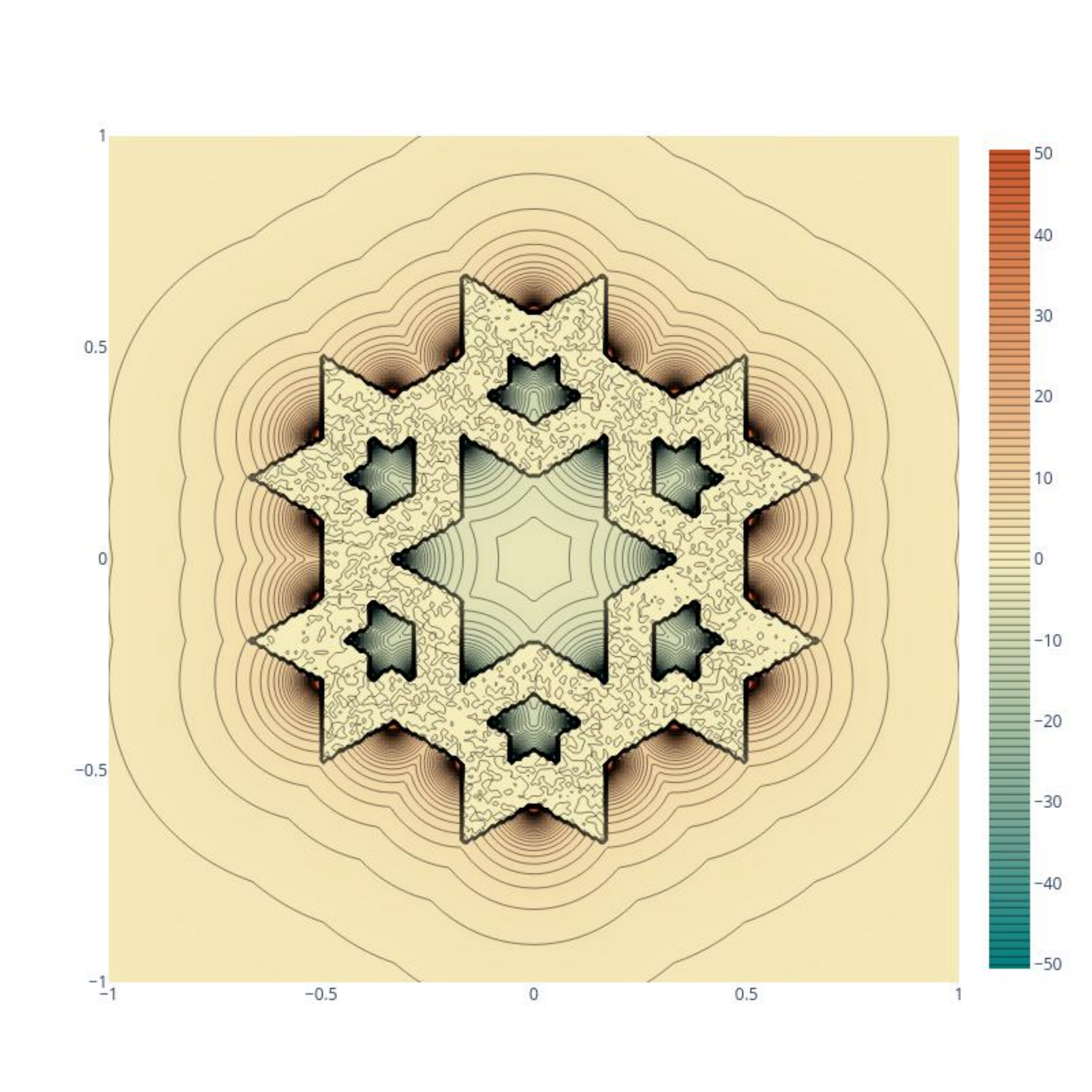}}
         &
         \raisebox{-0.5\height}{\includegraphics[scale=0.2, trim=20pt 20pt 20pt 20pt, clip]{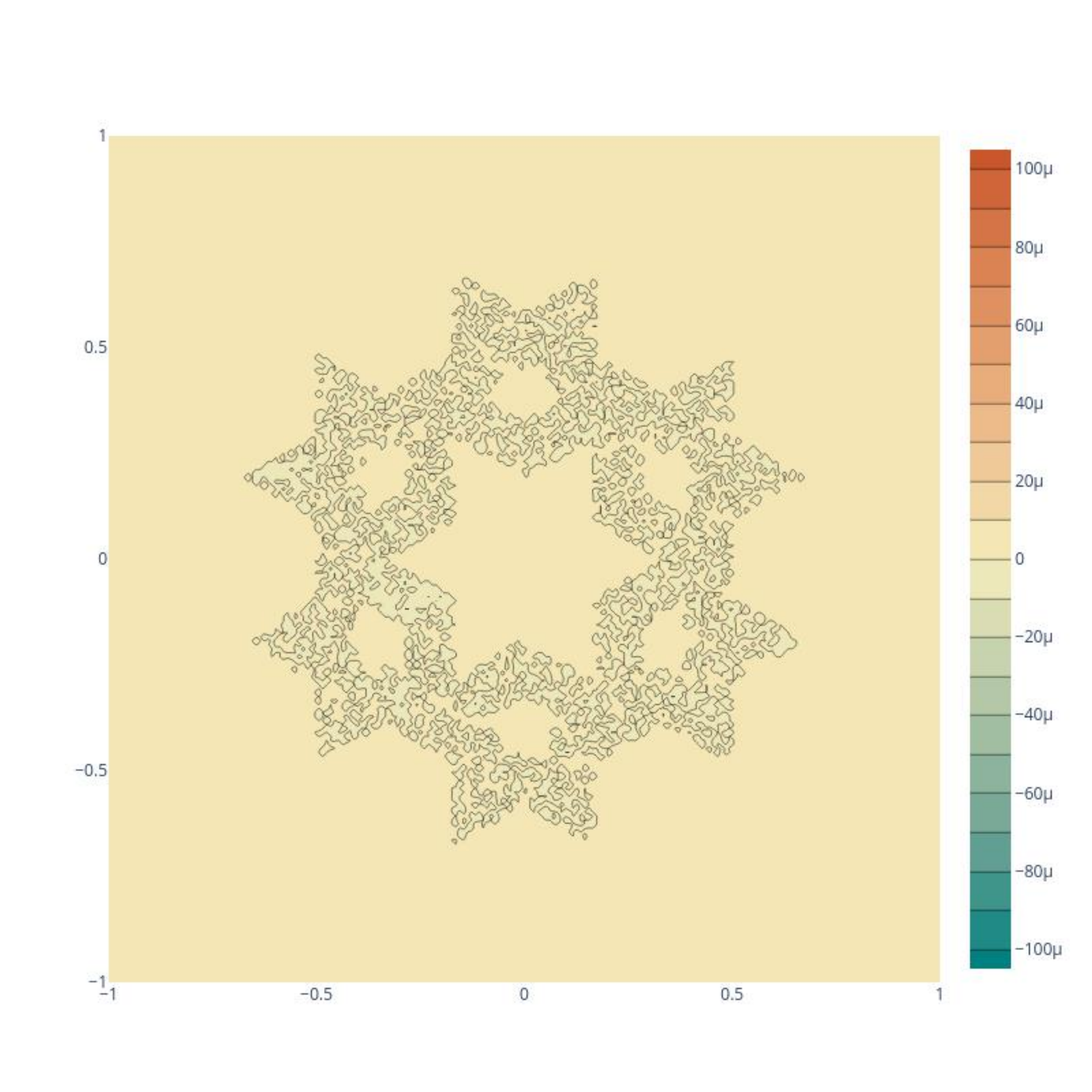}}\\
         & SDF & Gradient Norm & Divergence & Curl
     \end{tabular}}
        \caption{Ground Truth for 2D shapes. We can see that the divergence of the gradient vector field spikes sharply at points such as centres, skeletons or corners of the shape, and diffuses quickly from there.}
        \label{fig:appx:gt-vis}
\end{figure*}
% \begin{figure*}
%      \centering
%      \resizebox{\textwidth}{!}{
%      \begin{tabular}{c c c}
%          \rotatebox[origin=c]{90}{$20\times 20$ grid}
%          &
%          \includegraphics[scale=0.2, trim=20pt 20pt 20pt 20pt, clip]{assets/figures/theory-vis/2021-10-23_17-55-52_0_20x20.pdf}
%          &
%          \includegraphics[scale=0.2, trim=20pt 20pt 20pt 20pt, clip]{assets/figures/theory-vis/2021-10-23_17-55-52_0_20x20_with_div.pdf}\\
%          \rotatebox[origin=c]{90}{$200\times 200$ grid}
%          &
%          \includegraphics[scale=0.2, trim=20pt 20pt 20pt 20pt, clip]{assets/figures/theory-vis/2021-10-23_17-55-52_0_200x200.pdf}
%          &
%          \includegraphics[scale=0.2, trim=20pt 20pt 20pt 20pt, clip]{assets/figures/theory-vis/2021-10-23_17-55-52_0_200x200_with_div.pdf}\\
%          & without div & with div
%      \end{tabular}}
%         \caption{Toy SDF learning problem}
%         \label{fig:appx:theory-vis}
% \end{figure*}

\begin{figure*}
     \centering
     \resizebox{\textwidth}{!}{
     \begin{tabular}{c c c c}
     \setlength\tabcolsep{0pt}
         \includegraphics[scale=0.5]{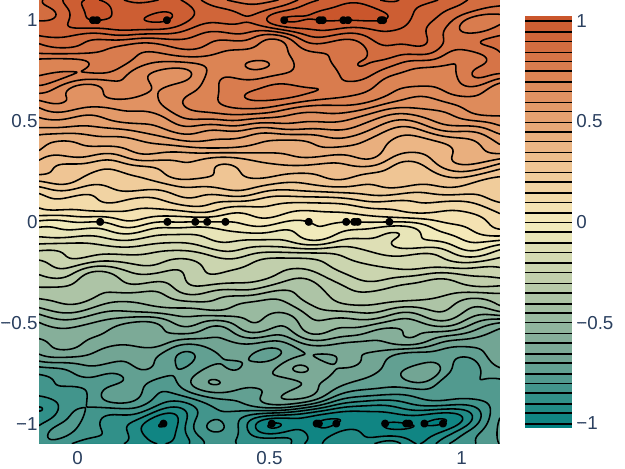}
         &
         \includegraphics[scale=0.5]{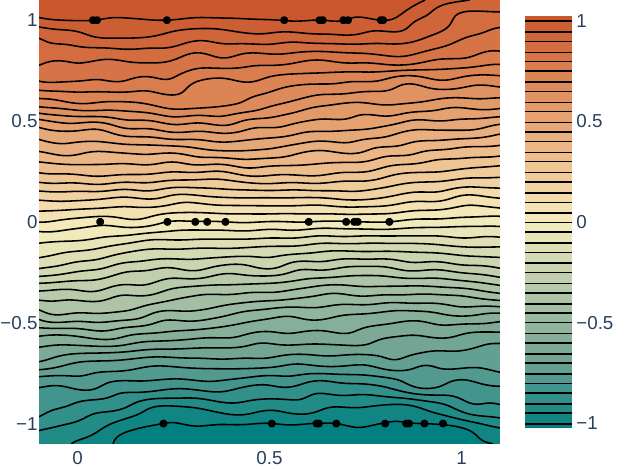}
         &
         \includegraphics[scale=0.5]{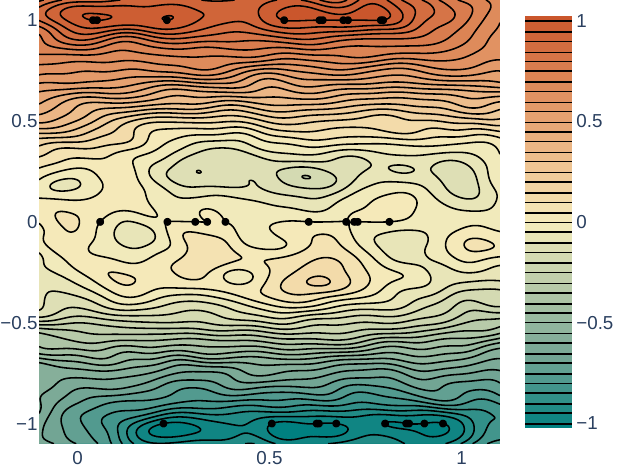}
         &
         \includegraphics[scale=0.5]{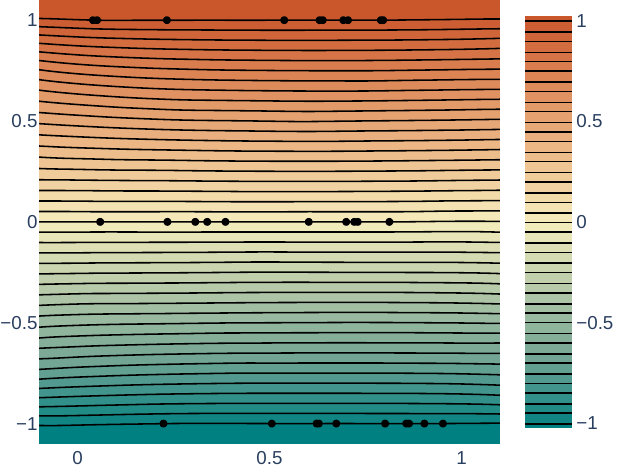}
         \\
         \includegraphics[scale=0.5]{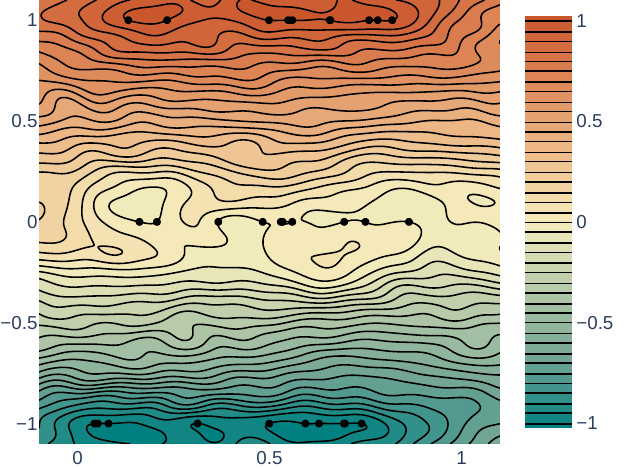}
         &
         \includegraphics[scale=0.5]{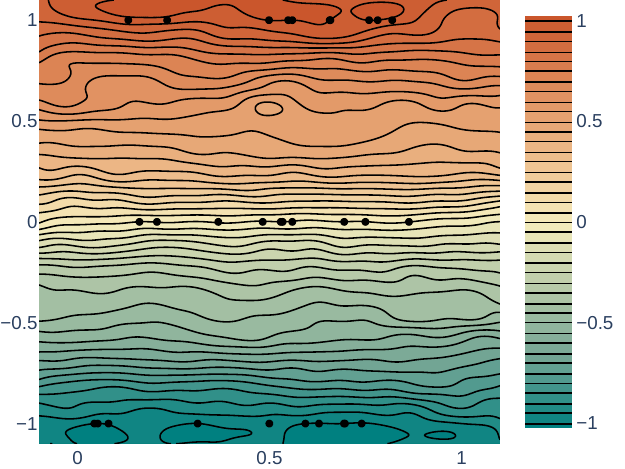}
         &
         \includegraphics[scale=0.5]{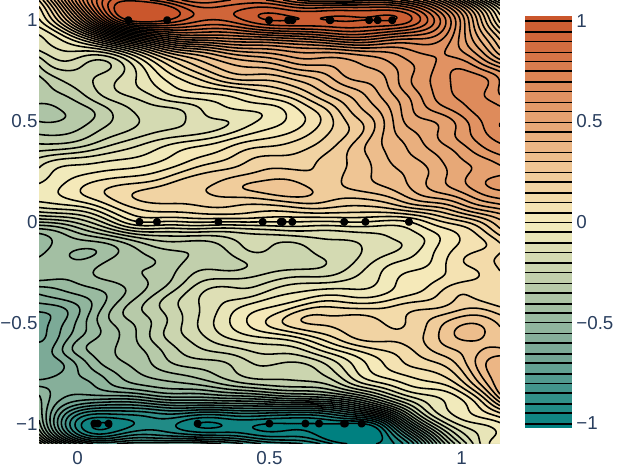}
         &
         \includegraphics[scale=0.5]{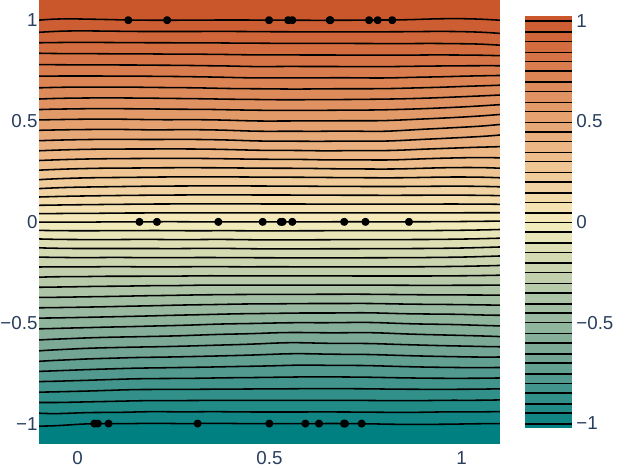}
         \\
         \includegraphics[scale=0.5]{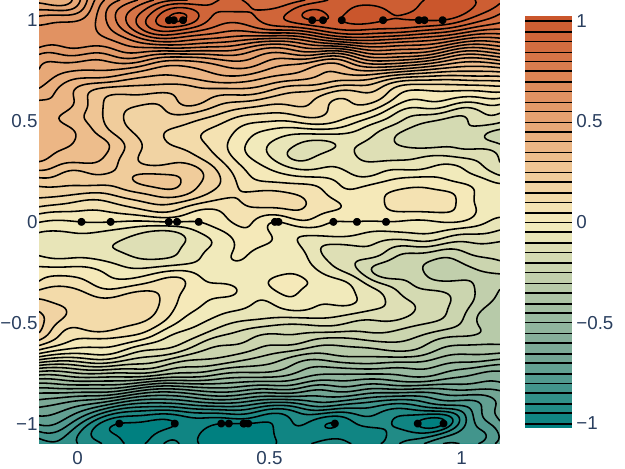}
         &
         \includegraphics[scale=0.5]{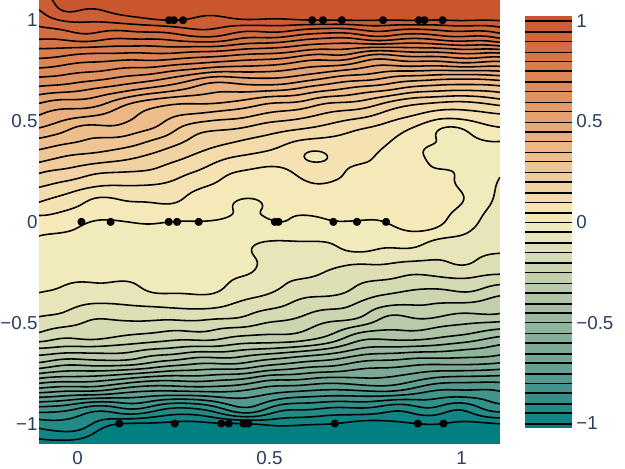}
         &
         \includegraphics[scale=0.5]{assets/figures/theory-vis/2021-11-11_13-53-00_4_200x200.pdf.png-min.png}
         &
         \includegraphics[scale=0.5]{assets/figures/theory-vis/2021-11-11_13-53-00_4_200x200_with_div.pdf.png-min.png}
         \\
         \includegraphics[scale=0.5]{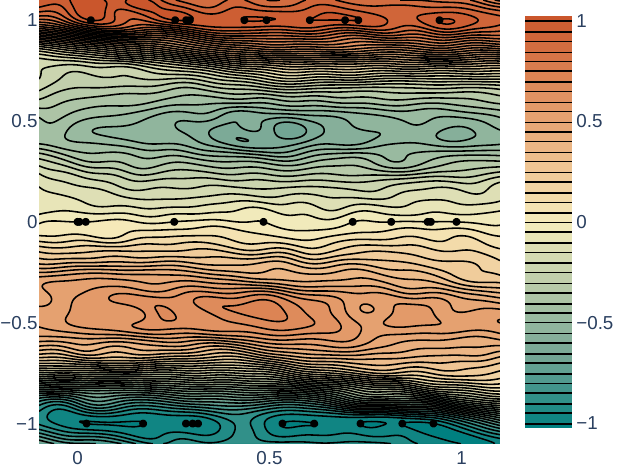}
         &
         \includegraphics[scale=0.5]{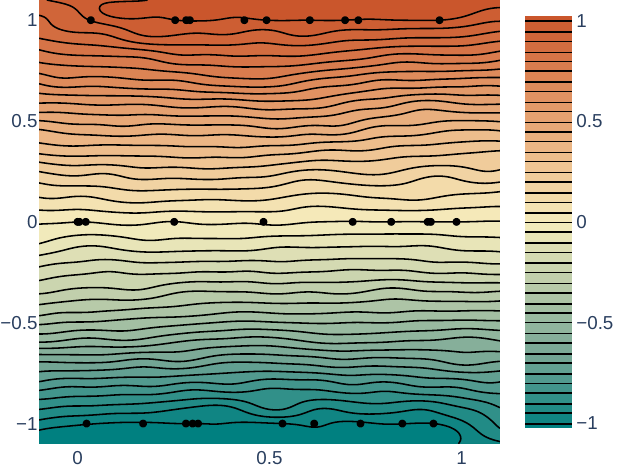}
         &
         \includegraphics[scale=0.5]{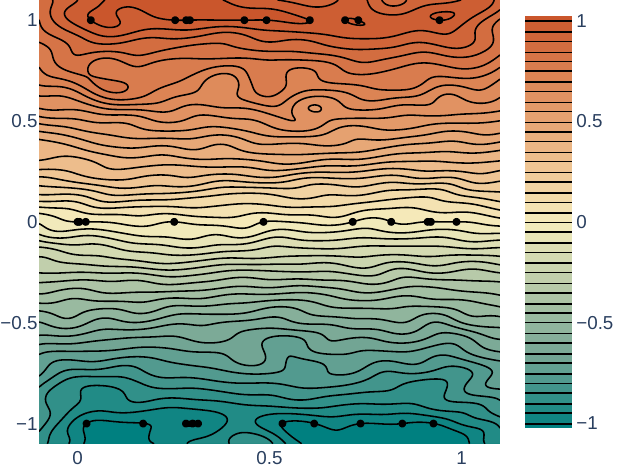}
         &
         \includegraphics[scale=0.5]{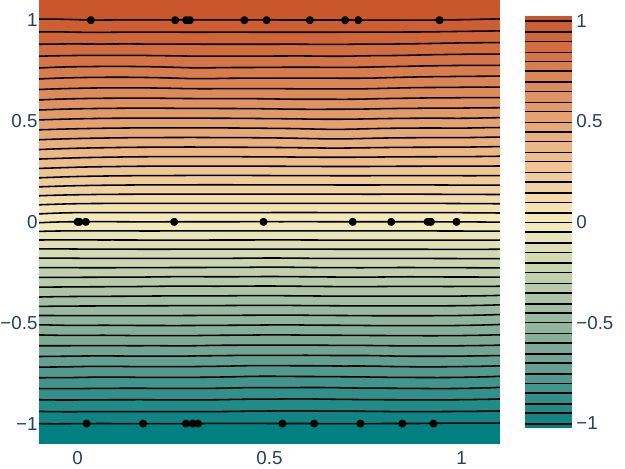}
         \\
         \includegraphics[scale=0.5]{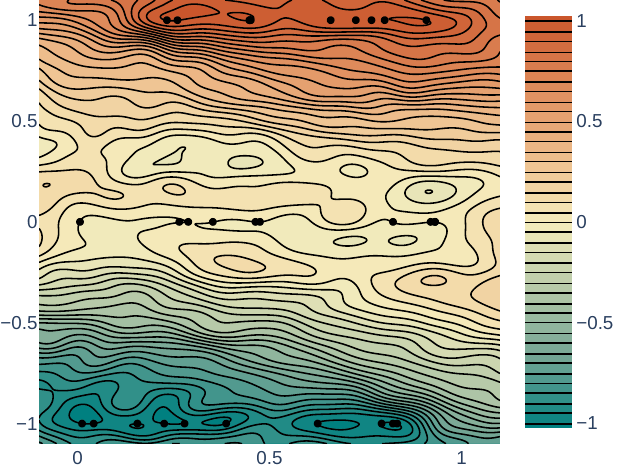}
         &
         \includegraphics[scale=0.5]{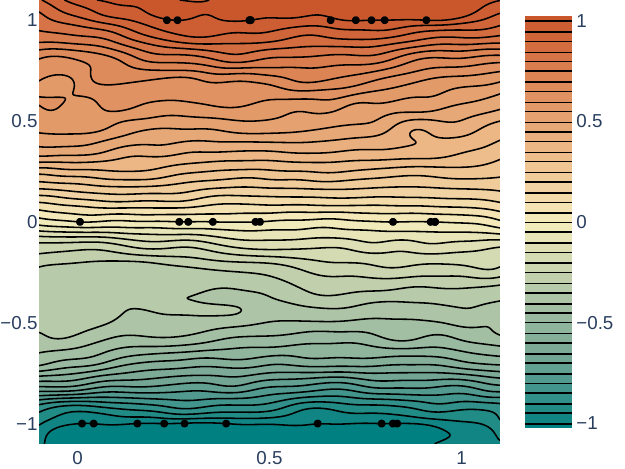}
         &
         \includegraphics[scale=0.5]{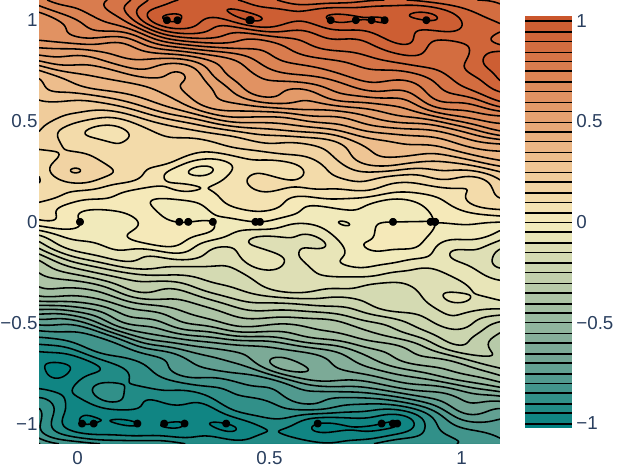}
         &
         \includegraphics[scale=0.5]{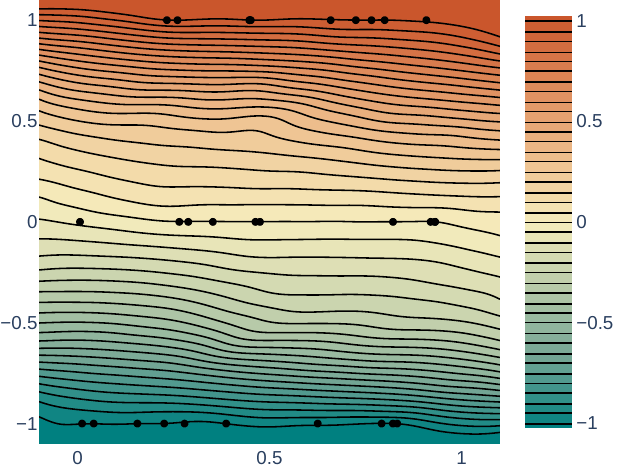}
         \\
         20x20 wo div & 20x20 w div & 200x200 wo div & 200x200 w div
     \end{tabular}
     }
        \caption{Five repetitions of the toy problem. Contour lines and coloring shows the learned function. Black dots on the lines $y=-1$, $y=0$ and $y=1$ show the point constraints. When the divergence term is present, the contour lines are more smooth and the spacing is more uniform, as desired. When the divergence term is absent, sometimes the sign of the function is considerably incorrect, with negative above $y=0$ and positive below.
        Notice that when the sampling for the point constraints are not uniform, e.g., the last row where the constraints are almost clustered in a diagonal, the learned function is often biased.}
        \label{fig:appx:theory-vis}
\end{figure*}

\subsubsection{Understanding the divergence constraint}\label{appx:sec:div-theory}
\textbf{The divergence theorem interpretation. }
The divergence theorem \cite{morse1954methods} states that integrating over the outward flow in a volume using a triple integral of the divergence is equivalent to a double integral of the flux through its encapsulating surface.
\begin{equation}
    \iiint_V \Delta \Phi(x; \theta)dV = \iint_S \nabla\Phi(x; \theta)\cdot \hat{n} dS
\end{equation}
To intuitively understand what the divergence at a point means, consider the above equation with the volume $V$ being a ball centered at the point, and taking the limit of the radius to 0. Then the divergence of a point is how much the vector field moves towards that point or away from that point from all directions, where mostly moving towards the point implies positive divergence (often called a sink), mostly moving away from that point implies negative divergence (often called a source), and it being balanced implies zero divergence. Thus a point having low divergence magnitude implies that the direction of the gradients are not changing much around that point. The theorem implies that divergence of a single point is heavily influenced by the surrounding region, so it incorporates a lot of local information of the (gradient) vector field. 

% This intuition is an important take away, which we will discuss at length in the following sections.
% ... (proof/discussion for regularisation, more diagrams of theory experiment)

\textbf{Minimising Divergence as Regularisation \& the Dirichlet Energy.}
A common setup in machine learning is to not only optimise for a given loss function, but to penalise model complexity by regularising towards less complex solution. This can be viewed as an Occam's Razor principle, where simpler solutions are often better explanations/predictions. We can quantify the complexity of the SDF we have solved for using the Dirichlet Energy. The Dirichlet Energy of a function $\Phi$ over a space $\Omega$ gives a notion for how smooth variable a function is \cite{bronstein2017geometric}, defined by the convex functional
\begin{equation}
    E[\Phi]=\frac{1}{2}\int_\Omega\|\nabla\Phi(x)\|_2^2 dx.
\end{equation}
Using Green's first identity we have that
\begin{equation}
    E[\Phi]=\frac{1}{2}\int_{\partial \Omega} \langle \nabla \Phi(x), \mathbf{n}(x)\rangle \Phi(x) dx - \frac{1}{2}\int_\Omega \Delta \Phi(x) \Phi(x) dx
\end{equation}
and its functional derivative is
\begin{equation}
     DE[\Phi]\Psi = \int_{\partial \Omega} \langle \nabla \Phi(x), \mathbf{n}(x)\rangle \Psi(x) dx - \int_\Omega \Delta \Phi(x) \Psi(x) dx
\end{equation}
where $\mathbf{n}$ is the outward normal vector to the boundary $\partial \Omega$.   Thus to minimise the Dirichlet energy, the functional derivative needs to be 0 for all $\Psi$. As $\Psi$ is a infinitesimal displacement of $\Phi$, it vanishes on $\partial \Omega$, so we get $\Delta \Phi(x)=0$. This is Laplace's equation, whose solutions are harmonic function (e.g. steady-state heat equation, which will have a unique solution under sufficiently regular boundary conditions). 

However we are more restrictive in the functions $\Phi$ we want to minimise $E[\Phi]$ for, specifically we only consider $\Phi$ that interpolate our surface points within its zero level set (which can be considered as a boundary condition) and the eikonal equation must hold. As a result we would not be able to solve for a harmonic function, as most SDFs are not harmonic (see Figure~\ref{fig:appx:gt-vis} for the actual divergence field). However as our functional is convex, to find a function that satisfies our conditions and has minimal Dirichlet energy, it suffices to find the function satisfying our conditions whose absolute value of the functional derivative is as small as possible, i.e. minimising our divergence term $L_{div}$ (Equation~\ref{eq:L-div}). 

Note that this loss term is specifically defined over $\Omega~\setminus~\Omega_0$, compared to \eqref{eq:eikonal} which is defined over $\Omega$. The difference, $\Omega_0$, is a set of measure zero and mathematically does not change much, but it does computationally: to evaluate our losses we have random samples over the space ($\Omega\approx \Omega \setminus \Omega_0$) and samples at ground truth surface points (i.e., heavy sampling on the surface $\Omega_0$). Thus \eqref{eq:eikonal} is computed over both sets of samples, while the divergence loss term is only computer over the former. If we computed the divergence loss over the heavily sampled surface points, it would make our surface drastically smooth. In practice we want the opposite of this: a surface with fine detail will have high variability in its SDF near the surface and low variability further away from the surface.

PHASE \cite{lipman2021phase} also use the Dirichlet energy term for regularisation in their method, but motivate it from the Van der Waals-Cahn-Hilliard (WCH) theory for the physical phenomenon of phase transitions, and they do not anneal the loss.

\textbf{Toy Problem.}
We give more visualisations for the toy problem discussed in \secref{Sec:approach_divergence}. The experiment was repeated 20 times for each of the four cases (20x20 grid without divergence, 20x20 grid with divergence, 200x200 grid without divergence and 200x200 grid with divergence), where the same randomly sampled point constraints were used for the four experiments in the same repetition. \figref{fig:appx:theory-vis} shows the learned functions for five repetitions. 
When the divergence term is present, the contour lines are more smooth and the spacing is more uniform, as desired, showing that they do a better job at maintaining the Eikonal equation. When the divergence term is absent, sometimes the sign of the function is considerably incorrect, with negative above $y=0$ and positive below, showing it is less stable and more variable and thus that our divergence term is acting as regularization. Also notice that when the sampling for the point constraints are not uniform, e.g., the last row where the constraints are almost clustered in a diagonal, the learned function is often biased.

\subsubsection{Understanding the divergence constraint in 2D }
We perform qualitative analysis of the proposed approach on 2D simple shapes: circle, L shape polygon and Koch's snowflake polygon. In this experiment we trained a network with 4 layers and 128 elements in each layer with sine activation functions for 10K epochs, sampling a new set of points in each epoch. 
We compare to SIREN \cite{sitzmann2020siren}, SIREN without normal vectors, and the proposed DiGS approach with the proposed MFGI initialziation and without it. 
The results are shown in \figref{fig:appx:2D_reconstructions} and \figref{fig:appx:2D_reconstructions_snowflake} where a heatmap is used to visualize the learned distance function, the eikonal term, the divergence, the curl and the difference between the unsigned predicted distance and the ground truth distance. For the circle shape, incorporating the divergence constraint without decay yields the best result since the circle is a smooth closed shape without fine detail. In contrast, Koch's snowflake is characterised with sharp edges, therefore starting with the divergence constraint and annealing it yields the best performance. In this case the divergence term guides the learning process to a smoothed version of the snowflake, and the annealing allows it to fit the geometry more tightly. SIREN without the normal vectors exhibits ghost geometries (zero level sets that should not appear), while DiGS does not. 

The initialization significantly effects the sign of the distance function as well as the model's ability to properly reconstruct fine detail, particularly for the snowflake example. 

\begin{figure*}
     \centering
     \begin{tabular}{M{0.05\textwidth} c c c c c}
        
        \rotatebox[origin=c]{90}{SIREN} &
         \raisebox{-.5\height}{\includegraphics[width=\recviscoltwod\linewidth, trim=30pt 30pt 7pt 37pt, clip]{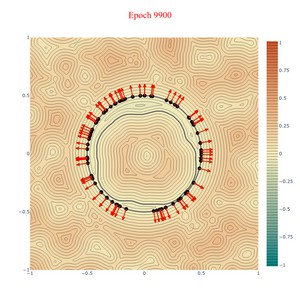}}
         &
        \raisebox{-.5\height}{\includegraphics[width=\recviscoltwod\linewidth, trim=30pt 30pt 7pt 37pt, clip]{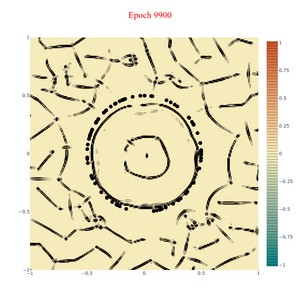}}
         &
         \raisebox{-.5\height}{\includegraphics[width=\recviscoltwod\linewidth, trim=30pt 30pt 7pt 37pt, clip]{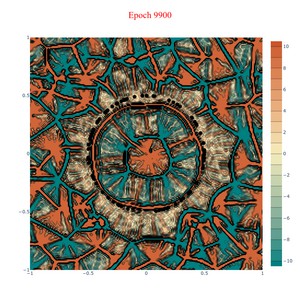}}
         &
        \raisebox{-.5\height}{\includegraphics[width=\recviscoltwod\linewidth, trim=30pt 30pt 7pt 37pt, clip]{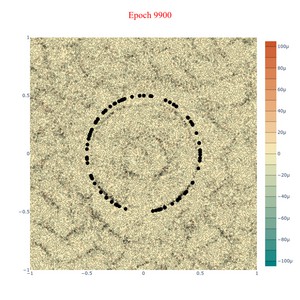}}
         &
        \raisebox{-.5\height}{\includegraphics[width=\recviscoltwod\linewidth, trim=30pt 30pt 7pt 37pt, clip]{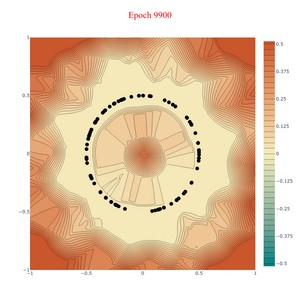}}
         \\
         
         \rotatebox[origin=c]{90}{\makecell{SIREN \\ wo n}   }  &
         \raisebox{-.5\height}{\includegraphics[width=\recviscoltwod\linewidth, trim=30pt 30pt 7pt 37pt, clip]{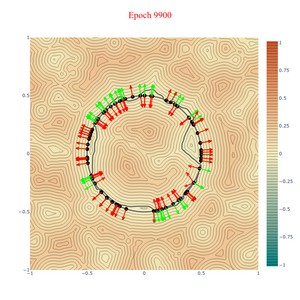}}
         &
        \raisebox{-.5\height}{\includegraphics[width=\recviscoltwod\linewidth, trim=30pt 30pt 7pt 37pt, clip]{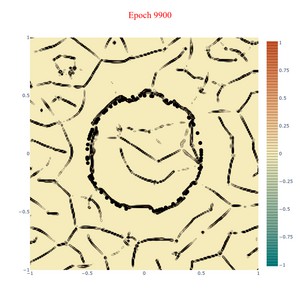}}
         &
         \raisebox{-.5\height}{\includegraphics[width=\recviscoltwod\linewidth, trim=30pt 30pt 7pt 37pt, clip]{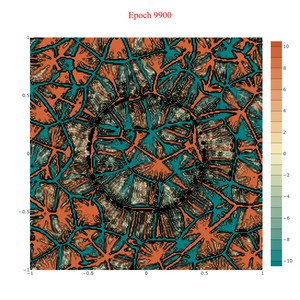}}
         &
        \raisebox{-.5\height}{\includegraphics[width=\recviscoltwod\linewidth, trim=30pt 30pt 7pt 37pt, clip]{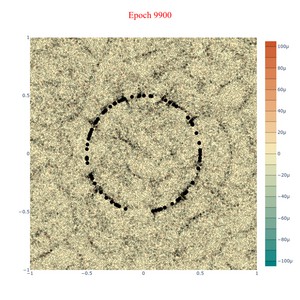}}
         &
        \raisebox{-.5\height}{\includegraphics[width=\recviscoltwod\linewidth, trim=30pt 30pt 7pt 37pt, clip]{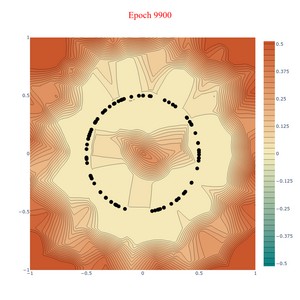}}
         \\
         
         \rotatebox[origin=c]{90}{\makecell{DiGS \\ no decay}}  &  
         \raisebox{-.5\height}{\includegraphics[width=\recviscoltwod\linewidth, trim=30pt 30pt 7pt 37pt, clip]{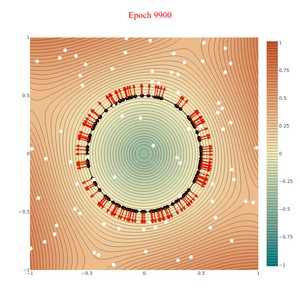}}
         &
        \raisebox{-.5\height}{\includegraphics[width=\recviscoltwod\linewidth, trim=30pt 30pt 7pt 37pt, clip]{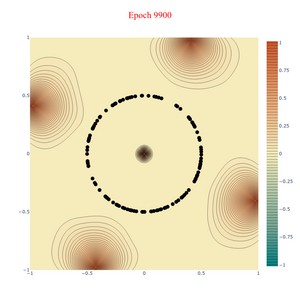}}
         &
        \raisebox{-.5\height}{\includegraphics[width=\recviscoltwod\linewidth, trim=30pt 30pt 7pt 37pt, clip]{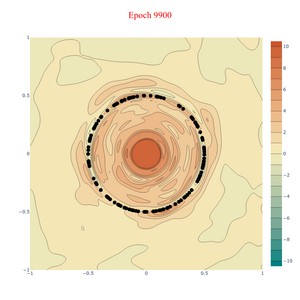}}
         &         
        \raisebox{-.5\height}{\includegraphics[width=\recviscoltwod\linewidth, trim=30pt 30pt 7pt 37pt, clip]{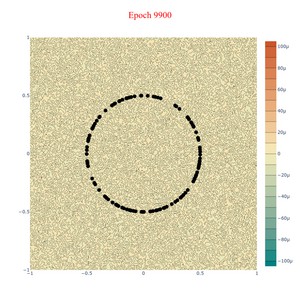}}
         &
        \raisebox{-.5\height}{\includegraphics[width=\recviscoltwod\linewidth, trim=30pt 30pt 7pt 37pt, clip]{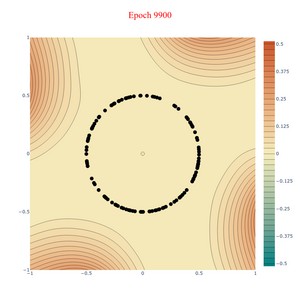}}
         \\
         
         \rotatebox[origin=c]{90}{\makecell{DiGS \\ wo MFGI}} & \raisebox{-.5\height}{\includegraphics[width=\recviscoltwod\linewidth, trim=30pt 30pt 7pt 37pt, clip]{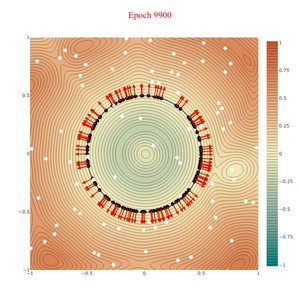}}
         &
        \raisebox{-.5\height}{\includegraphics[width=\recviscoltwod\linewidth, trim=30pt 30pt 7pt 37pt, clip]{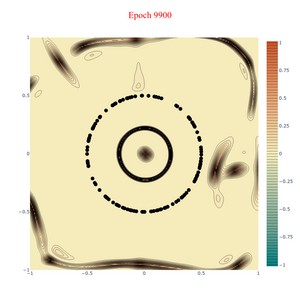}}
         &
        \raisebox{-.5\height}{\includegraphics[width=\recviscoltwod\linewidth, trim=30pt 30pt 7pt 37pt, clip]{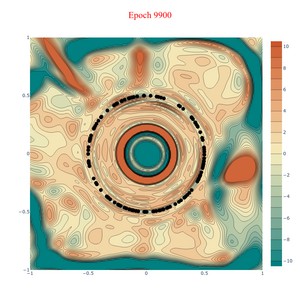}}
         &         
        \raisebox{-.5\height}{\includegraphics[width=\recviscoltwod\linewidth, trim=30pt 30pt 7pt 37pt, clip]{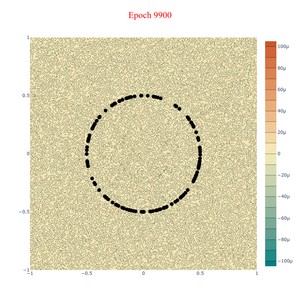}}
         &
        \raisebox{-.5\height}{\includegraphics[width=\recviscoltwod\linewidth, trim=30pt 30pt 7pt 37pt, clip]{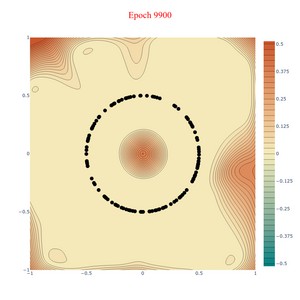}}
         \\
         
          \rotatebox[origin=c]{90}{\textbf{DiGS}} & \raisebox{-.5\height}{\includegraphics[width=\recviscoltwod\linewidth, trim=30pt 30pt 7pt 37pt, clip]{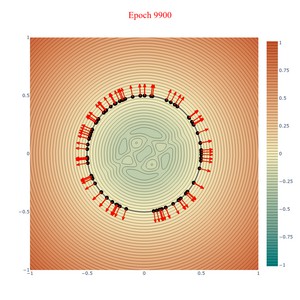}}
         &
        \raisebox{-.5\height}{\includegraphics[width=\recviscoltwod\linewidth, trim=30pt 30pt 7pt 37pt, clip]{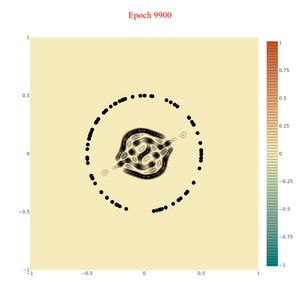}}
         &
        \raisebox{-.5\height}{\includegraphics[width=\recviscoltwod\linewidth, trim=30pt 30pt 7pt 37pt, clip]{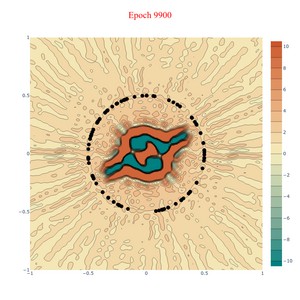}}
         &         
        \raisebox{-.5\height}{\includegraphics[width=\recviscoltwod\linewidth, trim=30pt 30pt 7pt 37pt, clip]{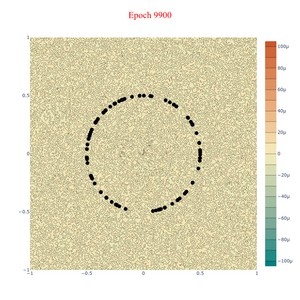}}
         &
        \raisebox{-.5\height}{\includegraphics[width=\recviscoltwod\linewidth, trim=30pt 30pt 7pt 37pt, clip]{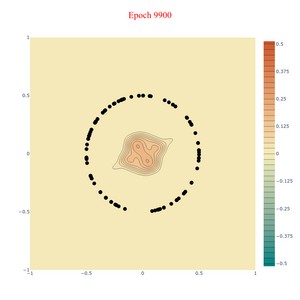}}
         \\

         \rule{0pt}{0.25ex} &&&&& \\
         \hline 
         \rule{0pt}{0.25ex} &&&&& \\
         
%%%%%%%%%%%%%%%%%%%%%%%%%%%%%%%%%%%%%%%%%%%%%%%%%%%% L shape %%%%%%%%%%%%%%%%%%%%%%%%%%%%%%%%%%%%%%%%%%%%

        \rotatebox[origin=c]{90}{SIREN} &
         \raisebox{-.5\height}{\includegraphics[width=\recviscoltwod\linewidth, trim=30pt 30pt 7pt 37pt, clip]{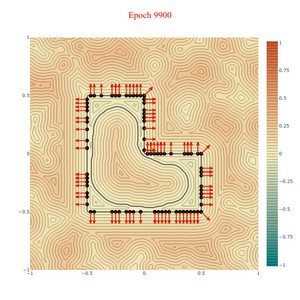}}
         &
        \raisebox{-.5\height}{\includegraphics[width=\recviscoltwod\linewidth, trim=30pt 30pt 7pt 37pt, clip]{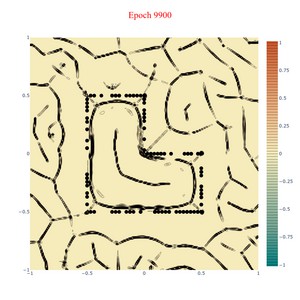}}
         &
         \raisebox{-.5\height}{\includegraphics[width=\recviscoltwod\linewidth, trim=30pt 30pt 7pt 37pt, clip]{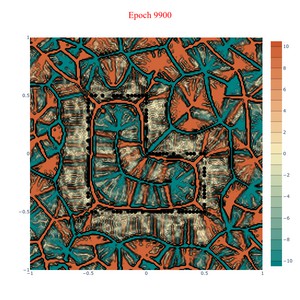}}
         &
        \raisebox{-.5\height}{\includegraphics[width=\recviscoltwod\linewidth, trim=30pt 30pt 7pt 37pt, clip]{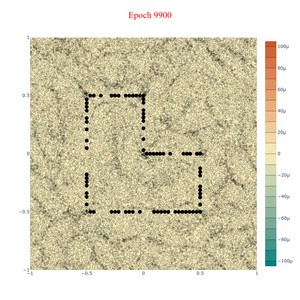}}
         &
        \raisebox{-.5\height}{\includegraphics[width=\recviscoltwod\linewidth, trim=30pt 30pt 7pt 37pt, clip]{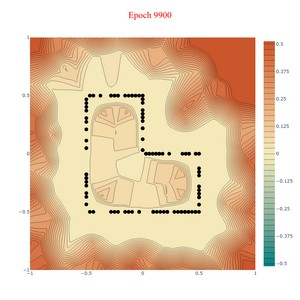}}
         \\
         
         \rotatebox[origin=c]{90}{\makecell{SIREN \\ wo n}   }  &
         \raisebox{-.5\height}{\includegraphics[width=\recviscoltwod\linewidth, trim=30pt 30pt 7pt 37pt, clip]{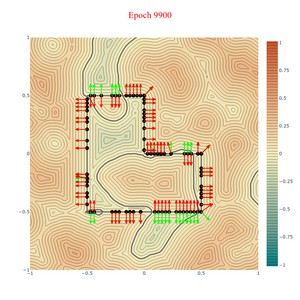}}
         &
        \raisebox{-.5\height}{\includegraphics[width=\recviscoltwod\linewidth, trim=30pt 30pt 7pt 37pt, clip]{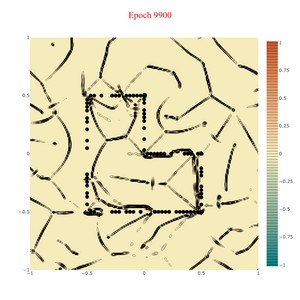}}
         &
         \raisebox{-.5\height}{\includegraphics[width=\recviscoltwod\linewidth, trim=30pt 30pt 7pt 37pt, clip]{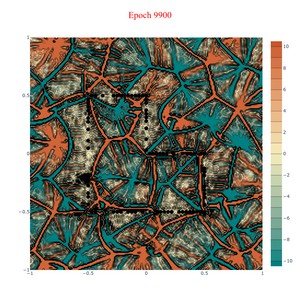}}
         &
        \raisebox{-.5\height}{\includegraphics[width=\recviscoltwod\linewidth, trim=30pt 30pt 7pt 37pt, clip]{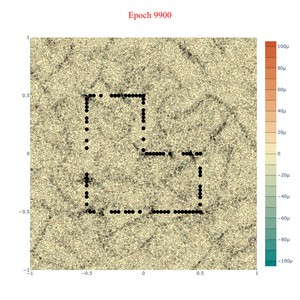}}
         &
        \raisebox{-.5\height}{\includegraphics[width=\recviscoltwod\linewidth, trim=30pt 30pt 7pt 37pt, clip]{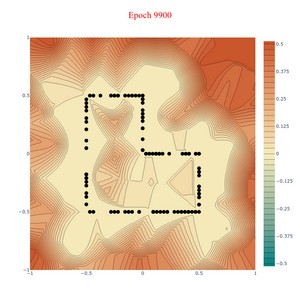}}
         \\
         
         \rotatebox[origin=c]{90}{\makecell{DiGS \\ no decay}}  &  
         \raisebox{-.5\height}{\includegraphics[width=\recviscoltwod\linewidth, trim=30pt 30pt 7pt 37pt, clip]{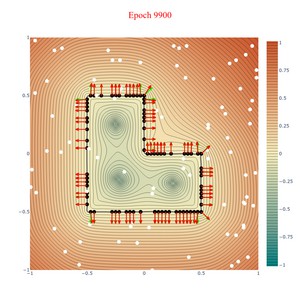}}
         &
        \raisebox{-.5\height}{\includegraphics[width=\recviscoltwod\linewidth, trim=30pt 30pt 7pt 37pt, clip]{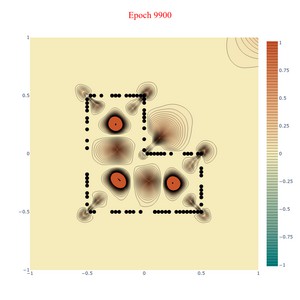}}
         &
        \raisebox{-.5\height}{\includegraphics[width=\recviscoltwod\linewidth, trim=30pt 30pt 7pt 37pt, clip]{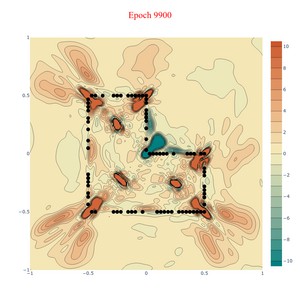}}
         &         
        \raisebox{-.5\height}{\includegraphics[width=\recviscoltwod\linewidth, trim=30pt 30pt 7pt 37pt, clip]{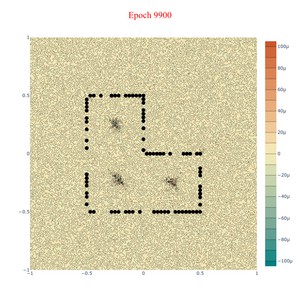}}
         &
        \raisebox{-.5\height}{\includegraphics[width=\recviscoltwod\linewidth, trim=30pt 30pt 7pt 37pt, clip]{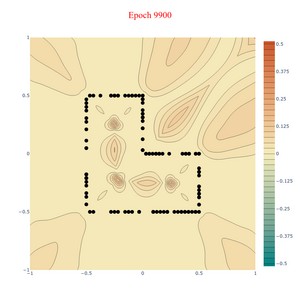}}
         \\
         
         \rotatebox[origin=c]{90}{\makecell{DiGS \\ wo MFGI}} & \raisebox{-.5\height}{\includegraphics[width=\recviscoltwod\linewidth, trim=30pt 30pt 7pt 37pt, clip]{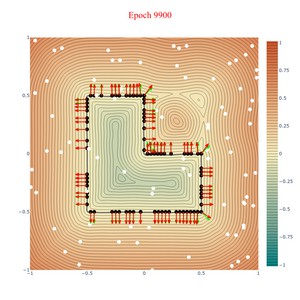}}
         &
        \raisebox{-.5\height}{\includegraphics[width=\recviscoltwod\linewidth, trim=30pt 30pt 7pt 37pt, clip]{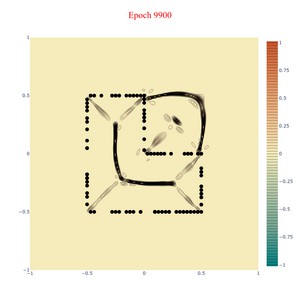}}
         &
        \raisebox{-.5\height}{\includegraphics[width=\recviscoltwod\linewidth, trim=30pt 30pt 7pt 37pt, clip]{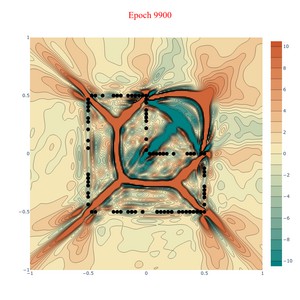}}
         &         
        \raisebox{-.5\height}{\includegraphics[width=\recviscoltwod\linewidth, trim=30pt 30pt 7pt 37pt, clip]{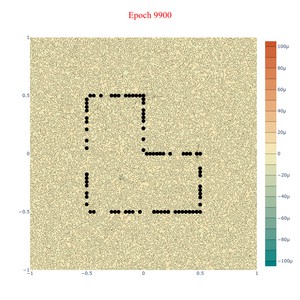}}
         &
        \raisebox{-.5\height}{\includegraphics[width=\recviscoltwod\linewidth, trim=30pt 30pt 7pt 37pt, clip]{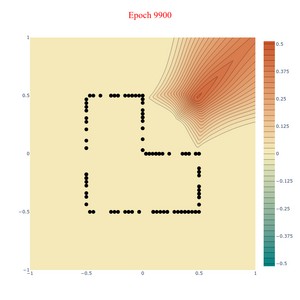}}
         \\
         
          \rotatebox[origin=c]{90}{\textbf{DiGS}} & \raisebox{-.5\height}{\includegraphics[width=\recviscoltwod\linewidth, trim=30pt 30pt 7pt 37pt, clip]{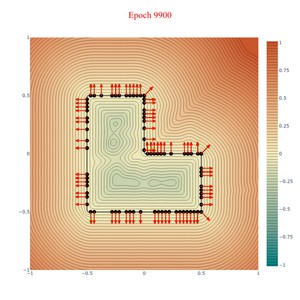}}
         &
        \raisebox{-.5\height}{\includegraphics[width=\recviscoltwod\linewidth, trim=30pt 30pt 7pt 37pt, clip]{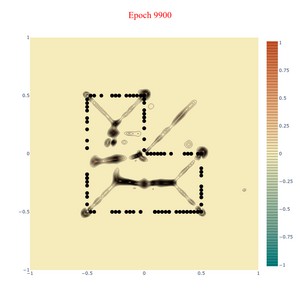}}
         &
        \raisebox{-.5\height}{\includegraphics[width=\recviscoltwod\linewidth, trim=30pt 30pt 7pt 37pt, clip]{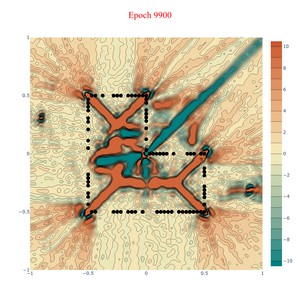}}
         &         
        \raisebox{-.5\height}{\includegraphics[width=\recviscoltwod\linewidth, trim=30pt 30pt 7pt 37pt, clip]{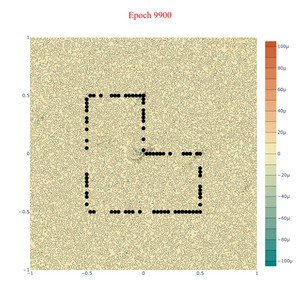}}
         &
        \raisebox{-.5\height}{\includegraphics[width=\recviscoltwod\linewidth, trim=30pt 30pt 7pt 37pt, clip]{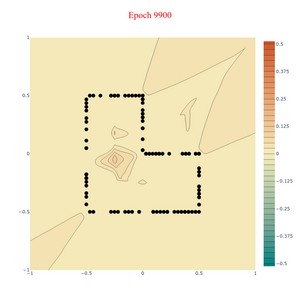}}
         \\
           method  & SDF & eikonal  & divergence & curl & distance gt diff  \\
     \end{tabular}
        \caption{Qualitative 2D results. Visualizing informative quantities as heatmaps over the evaluation space: (left to right) sign distance function, divergence, eikonal term, curl, and distance difference between ground truth and inference for the Circle and L shapes. }
        \label{fig:appx:2D_reconstructions}
\end{figure*}       
         
        %   \rule{0pt}{0.25ex} &&&&& \\
        %  \hline 
        %  \rule{0pt}{0.25ex} &&&&& \\
 \begin{figure*}
\centering
\begin{tabular}{M{0.05\textwidth} c c c c c}
         %%%%%%%%%%%%%%%%%%%%%%%%%%%%%%%%%%%%%%%%%%%%%% snowflake shape %%%%%%%%%%%%%%%%%%%%%%%%%%%%%%%%%%%%%%%%%%%%
         
        \rotatebox[origin=c]{90}{SIREN} &
         \raisebox{-.5\height}{\includegraphics[width=\recviscoltwod\linewidth, trim=30pt 30pt 7pt 37pt, clip]{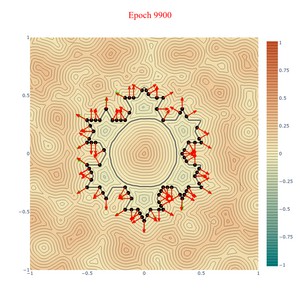}}
         &
        \raisebox{-.5\height}{\includegraphics[width=\recviscoltwod\linewidth, trim=30pt 30pt 7pt 37pt, clip]{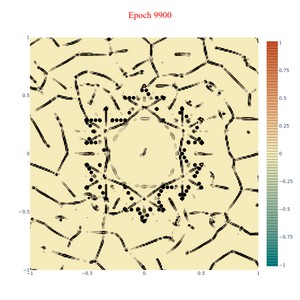}}
         &
         \raisebox{-.5\height}{\includegraphics[width=\recviscoltwod\linewidth, trim=30pt 30pt 7pt 37pt, clip]{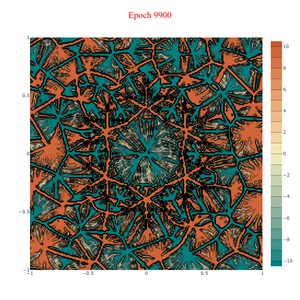}}
         &
        \raisebox{-.5\height}{\includegraphics[width=\recviscoltwod\linewidth, trim=30pt 30pt 7pt 37pt, clip]{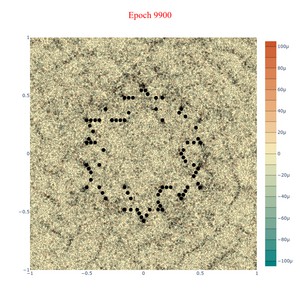}}
         &
        \raisebox{-.5\height}{\includegraphics[width=\recviscoltwod\linewidth, trim=30pt 30pt 7pt 37pt, clip]{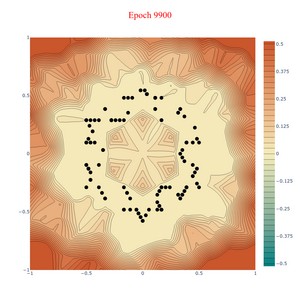}}
         \\
         
         \rotatebox[origin=c]{90}{\makecell{SIREN \\ wo n}   }  &
         \raisebox{-.5\height}{\includegraphics[width=\recviscoltwod\linewidth, trim=30pt 30pt 7pt 37pt, clip]{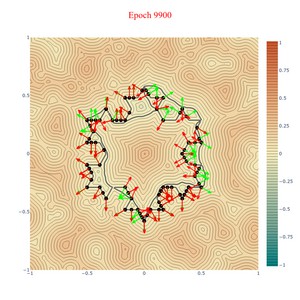}}
         &
        \raisebox{-.5\height}{\includegraphics[width=\recviscoltwod\linewidth, trim=30pt 30pt 7pt 37pt, clip]{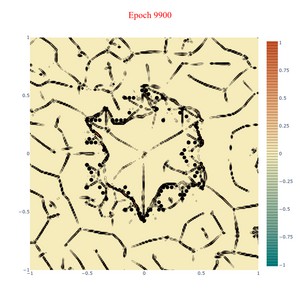}}
         &
         \raisebox{-.5\height}{\includegraphics[width=\recviscoltwod\linewidth, trim=30pt 30pt 7pt 37pt, clip]{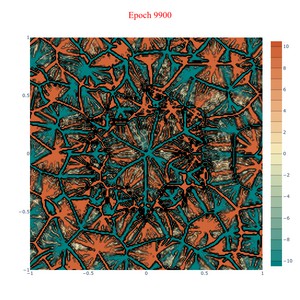}}
         &
        \raisebox{-.5\height}{\includegraphics[width=\recviscoltwod\linewidth, trim=30pt 30pt 7pt 37pt, clip]{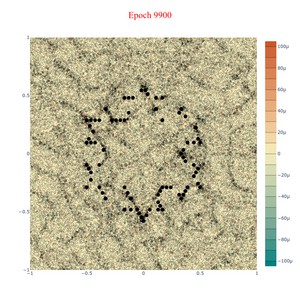}}
         &
        \raisebox{-.5\height}{\includegraphics[width=\recviscoltwod\linewidth, trim=30pt 30pt 7pt 37pt, clip]{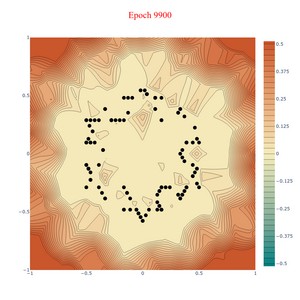}}
         \\
         
         \rotatebox[origin=c]{90}{\makecell{DiGS \\ no decay}}  &  
         \raisebox{-.5\height}{\includegraphics[width=\recviscoltwod\linewidth, trim=30pt 30pt 7pt 37pt, clip]{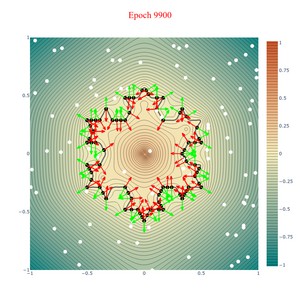}}
         &
        \raisebox{-.5\height}{\includegraphics[width=\recviscoltwod\linewidth, trim=30pt 30pt 7pt 37pt, clip]{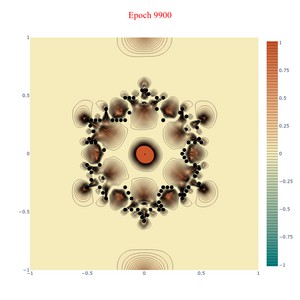}}
         &
        \raisebox{-.5\height}{\includegraphics[width=\recviscoltwod\linewidth, trim=30pt 30pt 7pt 37pt, clip]{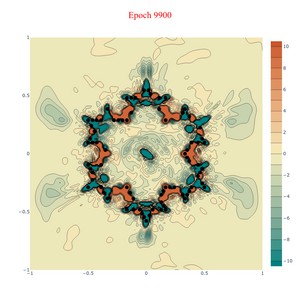}}
         &         
        \raisebox{-.5\height}{\includegraphics[width=\recviscoltwod\linewidth, trim=30pt 30pt 7pt 37pt, clip]{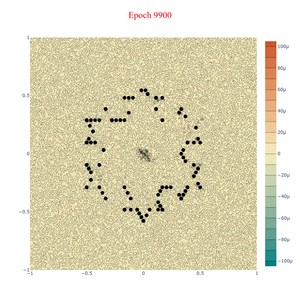}}
         &
        \raisebox{-.5\height}{\includegraphics[width=\recviscoltwod\linewidth, trim=30pt 30pt 7pt 37pt, clip]{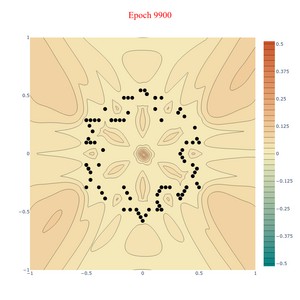}}
         \\
         
         \rotatebox[origin=c]{90}{\makecell{DiGS \\ wo MFGI}} & \raisebox{-.5\height}{\includegraphics[width=\recviscoltwod\linewidth, trim=30pt 30pt 7pt 37pt, clip]{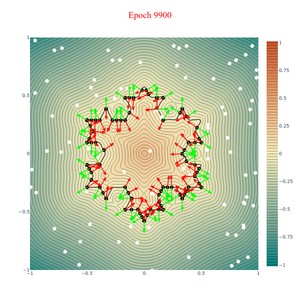}}
         &
        \raisebox{-.5\height}{\includegraphics[width=\recviscoltwod\linewidth, trim=30pt 30pt 7pt 37pt, clip]{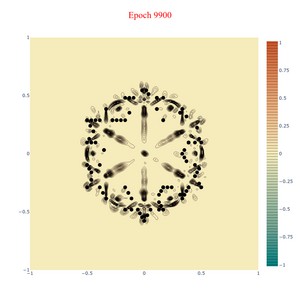}}
         &
        \raisebox{-.5\height}{\includegraphics[width=\recviscoltwod\linewidth, trim=30pt 30pt 7pt 37pt, clip]{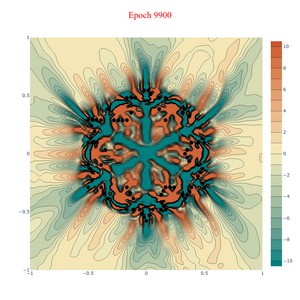}}
         &         
        \raisebox{-.5\height}{\includegraphics[width=\recviscoltwod\linewidth, trim=30pt 30pt 7pt 37pt, clip]{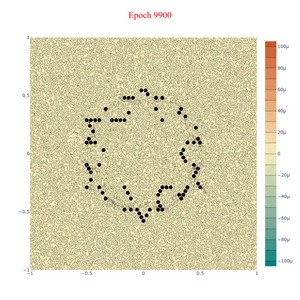}}
         &
        \raisebox{-.5\height}{\includegraphics[width=\recviscoltwod\linewidth, trim=30pt 30pt 7pt 37pt, clip]{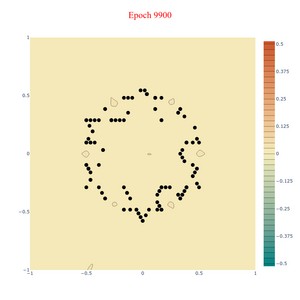}}
         \\
         
          \rotatebox[origin=c]{90}{\textbf{DiGS}} & \raisebox{-.5\height}{\includegraphics[width=\recviscoltwod\linewidth, trim=30pt 30pt 7pt 37pt, clip]{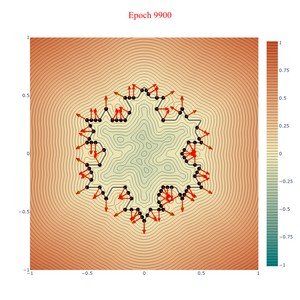}}
         &
        \raisebox{-.5\height}{\includegraphics[width=\recviscoltwod\linewidth, trim=30pt 30pt 7pt 37pt, clip]{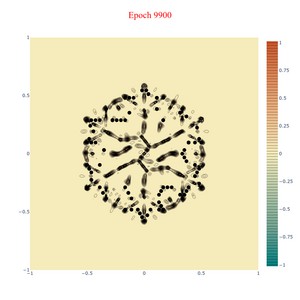}}
         &
        \raisebox{-.5\height}{\includegraphics[width=\recviscoltwod\linewidth, trim=30pt 30pt 7pt 37pt, clip]{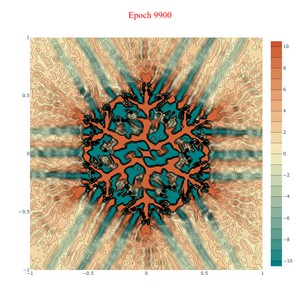}}
         &         
        \raisebox{-.5\height}{\includegraphics[width=\recviscoltwod\linewidth, trim=30pt 30pt 7pt 37pt, clip]{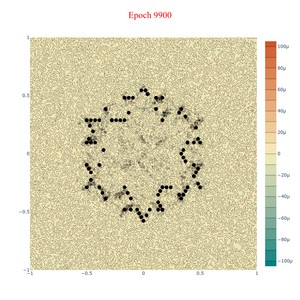}}
         &
        \raisebox{-.5\height}{\includegraphics[width=\recviscoltwod\linewidth, trim=30pt 30pt 7pt 37pt, clip]{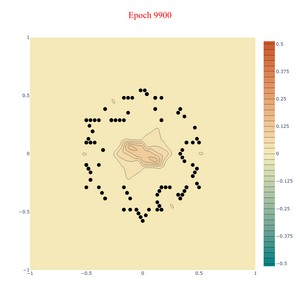}}
         \\
         
         method  & SDF & eikonal  & divergence & curl & distance gt diff  \\
     \end{tabular}
        \caption{Qualitative 2D results. Visualizing informative quantities as heatmaps over the evaluation space: (left to right) sign distance function, divergence, eikonal term, curl, and distance difference between ground truth and inference for the Koch Snoflake shape. }
        \label{fig:appx:2D_reconstructions_snowflake}
\end{figure*}

\subsubsection{Second order supervision constraint}
\label{appx:sec:second_order_supervision}
Normal vectors are often not available during training and are therefore estimated using some local approximation method \cite{cazals2005estimating, guerrero2018pcpnet, ben2020deepfit, ben2019nestinet}. Some methods also provide an approximation of the principal curvatures \cite{cazals2005estimating, ben2020deepfit}. The mean curvatures $\kappa_{mean} = \frac{1}{2}(\kappa_1 + \kappa_2)$ provides second order information that can be utilised as supervision for learning shape representation. We propose a supervised variation to DiGS that penalizes points on the surface for having a different mean curvature than the divergence of the vector field using the following constraint: 
\begin{equation}
L_{curv} = \int_{ \Omega_0} \abs{ \Delta \Phi(x; \theta)} - 2\abs{\kappa_{mean}}dx
\end{equation}
The new supervised loss is given by: 

\begin{equation}
\begin{split}
L_{DiGS+curv} =& \lambda_{A} L_{A} +  \lambda_{B} L_{B} + \\ & \lambda_{C2} L_{C2} + \lambda_{D} L_{D}  + \lambda_{curv} L_{curv}
\end{split}
\end{equation}

Note that normal and curvature estimations are noisy and highly depend on local neighboring points support size, therefore adding these supervisory signals does not guarantee improved performance. 

\clearpage
\pagebreak

%%%%%%%%%%%%%%%%%%%%%%%%%%%%%%%%%%%%%%%%%%%%%%%%%%%%%%%%%%%%%%%%%%%%%%%%%%%%%%%%%%
\subsection{Geometric initialization and training procedure}\label{sec:appx:approach_init}
\subsubsection{Further initialization visualizations}
Following \secref{Sec:approach_init}, we provide a visualization for the proposed geometric initialization with and without multi-frequencies in \figref{appx:geometric_init} which shows the SDF, Eikonal term and divergence term. It shows the sphere-like (circle-like in this 2D case) level sets which provide a much smoother Eikonal and divergence terms. We qualitatively compare it to the initialization method proposed by Sitzmann et. al. \cite{sitzmann2020siren}. This initialization plays a major role in shape space learning.

\subsubsection{Proofs}

The following propositions and proofs show how we intialize our network to a sphere (i.e. the SDF to the function $\Phi(x)=\|x\|_2$). Following Williams~et~al.~\cite{willians2021drps}, rather than roughly approximating the norm with our function class, we instead do a much better approximation to the squared norm, $\Phi(x)=\|x\|_2^2$, (which we do in \eqnref{eg:appx:sq_approx}). We then apply the following function to the output of the SIREN:
\begin{equation}
    \nu(x) = \sign{x}\sqrt{|x| + \varepsilon},
\end{equation}
where $\sign{x}$ is important as we are learning an SDF, and $\varepsilon$ is important for both numerical stability and ensuring the function's derivative is  continuous and sub-differentiable. We use $\varepsilon=10^{-8}$.

\begin{customprop}{4.1}
\label{prop:appx_single_layer}
Let $\Phi$ be a single hidden layer SIREN ($n=1$ in \eqnref{Eq:SIREN}) of dimension $M_n$ and let $x$ be a point within the unit ball. Set,  $\mathbf{w}_n=-\mathbf{1}$, $\mathbf{W}_{n-1}=\frac{\pi}{2}I$, $\mathbf{b}_{n-1}=\frac{\pi}{2}\mathbf{1}$ and $b_n=M_n$. Then, $\nu(\Phi(x)) \approx \norm{x}{2}$.
\end{customprop}

\begin{proof}
For 1D input $z \in [-1,1]$ we can approximate $z^2$ by $1-\sin\left(\frac{\pi}{2}z+\frac{\pi}{2}\right)$ (see \figref{fig:l1_approx}). 

Then
\begin{align}
    \Phi(x) &= \left(-\mathbf{1}\right)^T\sin\left(\left(\frac{\pi}{2}I\right)x+\left(\frac{\pi}{2}\mathbf{1}\right)\right)+M_n\\
    &=\sum_{i=1}^{M_n} 1-\sin\left(\frac{\pi}{2}x_i+\frac{\pi}{2}\right)\\
    \label{eg:appx:sq_approx}
    &\approx \sum_{i=1}^{M_n}x_i^2 \\
    &= \|x\|_2^2
\end{align}
so $\nu(\Phi(x)) \approx \norm{x}{2}$.
\end{proof}

To extend this to networks with an arbitrary number of layers, we design layers $\phi_i$ that preserve the norm on expectation w.r.t. the weights of each layer up to the penultimate layer, i.e., $\mathbb{E}[\norm{\phi_i(x)}{2}]=\norm{x}{2}$ for $i = 1, \ldots, n - 2$. We first prove the following lemmas.

\begin{figure}[tb]
     \centering
     \begin{tabular}{c c c c}
         \rotatebox[origin=c]{90}{SIREN}
         &
         \raisebox{-.5\height}{\includegraphics[width=0.2\linewidth, trim=30pt 30pt 7pt 37pt, clip]{assets/figures/initialization/sdf_128_siren_init.jpg} }
         &
         \raisebox{-.5\height}{\includegraphics[width=0.2\linewidth, trim=30pt 30pt 7pt 37pt, clip]{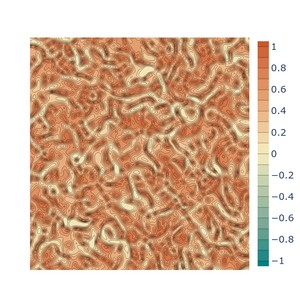} }
         &
         \raisebox{-.5\height}{\includegraphics[width=0.2\linewidth, trim=30pt 30pt 7pt 37pt, clip]{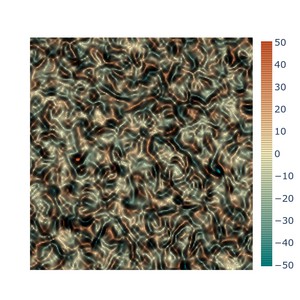}} \\
         \rule{0pt}{0.05ex} &&& \\
         
         \rotatebox[origin=c]{90}{Geometric}
         &
         \raisebox{-.5\height}{\includegraphics[width=0.2\linewidth, trim=30pt 30pt 7pt 37pt, clip]{assets/figures/initialization/sdf_128_geometric_sine_init.jpg} }
         &
         \raisebox{-.5\height}{\includegraphics[width=0.2\linewidth, trim=30pt 30pt 7pt 37pt, clip]{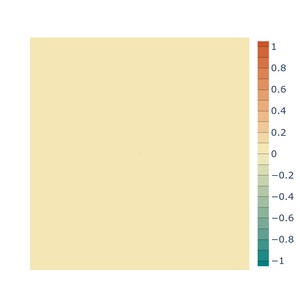} }
         &
         \raisebox{-.5\height}{\includegraphics[width=0.2\linewidth, trim=30pt 30pt 7pt 37pt, clip]{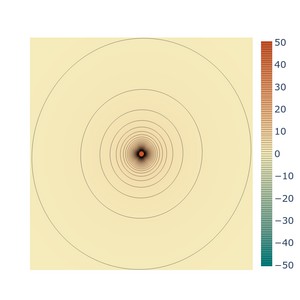}}          \\
         \rule{0pt}{0.05ex} &&& \\

         \rotatebox[origin=c]{90}{MFGI}
         &
         \raisebox{-.5\height}{\includegraphics[width=0.2\linewidth, trim=30pt 30pt 7pt 37pt, clip]{assets/figures/initialization/sdf_128_mfgi_init.jpg} }
         &
         \raisebox{-.5\height}{\includegraphics[width=0.2\linewidth, trim=30pt 30pt 7pt 37pt, clip]{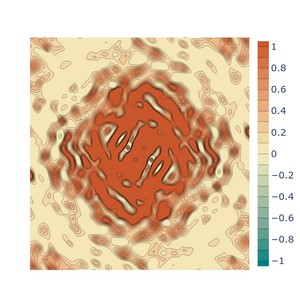} }
         &
         \raisebox{-.5\height}{\includegraphics[width=0.2\linewidth, trim=30pt 30pt 7pt 37pt, clip]{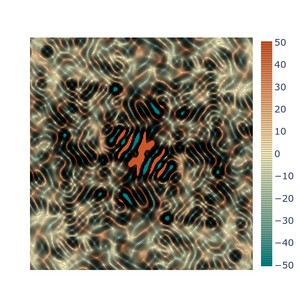}} \\

         &SDF &  Eikonal & divergence \\
     \end{tabular}
        \caption{Visualization of the proposed geometric initialization and multi-frequency geometric initialization for sinusoidal representation networks in 2D compared to Sitzmann et. al. \cite{sitzmann2020siren}. Depicting the  sign distance function (left), eikonal (middle) and, divergence (right).}
        \label{appx:geometric_init}
\end{figure}
\begin{figure}[tb]
    \centering
    \includegraphics[width=0.98\linewidth]{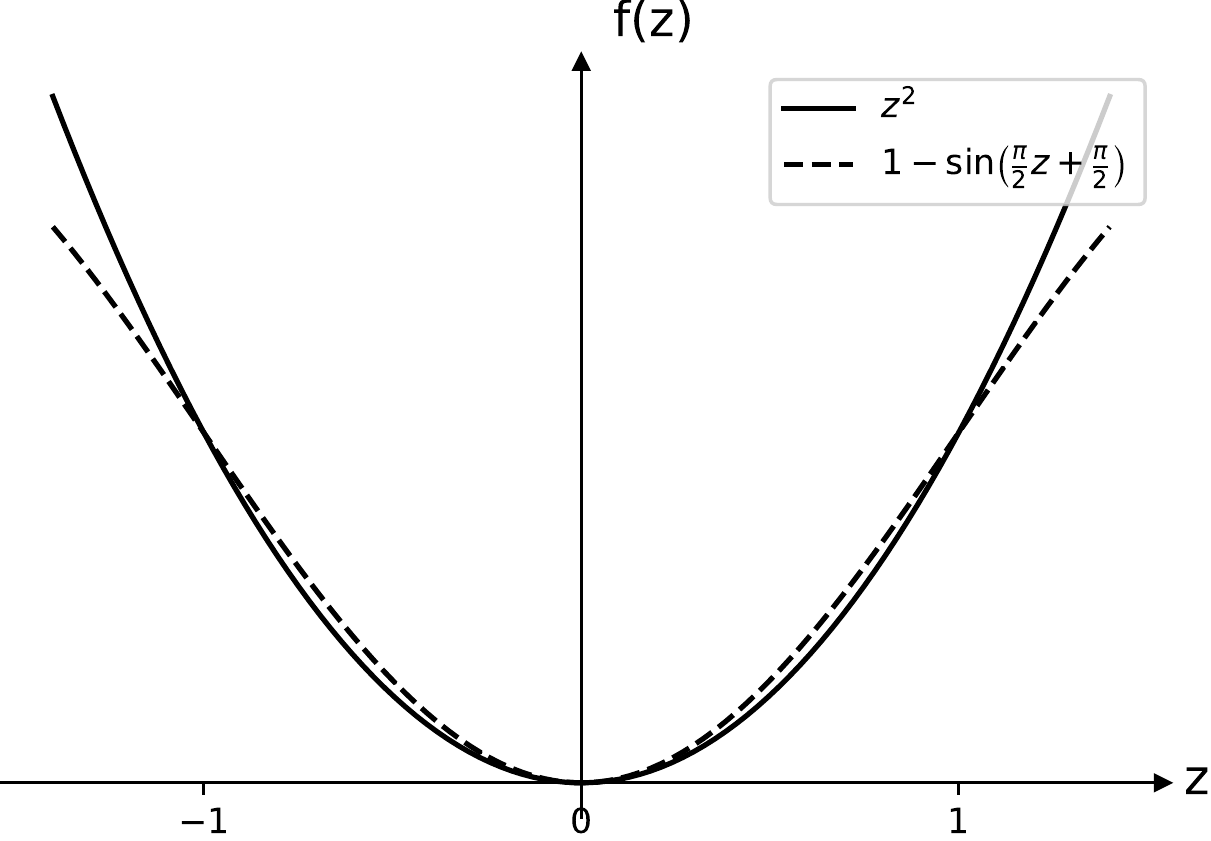}

    \caption{Approximating $z^2$ using $1-\sin\left(\frac{\pi}{2}z+\frac{\pi}{2}\right)$.}
    \label{fig:l1_approx}
\end{figure}

\begin{lemma}\label{lemma:helper1}
    Let $X,Y$ be two random $d$-dimensional vectors with elements $X_i$ and $Y_i$ sampled i.i.d. from $\mathcal{U}\left(-\sqrt{\frac{3}{d}},\sqrt{\frac{3}{d}}\right)$. Then $\expectation{\|X\|} = \expectation{\|Y\|} = 1$ and $\expectation{|\langle X,Y \rangle|} = 0$.
\end{lemma}

In words, random vectors generated i.i.d. from the above distribution are on average unit norm and orthogonal.

\begin{proof}
    We have that $\expectation{X_i}=\expectation{Y_i}=0$ and $\textbf{var}(X_i)=\textbf{var}(Y_i)=\frac{1}{12}\left(2\sqrt{\frac{3}{d}}\right)^2=\frac{1}{d}$. Thus note that $\expectation{X_i^2}=\expectation{Y_i^2}=\expectation{(X_i-\expectation{X_i})^2}=\textbf{var}(X_i)=\frac{1}{d}.$
    By the weak law of large numbers we have that for any $\varepsilon>0$
    \begin{align}
        \P\left(\left|\frac{X_1^2+...+X_d^2}{d} - \expectation{X_i^2} \right| > \varepsilon\right) \leq \frac{\textbf{var}(X_i^2)}{d\varepsilon^2}\\
        \therefore
        \Pr\left(\left|\frac{\|X\|_2^2}{d} - \frac{1}{d} \right|> \varepsilon\right) \leq \frac{\textbf{var}(X_i^2)}{d\varepsilon^2}\\
        \therefore
        \Pr\left(\left|\frac{\|X\|_2^2}{d} - \frac{1}{d} \right|\leq \varepsilon\right) \geq 1- \frac{\textbf{var}(X_i^2)}{d\varepsilon^2}
    \end{align}
    so $\|X\|_2^2 \approx 1$. It follows that $\expectation{\|X\|_2}=\expectation{\|Y\|_2}=1$.
    
    Furthermore, we have that 
    \begin{align}
        \expectation{\langle X, Y \rangle} 
            &= \expectation{\sum_{i=1}^d X_i Y_i}\\
            &= d\,\expectation{X_iY_i}\\
            \label{eq:appx:independent}
            &= d\,\expectation{X_i}\expectation{Y_i}\\
            &= 0. 
    \end{align}
    where \eqnref{eq:appx:independent} follows from them being independent.
\end{proof}

\begin{lemma}\label{lemma:helper2}
    Consider $A\in\reals^{m\times p}$, $p\leq m$, such that $A_{ij}\sim \mathcal{U}\left(-\sqrt{\frac{3}{m}},\sqrt{\frac{3}{m}}\right)$. Then for any $x\in \reals^{p}$ we have that $\|Ax\|_2\approx \|x\|_2$.
\end{lemma}

\begin{proof}
By \lemref{lemma:helper1} the $p$ columns of $A$, which are of dimension $m$, are on average of unit norm and orthogonal to each other (note $p\leq m$). As a result $A^TA\approx I_p$, so
\begin{align*}
    \|Ax\|_2^2
        = x^T A^T A x
        \approx x^T Ix
        = \|x\|_2^2
\end{align*}
implying that $\|Ax\|_2\approx \|x\|_2$.
\end{proof}

\begin{lemma}\label{lemma:helper3}
    Consider $A\in\reals^{m\times p}$, $m>>10$, such that $A_{ij}\sim \mathcal{U}\left(-\sqrt{\frac{3}{m}},\sqrt{\frac{3}{m}}\right)$. Then for $x\in \reals^{p}$ s.t. $\|x\|_2\leq 1$ we have that $\sin(Ax)\approx Ax$.
\end{lemma}
\begin{proof}
By \lemref{lemma:helper2} we have that $\|Ax\|_2\approx \|x\|_2$. Furthermore since $A$ is generated uniform randomly, the values of $Ax\in\reals^m$ should be randomly distributed, therefore as $\|Ax\|_2\approx \|x\|_2\leq 1$ and $m \gg 10$, with high probability $|(Ax)_i|< 0.2$. Thus $\sin\left((Ax)_i\right)\approx (Ax)_i$ (as it is within the linear region of sine), so $\sin(Ax)\approx Ax$.
\end{proof}

\begin{customprop}{4.2}
\label{prop:appx_multi_layer}
Let $\Phi$ be a $n$-hidden layer SIREN (\eqnref{Eq:SIREN}) 
that maps from $\mathbb{R}^{M_0}\to\mathbb{R}$ and $\|x\|_2\leq 1$. Set $\mathbf{W}_i\sim\mathcal{U}\left(-c^i_{wr},c^i_{wr}\right)$, 
$c^i_{wr}=\sqrt{\frac{3}{M_{i+1}}}$,
$\mathbf{b}_i=\mathbf{0}$ for $0\leq i \leq n-2$ and $\mathbf{W}_{n-1}=\frac{\pi}{2}I$, $\mathbf{b}_{n-1}=\frac{\pi}{2} \mathbf{1}$, $\mathbf{w}_n=-\mathbf{1}$ and $b_n=M_{n}$. Then $\nu(\Phi(x))\approx \|x\|_2$.
\end{customprop}

\begin{proof}
We first prove that the input to the last hidden layer, $x_{n-1}$, has the property that $\|x_{n-1}\|_2\approx\|x\|_2$. \propref{prop:appx_single_layer} then implies that $\nu(\Phi(x))\approx \|x_{n-1}\|_2\approx \|x\|_2$, as $x_{n-1}$ is essentially the input to a one hidden layer network satisfying \propref{prop:appx_single_layer}.

Now for for $0\leq i\leq n-2$, if $\|x_i\|_2\leq 1$,
\begin{align}
    \|x_i\|_2
        &\approx \|\mathbf{W}_ix_i\|_2 \tag{$\ast$}\\
        &\approx \|\sin(\mathbf{W}_ix_i+\mathbf{b}_i)\|_2 \tag{$\dagger$}\\
        &= \|x_{i+1}\|_2
\end{align}
where $\ast$ holds from \lemref{lemma:helper2} and $\dagger$ holds due to $\mathbf{b}_i=\mathbf{0}$ and  \lemref{lemma:helper3}. 

Thus as $\|x_0\|_2=\|x\|_2\leq 1$, by induction $\|x\|\approx \|x_{n-1}\|_2$.
\end{proof}

\begin{figure*}
     \centering
     %\begin{tabular}{c c c c c c c c}
     \begin{tabular}{m{1em} c c c c c c}
     \setlength\tabcolsep{0pt}
     \rotatebox[origin=c]{90}{\parbox{1.5cm}{\centering\scriptsize \textbf{Our DiGS}\\L}} &
     \raisebox{-0.5\height}{\includegraphics[width=0.12\textwidth,trim=30pt 30pt 7pt 37pt,clip]{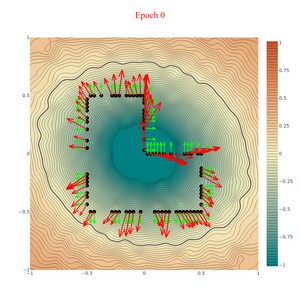}} &
     \raisebox{-0.5\height}{\includegraphics[width=0.12\textwidth,trim=30pt 30pt 7pt 37pt,clip]{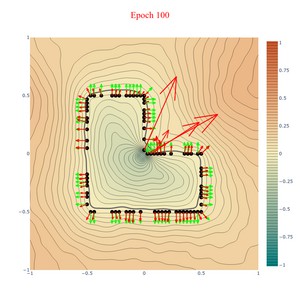}} &
     %\raisebox{-0.5\height}{\includegraphics[width=0.12\textwidth,trim=30pt 30pt 7pt 37pt,clip]{assets/figures/extended_training/digs_L_0500.jpg}} &
     \raisebox{-0.5\height}{\includegraphics[width=0.12\textwidth,trim=30pt 30pt 7pt 37pt,clip]{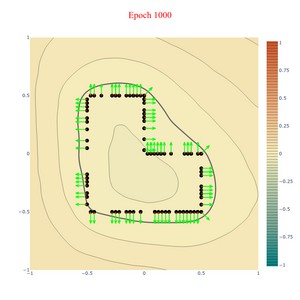}} &
     %\raisebox{-0.5\height}{\includegraphics[width=0.12\textwidth,trim=30pt 30pt 7pt 37pt,clip]{assets/figures/extended_training/digs_L_1500.jpg}} &
     \raisebox{-0.5\height}{\includegraphics[width=0.12\textwidth,trim=30pt 30pt 7pt 37pt,clip]{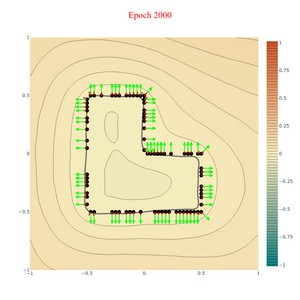}} &
     \raisebox{-0.5\height}{\includegraphics[width=0.12\textwidth,trim=30pt 30pt 7pt 37pt,clip]{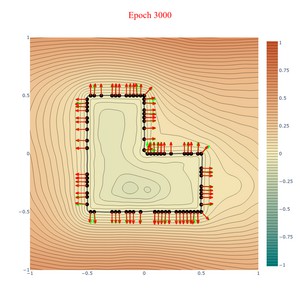}} &
     \raisebox{-0.5\height}{\includegraphics[width=0.12\textwidth,trim=30pt 30pt 7pt 37pt,clip]{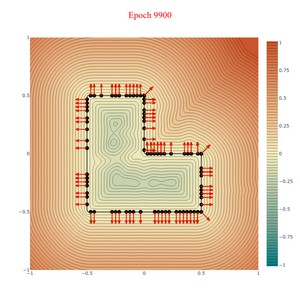}}
     \\
     \rotatebox[origin=c]{90}{\parbox{1.5cm}{\centering\scriptsize SIREN wo n\\ L}} &
     \raisebox{-0.5\height}{\includegraphics[width=0.12\textwidth,trim=30pt 30pt 7pt 37pt,clip]{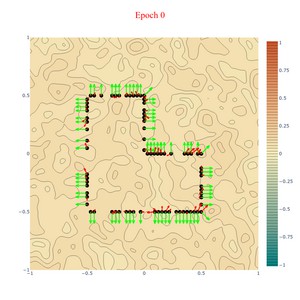}} &
     \raisebox{-0.5\height}{\includegraphics[width=0.12\textwidth,trim=30pt 30pt 7pt 37pt,clip]{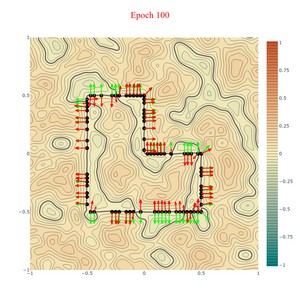}} &
     %\raisebox{-0.5\height}{\includegraphics[width=0.12\textwidth,trim=30pt 30pt 7pt 37pt,clip]{assets/figures/extended_training/swn_L_0500.jpg}} &
     \raisebox{-0.5\height}{\includegraphics[width=0.12\textwidth,trim=30pt 30pt 7pt 37pt,clip]{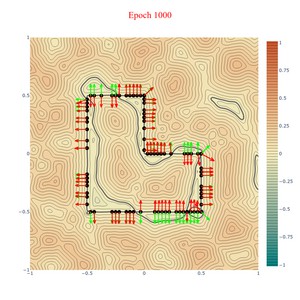}} &
     %\raisebox{-0.5\height}{\includegraphics[width=0.12\textwidth,trim=30pt 30pt 7pt 37pt,clip]{assets/figures/extended_training/swn_L_1500.jpg}} &
     \raisebox{-0.5\height}{\includegraphics[width=0.12\textwidth,trim=30pt 30pt 7pt 37pt,clip]{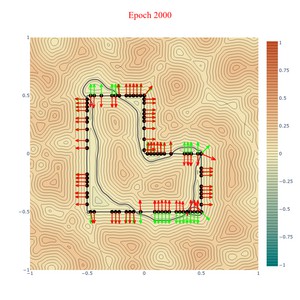}} &
     \raisebox{-0.5\height}{\includegraphics[width=0.12\textwidth,trim=30pt 30pt 7pt 37pt,clip]{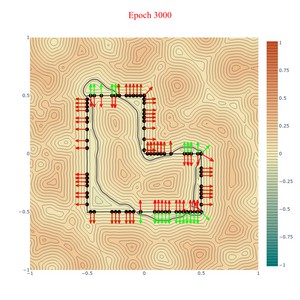}} &
     \raisebox{-0.5\height}{\includegraphics[width=0.12\textwidth,trim=30pt 30pt 7pt 37pt,clip]{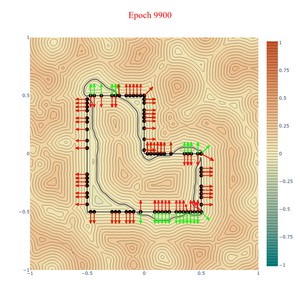}}
     \\
     \rotatebox[origin=c]{90}{\parbox{1.5cm}{\centering\scriptsize \textbf{Our DiGS}\\ Snowflake}} &
     \raisebox{-0.5\height}{\includegraphics[width=0.12\textwidth,trim=30pt 30pt 7pt 37pt,clip]{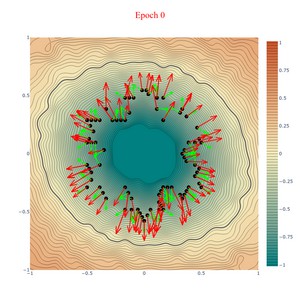}} &
     \raisebox{-0.5\height}{\includegraphics[width=0.12\textwidth,trim=30pt 30pt 7pt 37pt,clip]{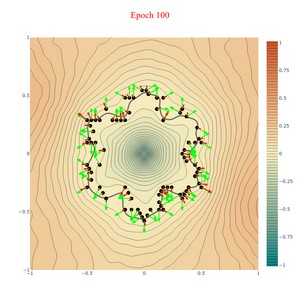}} &
     %\raisebox{-0.5\height}{\includegraphics[width=0.12\textwidth,trim=30pt 30pt 7pt 37pt,clip]{assets/figures/extended_training/digs_snowflake_0500.jpg}} &
     \raisebox{-0.5\height}{\includegraphics[width=0.12\textwidth,trim=30pt 30pt 7pt 37pt,clip]{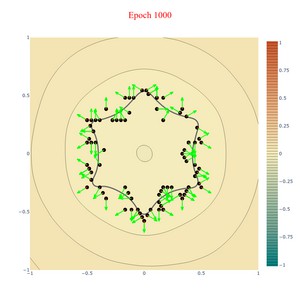}} &
     %\raisebox{-0.5\height}{\includegraphics[width=0.12\textwidth,trim=30pt 30pt 7pt 37pt,clip]{assets/figures/extended_training/digs_snowflake_1500.jpg}} &
     \raisebox{-0.5\height}{\includegraphics[width=0.12\textwidth,trim=30pt 30pt 7pt 37pt,clip]{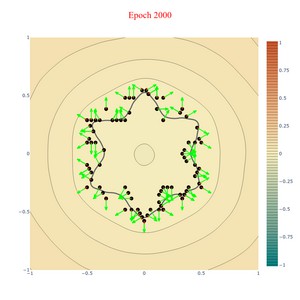}} &
     \raisebox{-0.5\height}{\includegraphics[width=0.12\textwidth,trim=30pt 30pt 7pt 37pt,clip]{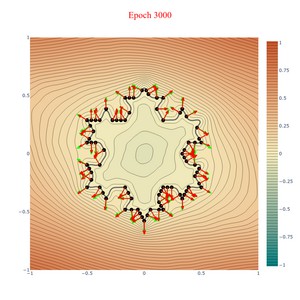}} &
     \raisebox{-0.5\height}{\includegraphics[width=0.12\textwidth,trim=30pt 30pt 7pt 37pt,clip]{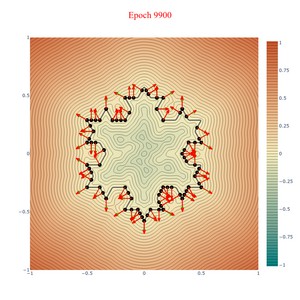}}
     \\
     \rotatebox[origin=c]{90}{\parbox{1.5cm}{\centering\scriptsize SIREN wo n\\ Snowflake}} &
     \raisebox{-0.5\height}{\includegraphics[width=0.12\textwidth,trim=30pt 30pt 7pt 37pt,clip]{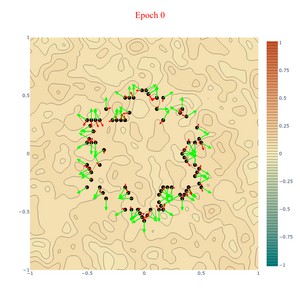}} &
     \raisebox{-0.5\height}{\includegraphics[width=0.12\textwidth,trim=30pt 30pt 7pt 37pt,clip]{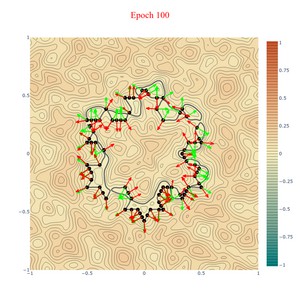}} &
     %\raisebox{-0.5\height}{\includegraphics[width=0.12\textwidth,trim=30pt 30pt 7pt 37pt,clip]{assets/figures/extended_training/swn_snowflake_0500.jpg}} &
     \raisebox{-0.5\height}{\includegraphics[width=0.12\textwidth,trim=30pt 30pt 7pt 37pt,clip]{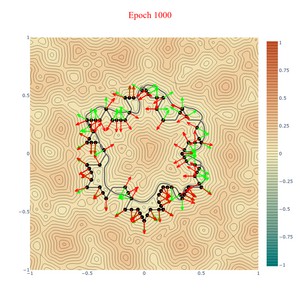}} &
     %\raisebox{-0.5\height}{\includegraphics[width=0.12\textwidth,trim=30pt 30pt 7pt 37pt,clip]{assets/figures/extended_training/swn_snowflake_1500.jpg}} &
     \raisebox{-0.5\height}{\includegraphics[width=0.12\textwidth,trim=30pt 30pt 7pt 37pt,clip]{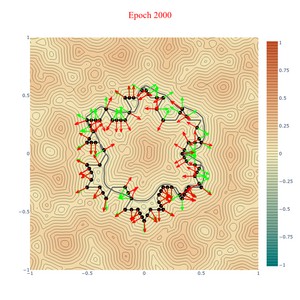}} &
     \raisebox{-0.5\height}{\includegraphics[width=0.12\textwidth,trim=30pt 30pt 7pt 37pt,clip]{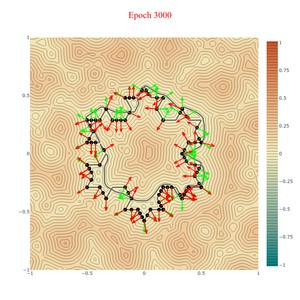}} &
     \raisebox{-0.5\height}{\includegraphics[width=0.12\textwidth,trim=30pt 30pt 7pt 37pt,clip]{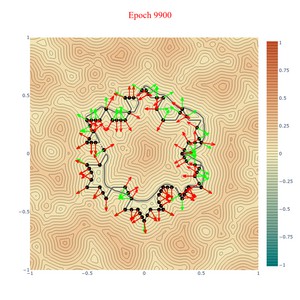}}
     \\
     \rotatebox[origin=c]{90}{\parbox{1.5cm}{\centering\scriptsize DiGS\\ Phase}}
     &
     Init & $\substack{\text{Middle of}\\\text{High Phase}}$ & $\substack{\text{Annealing}\\\text{Phase Starts}}$ & $\substack{\text{Middle of}\\\text{Annealing Phase}}$ & $\substack{\text{Low Phase}\\\text{Starts}}$ & Final Shape
     \end{tabular}
    \caption{Visualisation of DiGS and SIREN wo n at six points during training for 2D shapes. The 6 points are labelled according to their position relative to the four phases of DiGS' training. A contour plot of the learned function is shown: the black dots are surface points, the green arrows are ground truth normals and the red arrows are the current normals at those points.}
    \label{fig:appx:training-procedure-2D}
\end{figure*}
     
\begin{figure*}
     \centering
     %\begin{tabular}{c c c c c c c c}
     \begin{tabular}{m{1em} c c c c c c}
     \setlength\tabcolsep{0pt}
     \rotatebox[origin=c]{90}{\parbox{1.5cm}{\centering\scriptsize \textbf{Our DiGS}\\ Gargoyle}} &
     \raisebox{-0.5\height}{\includegraphics[width=0.12\textwidth]{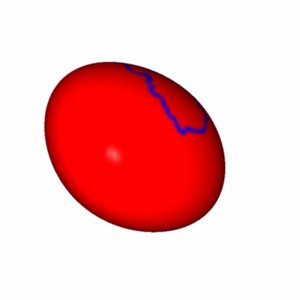}} &
     \raisebox{-0.5\height}{\includegraphics[width=0.12\textwidth]{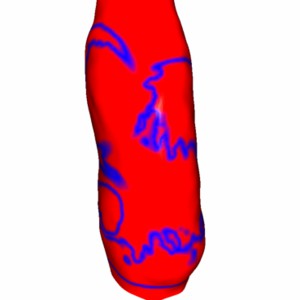}} &
     %\raisebox{-0.5\height}{\includegraphics[width=0.12\textwidth]{assets/figures/extended_training/digs_gargoyle_0500.jpg}} &
     \raisebox{-0.5\height}{\includegraphics[width=0.12\textwidth]{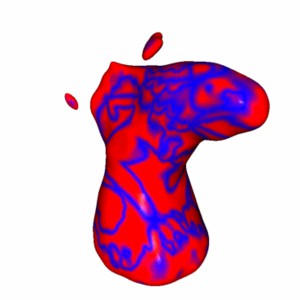}} &
     %\raisebox{-0.5\height}{\includegraphics[width=0.12\textwidth]{assets/figures/extended_training/digs_gargoyle_1500.jpg}} &
     \raisebox{-0.5\height}{\includegraphics[width=0.12\textwidth]{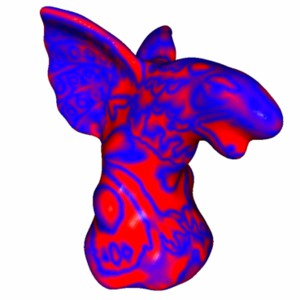}} &
     \raisebox{-0.5\height}{\includegraphics[width=0.12\textwidth]{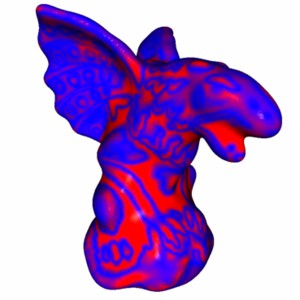}} &
     \raisebox{-0.5\height}{\includegraphics[width=0.12\textwidth]{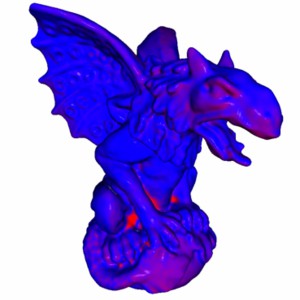}}
     \\
     \rotatebox[origin=c]{90}{\parbox{1.5cm}{\centering\scriptsize SIREN wo n\\ Gargoyle}} &
      &
     \raisebox{-0.5\height}{\includegraphics[width=0.12\textwidth]{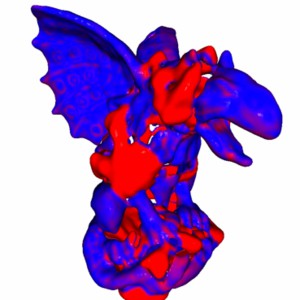}} &
     %\raisebox{-0.5\height}{\includegraphics[width=0.12\textwidth]{assets/figures/extended_training/swn_gargoyle_0500.jpg}} &
     \raisebox{-0.5\height}{\includegraphics[width=0.12\textwidth]{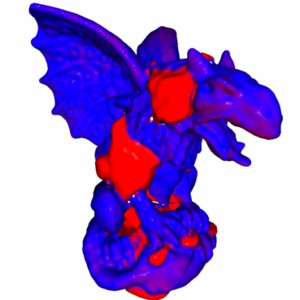}} &
     %\raisebox{-0.5\height}{\includegraphics[width=0.12\textwidth]{assets/figures/extended_training/swn_gargoyle_1500.jpg}} &
     \raisebox{-0.5\height}{\includegraphics[width=0.12\textwidth]{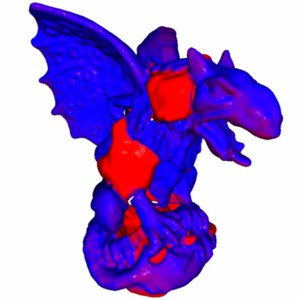}} &
     \raisebox{-0.5\height}{\includegraphics[width=0.12\textwidth]{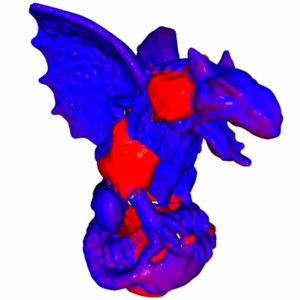}} &
     \raisebox{-0.5\height}{\includegraphics[width=0.12\textwidth]{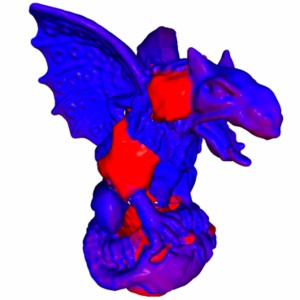}}
     \\
     \rotatebox[origin=c]{90}{\parbox{1.5cm}{\centering\scriptsize \textbf{Our DiGS}\\ DC}} &
     \raisebox{-0.5\height}{\includegraphics[width=0.12\textwidth]{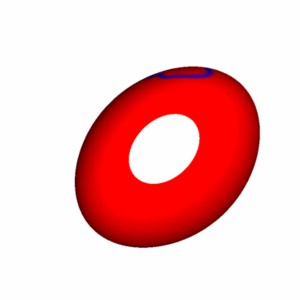}} &
     \raisebox{-0.5\height}{\includegraphics[width=0.12\textwidth]{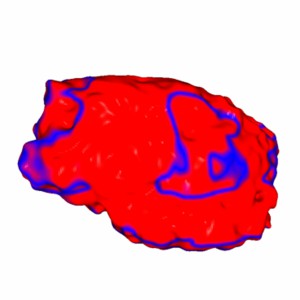}} &
     %\raisebox{-0.5\height}{\includegraphics[width=0.12\textwidth]{assets/figures/extended_training/digs_dc_0500.jpg}} &
     \raisebox{-0.5\height}{\includegraphics[width=0.12\textwidth]{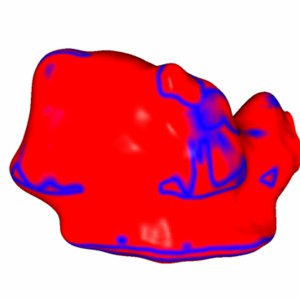}} &
     %\raisebox{-0.5\height}{\includegraphics[width=0.12\textwidth]{assets/figures/extended_training/digs_dc_1500.jpg}} &
     \raisebox{-0.5\height}{\includegraphics[width=0.12\textwidth]{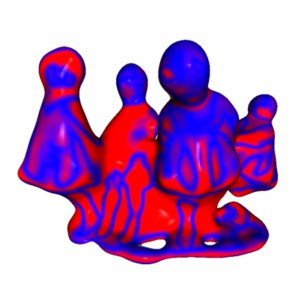}} &
     \raisebox{-0.5\height}{\includegraphics[width=0.12\textwidth]{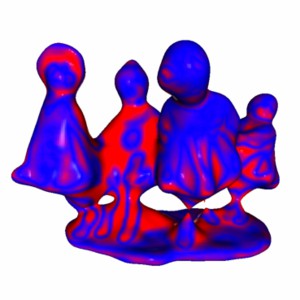}} &
     \raisebox{-0.5\height}{\includegraphics[width=0.12\textwidth]{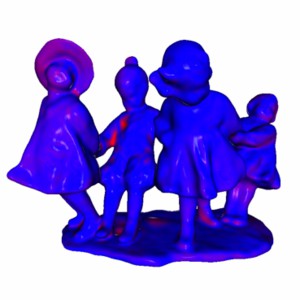}}
     \\
     \rotatebox[origin=c]{90}{\parbox{1.5cm}{\centering\scriptsize SIREN wo n\\ DC}} &
      &
     \raisebox{-0.5\height}{\includegraphics[width=0.12\textwidth]{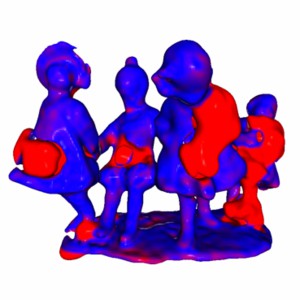}} &
     %\raisebox{-0.5\height}{\includegraphics[width=0.12\textwidth]{assets/figures/extended_training/swn_dc_0500.jpg}} &
     \raisebox{-0.5\height}{\includegraphics[width=0.12\textwidth]{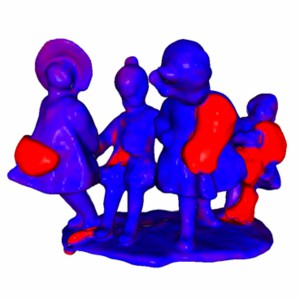}} &
     %\raisebox{-0.5\height}{\includegraphics[width=0.12\textwidth]{assets/figures/extended_training/swn_dc_1500.jpg}} &
     \raisebox{-0.5\height}{\includegraphics[width=0.12\textwidth]{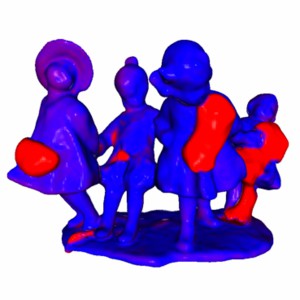}} &
     \raisebox{-0.5\height}{\includegraphics[width=0.12\textwidth]{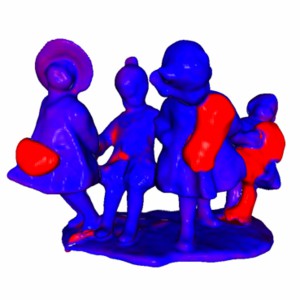}} &
     \raisebox{-0.5\height}{\includegraphics[width=0.12\textwidth]{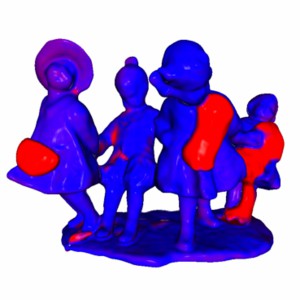}}
     \\
     \rotatebox[origin=c]{90}{\parbox{1.5cm}{\centering\scriptsize DiGS\\ Phase}}
     &
     Init & $\substack{\text{Middle of}\\\text{High Phase}}$ & $\substack{\text{Annealing}\\\text{Phase Starts}}$ & $\substack{\text{Middle of}\\\text{Annealing Phase}}$ & $\substack{\text{Low Phase}\\\text{Starts}}$ & Final Shape
     \end{tabular}
    \caption{Visualisation of DiGS and SIREN wo n at six points during training on 3D shapes from SRB~\cite{berger2013benchmark}. The 6 points are labelled according to their position relative to the four phases of DiGS' training. The current zero level set at the point is shown and colored by the distance to the closest ground truth point (thresholded at one). As SIREN wo n does not have a zero level set at initialization (see the contour plot for the 2D examples), there is no visualization for its initialization in 3D. }
    \label{fig:appx:training-procedure-3D}
\end{figure*}
\subsubsection{DiGS training procedure comparison}
Following \secref{sec:overall-method}, we provide further visualizations in \figref{fig:appx:training-procedure-2D} and \figref{fig:appx:training-procedure-3D} to compare the difference between the training of DiGS and SIREN wo n, highlighting the four phases of DiGS. SIREN wo n fits to the shape very quickly, but does not do so in a consistent manner and thus has multiple surfaces interpolating the surface points that have differing orientations for what is inside and out (best seen with the 2D contours). It then tries to refine and improve upon this, but is stuck with the ghost geometry from its early fitting. DiGS on the other hand has a structured training procedure that prevents this: it slowly changes the SDF for the noisy sphere, allowing more and more details, which reduces incorrect orientation and thus ghost geometry from occurring.  As SIREN wo n does not have a zero level set at initialization (see the contour plot for the 2D examples), there is no visualization for its initialization in 3D. Note that the initialization for DC is the same noisy sphere as the initialization for the gargoyle shape, however it may appear slightly different because it is rotated and cut due to the bounding box for the DC shape.

\clearpage
% \pagebreak
% \pagebreak
%%%%%%%%%%%%%%%%%%%%%%%%%%%%%%%%%%%%%%%%%%%%%%%%%%%%%%%%%%%%%%%%%%%%%%%%%%%%%%%%%%

\subsection{Additional experiments and detail}
\label{sec:appx:additoinal_results}

Code and model parameters available at our project page \url{https://chumbyte.github.io/DiGS-Site/}.

\subsubsection{Evaluation metrics}
To compare between two point sets $\mathcal{X}_1, \mathcal{X}_2 \subset \mathbb{R}^3 $, we use the Chamfer ($d_C$) and Hausdorff ($d_H$) distances from Williams~et~al.~\cite{williams2019DGP}. For the ShapeNet dataset, instead of this Chamfer distance we report the squared Chamfer to be consistent with previous works \cite{park2019deepsdf,williams2020neural_splines}. Note that these metrics compares the accuracy of the predicted surface of the shape.
To compare between the implicit function learnt and the ground truth mesh, we use the volumetric Intersection over Union (IoU) of the interior of the shapes as per Mescheder~et~al.~\cite{mescheder2019occupancy}. Note that this metric compares the underlying occupancy (interior) predicted by the method.

We define the 
\begin{align}
    \label{eq:appx:chamfer}
    d_C(\mathcal{X}_1, \mathcal{X}_2) &= \frac{1}{2}(d_{\vec{C}}(\mathcal{X}_1, \mathcal{X}_2) + d_{\vec{C}}(\mathcal{X}_2, \mathcal{X}_1))\\
    \label{eq:appx:hausdorff}
    d_H(\mathcal{X}_1, \mathcal{X}_2) &= \max(d_{\vec{H}}(\mathcal{X}_1, \mathcal{X}_2),  d_{\vec{H}}(\mathcal{X}_2, \mathcal{X}_1))\\
\end{align}
where
\begin{align}
    d_{\vec{C}}(\mathcal{X}_1, \mathcal{X}_2) &= \frac{1}{\abs{\mathcal{X}_1}} \sum_{x_1 \in \mathcal{X}_1 } \min_{x_2 \in \mathcal{X}_2} \norm{x_1 - x_2}{2}\\
    d_{\vec{H}}(\mathcal{X}_1, \mathcal{X}_2) &= \max_{x_1 \in \mathcal{X}_1 } \min_{x_2 \in \mathcal{X}_2} \norm{x_1 - x_2}{2}
\end{align}
are the one directional Chamfer distance and one directional Hausdorff distance respectively. 

We define the Chamfer distance used in Park~et~al.~\cite{park2019deepsdf} as the squared Chamfer distance, given by

\begin{align}
    \label{eq:appx:squared-chamfer}
    d_C^{sq}(\mathcal{X}_1, \mathcal{X}_2) &= d_{\vec{C}}^{sq}(\mathcal{X}_1, \mathcal{X}_2) + d_{\vec{C}}^{sq}(\mathcal{X}_2, \mathcal{X}_1)\\
    d_{\vec{C}}^{sq}(\mathcal{X}_1, \mathcal{X}_2) &= \frac{1}{\abs{\mathcal{X}_1}} \sum_{x_1 \in \mathcal{X}_1 } \min_{x_2 \in \mathcal{X}_2} \norm{x_1 - x_2}{2}^2.\\
\end{align}

For the volumetric IoU, following Mescheder~et~al.~\cite{mescheder2019occupancy} we obtain unbiased estimates of the occupied (interior) volume of the shapes by randomly sampling 100k points $\mathcal{X}$ in the space. Thus given the known occupancy of the ground truth mesh $O_{GT}(x)\in\{0,1\}$ and the predicted SDF $\Phi(x)\in \reals$ for a points $x\in\mathcal{X}$, the occupancy of the SDF is given by
\begin{align}
    O_{\Phi}(x) = 
        \begin{cases}
            1 & \Phi(x)< 0\\
            0 & \text{otherwise}
        \end{cases}
\end{align}
and the IoU is given by
\begin{align}
    \label{eq:appx:IoU}
    IoU_{\mathcal{X}}(O_{GT},O_{\Phi}) &= \frac{\sum_{x\in\mathcal{X}} O_{GT}(x) \text{ and } O_{\Phi}(x)}{\sum_{x\in\mathcal{X}}O_{GT}(x) \text{ or } O_{\Phi}(x)}.
\end{align}

\begin{table}
    \centering
     \setlength\tabcolsep{3.0pt} % default value: 6pt
    \begin{tabular}{c c c c}
        \toprule
         method & SIREN & IGR & DiGS  \\
         \hline
         time [ms] & 5.2 & 17.5 & 12.0\\ 
         \# parameters & ~66.5K & ~2.1M & ~66.5K\\
         \bottomrule
    \end{tabular}
    % \begin{tabular}{c c c c c}
    %     \toprule
    %      method & SIREN & IGR & DiGS & DiGS + weight reg  \\
    %      \hline
    %      time [ms] & 5.2 & 17.5 & 12.0 & 12.75 \\ 
    %      \# parameters & ~66.5K & ~2.1M & ~66.5K & ~66.5K \\
    %      \bottomrule
    % \end{tabular}
    \caption{Time performance results per iteration. Comparing standard SIREN (excluding the normal estimation stage time) to DiGS.
    % with and without weight regularization
    Results are reported in milliseconds (ms).}
    \label{appx:timing_results}
\end{table}

\subsubsection{Experimental setup and timing performance}
The surface reconstruction experiments on the Surface Reconstruction Benchmark dataset \cite{berger2013benchmark} were implemented in PyTorch and trained on a single Nvidia RTX 2080 GPU. 
The architecture for 2D experiments our method is a 4 layer MLP with sinusoidal activation (SIREN) with 128 nodes in each layer.
The architecture for 3D reconstruction experiments our method is a 4 layer MLP with sinusoidal activation (SIREN) with 256 nodes in each layer, and for scene reconstruction we increase that to 8 layers and 512 nodes. For shapespace we used an 8 layer MLP with sinusoidal activation (SIREN) with 512 nodes in each layer. 
Our method does not add parameters to the standard SIREN approach. \\
Note that due computing the divergence term, there is an increase in computation time which is mainly attributed to computing (and back propagating) the gradient of the gradient. At inference time SIREN and DiGS have the same time and computation performance.  Note that to train the standard SIREN, a preprocessing normal estimation stage is required. The number of parameters and timing performance are reported in \tabref{appx:timing_results} for the 2D reconstruction networks. It shows the increase in training time (per iteration) for DiGS compared to SIREN, however DiGS is still faster and has fewer parameters compared to IGR.

%%%%%%%%%%%%%%% SRB %%%%%%%%%%%%%%%%

\begin{table*}[]\scriptsize
    \renewcommand{\arraystretch}{1.2}
    \centering
    \resizebox{\textwidth}{!}{
    \begin{tabular}{l c c | c c c c | c c c c | c c c c | c c c c | c c c c}
         \toprule
         & \multicolumn{2}{c}{Mean}
         & \multicolumn{4}{c}{Anchor} & \multicolumn{4}{c}{Daratech} & \multicolumn{4}{c}{DC} 
         & \multicolumn{4}{c}{Gargoyle} & \multicolumn{4}{c}{Lord Quas}\\ 
                    & \multicolumn{2}{c}{GT}
                    & \multicolumn{2}{c}{GT} & \multicolumn{2}{c}{Scans}
                    & \multicolumn{2}{c}{GT} & \multicolumn{2}{c}{Scans}
                    & \multicolumn{2}{c}{GT} & \multicolumn{2}{c}{Scans}
                    & \multicolumn{2}{c}{GT} & \multicolumn{2}{c}{Scans}
                    & \multicolumn{2}{c}{GT} & \multicolumn{2}{c}{Scans} \\
          Method    & $d_{C}$ & $d_H$ 
                    & $d_{C}$ & $d_H$ & $d_{\vec{C}}$ & $d_{\vec{H}}$
                    & $d_{C}$ & $d_H$ & $d_{\vec{C}}$ & $d_{\vec{H}}$
                    & $d_{C}$ & $d_H$ & $d_{\vec{C}}$ & $d_{\vec{H}}$
                    & $d_{C}$ & $d_H$ & $d_{\vec{C}}$ & $d_{\vec{H}}$
                    & $d_{C}$ & $d_H$ & $d_{\vec{C}}$ & $d_{\vec{H}}$ \\
            \hline
            
            DGP     
                    & 0.21 & 5.18
                    & 0.33 & 8.82 & \textbf{0.08} & 2.79 
                    & 0.20 & 3.14 & \textbf{0.04} & 1.89
                    & 0.18 & 4.31 & \textbf{0.04} & \textbf{2.53}
                    & 0.21 & 5.98 & \textbf{0.06} & 3.41 
                    & 0.14 & 3.67 & \textbf{0.04} & 2.03\\
            IGR     
                    & 0.19 & 2.99
                    & 0.23 & 4.71 & 0.12 & 1.32
                    & 0.25 & 4.01 & 0.14 & 1.59
                    & 0.17 & 2.22 & 0.09 & 2.61
                    & 0.18 & \textbf{2.85} & 0.1 & 1.29
                    & 0.12 & 1.17 & 0.07 & 0.98 \\ 
            SIREN   
                    & 0.19 & 3.86
                    & 0.31 & 7.32 & 0.11 & 1.23
                    & 0.21 & 4.74 & 0.09 & 1.85 
                    & 0.15 & 2.37 & 0.07 & 2.71 
                    & 0.17 & 4.26 & 0.09 & \textbf{0.82} 
                    & 0.12 & \textbf{0.62} & 0.08 & 0.81 \\
            NSP
                    & 0.17 & 2.85
                    & 0.22 & 4.65 & 0.11 & \textbf{1.11}
                    & 0.21 & 4.35 & 0.08 & \textbf{1.14}
                    & \textbf{0.14} & \textbf{1.35} & 0.06 & 2.75
                    & \textbf{0.16} & 3.20 & 0.08 & 2.75
                    & 0.12 & 0.69 & 0.05 & \textbf{0.62}\\
            PHASE
                    & \textbf{0.16} & \textbf{2.77}
                    & \textbf{0.21} & \textbf{4.29} & 0.09 & 1.23
                    & \textbf{0.18} & \textbf{2.92} & 0.08 & 1.80
                    & 0.15 & 2.52 & 0.05 & 2.78
                    & \textbf{0.16} & 3.14 & 0.07 & 1.09
                    & \textbf{0.11} & 0.96 & 0.04 & 0.96\\
            DiGS + n
                    & 0.18 & 3.55
                    & 0.28 & 5.71 & 0.11 & 1.14 
                    & 0.21 & 5.02 & 0.09 & 1.75
                    & 0.15 & 2.13 & 0.06 & 2.74 
                    & \textbf{0.16} & 3.81 & 0.09 & 0.90 
                    & 0.12 & 1.1 & 0.06 & 0.77 \\

            \hline

            IGR wo n
                    & 1.38 & 16.33
                    & 0.45 & 7.45 & 0.17 & 4.55 
                    & 4.9 & 42.15 & 0.7 & 3.68
                    & 0.63 & 10.35 & 0.14 & 3.44
                    & 0.77 & 17.46 & 0.18 & 2.04 
                    & 0.16 & 4.22 & 0.08 & 1.14 \\
            SIREN wo n
                    & 0.42 & 7.67
                    & 0.72 & 10.98 & 0.11 & 1.27
                    & 0.21 & 4.37 & 0.09 & 1.78 
                    & 0.34 & 6.27 & 0.06 & \textbf{2.71} 
                    & 0.46 & 7.76 & 0.08 & \textbf{0.68}
                    & 0.35 & 8.96 & 0.06 & \textbf{0.65} \\
            SAL
                    & 0.36 & 7.47
                    & 0.42 & 7.21 & 0.17 & 4.67
                    & 0.62 & 13.21 & 0.11 & 2.15
                    & 0.18 & 3.06 & 0.08 & 2.82
                    & 0.45 & 9.74 & 0.21 & 3.84
                    & 0.13 & 4.14 & 0.07 & 4.04 \\
            IGR+FF
                    & 0.96 & 11.06
                    & 0.72 & 9.48 & 0.24 & 8.89
                    & 2.48 & 19.6 & 0.74 & 4.23
                    & 0.86 & 10.3 & 0.28 & 3.98
                    & 0.26 & 5.24 & 0.18 & 2.93
                    & 0.49 & 10.7 & 0.14 & 3.71\\
            PHASE+FF
                    & 0.22 & 4.96
                    & \textbf{0.29} & 7.43 & \textbf{0.09} & 1.49
                    & 0.35 & 7.24 & \textbf{0.08} & \textbf{1.21}
                    & 0.19 & 4.65 & \textbf{0.05} & 2.78
                    & \textbf{0.17} & 4.79 & \textbf{0.07} & 1.58
                    & \textbf{0.11} & \textbf{0.71} & \textbf{0.05} & 0.74\\
            Our DiGS
                    & \textbf{0.19} & \textbf{3.52}
                    & \textbf{0.29} & \textbf{7.19} & 0.11 & \textbf{1.17}
                    & \textbf{0.20} & \textbf{3.72} & 0.09 & 1.80
                    & \textbf{0.15} & \textbf{1.70} & 0.07 & 2.75
                    & \textbf{0.17} & \textbf{4.10} & 0.09 & 0.92
                    & 0.12 & 0.91 & 0.06 & 0.70\\

         \bottomrule
    \end{tabular}
    }
    \caption{Results on the Surface Reconstruction Benchmark using Chamfer $d_{C}$, Hausdorff distance $d_H$. We compare methods with normal supervision above the line and without normal supervision below the line.  The \textit{scans} column reports the one sided distances ($d_{\vec{C}}, d_{\vec{H}}$) between the reconstruction and the simulated scans which give a measure of the reconstruction's overfit to the noisy input. }
    \label{tab:results:full_srb}
\end{table*}

\begin{figure*}
     \centering
     \begin{tabular}{c c c | c c c c}
         
         \includegraphics[width=\recviscol\linewidth]{assets/figures/surface_recon/DiGS_MFGI/anchor.ply.jpg}
         &
         \includegraphics[width=\recviscol\linewidth]{assets/figures/surface_recon/IGR_wo_n/anchor.ply.jpg}
         &
         \includegraphics[width=\recviscol\linewidth]{assets/figures/surface_recon/SIREN_wo_n/anchor.ply.jpg}
         & 
         \includegraphics[width=\recviscol\linewidth]{assets/figures/surface_recon/DiGS_MFGI_w_n/anchor.ply.jpg}
         & 
         \includegraphics[width=\recviscol\linewidth]{assets/figures/surface_recon/SIREN/anchor.ply.jpg}
         & 
         \includegraphics[width=\recviscol\linewidth]{assets/figures/surface_recon/IGR/anchor.ply.jpg} \\

         \includegraphics[width=\recviscol\linewidth]{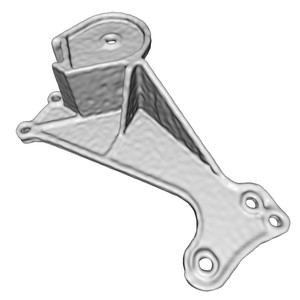}
         &
         \includegraphics[width=\recviscol\linewidth]{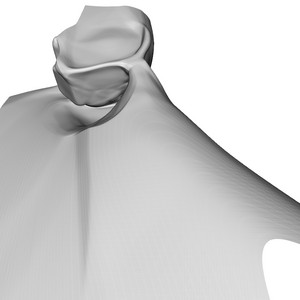}
         &
         \includegraphics[width=\recviscol\linewidth]{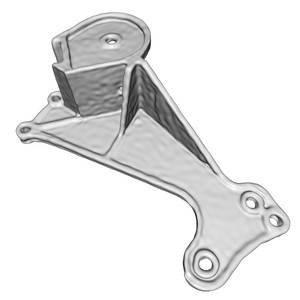}
         & 
         \includegraphics[width=\recviscol\linewidth]{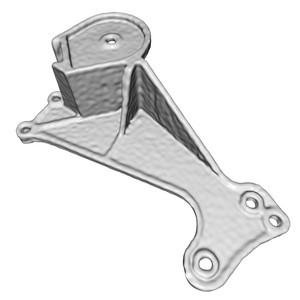}
         & 
         \includegraphics[width=\recviscol\linewidth]{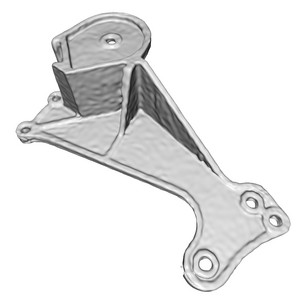}
         & 
         \includegraphics[width=\recviscol\linewidth]{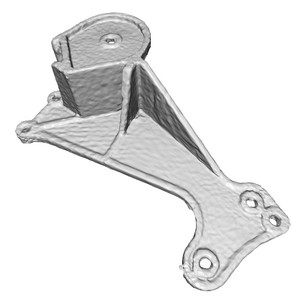} \\
         
         \includegraphics[width=\recviscol\linewidth]{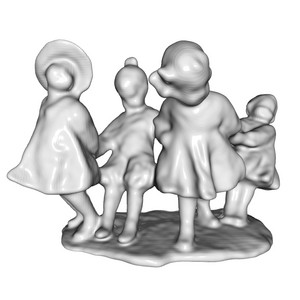} 
         &
         \includegraphics[width=\recviscol\linewidth]{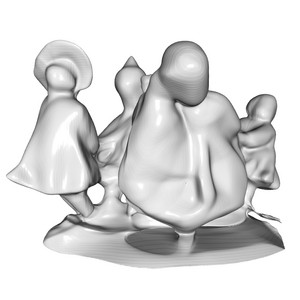} 
         &
         \includegraphics[width=\recviscol\linewidth]{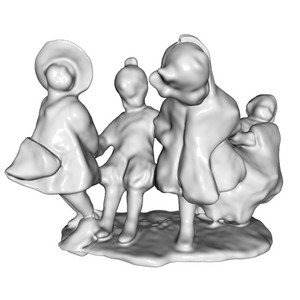}
         & 
         \includegraphics[width=\recviscol\linewidth]{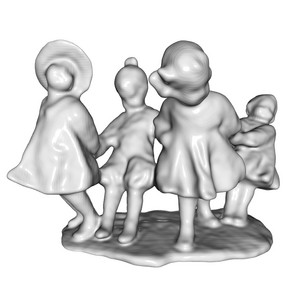}
         & 
         \includegraphics[width=\recviscol\linewidth]{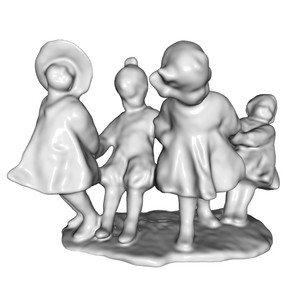}
         & 
         \includegraphics[width=\recviscol\linewidth]{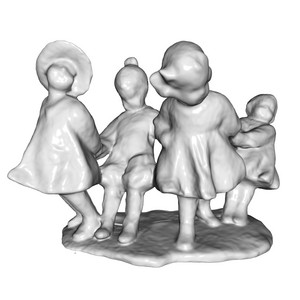} \\

         \includegraphics[width=\recviscol\linewidth]{assets/figures/surface_recon/DiGS_MFGI/gargoyle.ply.jpg} 
         &
         \includegraphics[width=\recviscol\linewidth]{assets/figures/surface_recon/IGR_wo_n/gargoyle.ply.jpg} 
         &
         \includegraphics[width=\recviscol\linewidth]{assets/figures/surface_recon/SIREN_wo_n/gargoyle.ply.jpg} 
         & 
         \includegraphics[width=\recviscol\linewidth]{assets/figures/surface_recon/DiGS_MFGI_w_n/gargoyle.ply.jpg}
         & 
         \includegraphics[width=\recviscol\linewidth]{assets/figures/surface_recon/SIREN/gargoyle.ply.jpg} 
         & 
         \includegraphics[width=\recviscol\linewidth]{assets/figures/surface_recon/IGR/gargoyle.ply.jpg} \\

         \includegraphics[width=\recviscol\linewidth]{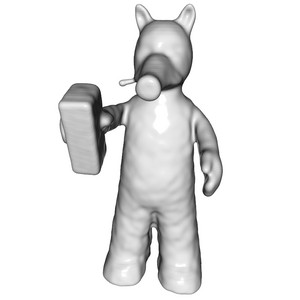} 
         &
        \includegraphics[width=\recviscol\linewidth]{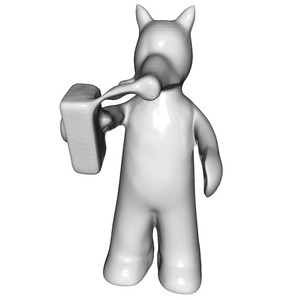} 
         &
         \includegraphics[width=\recviscol\linewidth]{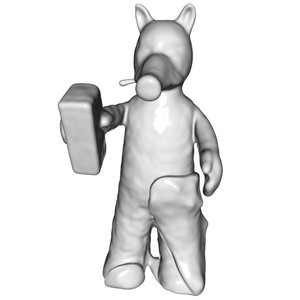} 
         & 
         \includegraphics[width=\recviscol\linewidth]{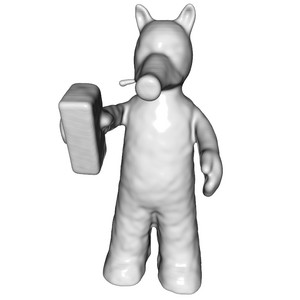}
         & 
         \includegraphics[width=\recviscol\linewidth]{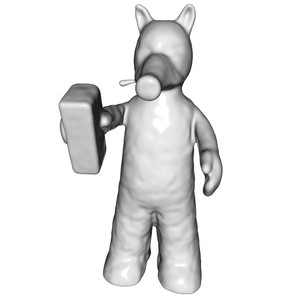} 
         & 
         \includegraphics[width=\recviscol\linewidth]{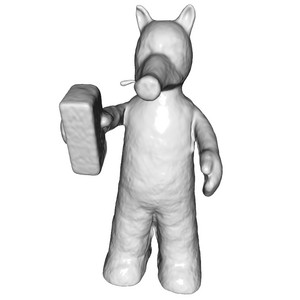} \\         
         
         \textbf{Our DiGs} & IGR wo n & SIREN wo n  & \textbf{Our DiGS} + n & SIREN & IGR \\
         \multicolumn{3}{c}{without normals} & \multicolumn{2}{c}{ with normals}

     \end{tabular}
        \caption{Qualitative results of surface reconstruction on the anchor and gargoyle shapes from the Surface Reconstruction Benchmark \cite{berger2013benchmark} compared to state of the art approaches (IGR, SIREN) that use normal vectors as ground truth. ~\\~\\}
        \label{fig:appx:surface_recon_qual_results}
\end{figure*}

\subsubsection{Surface reconstruction on SRB}\label{sec:appx:srb-details}
\textbf{Implementation details.}
The input point cloud is first centered to zero and scaled to have maximum norm of one. Then a bounding box that is 1.1 times the size of the shape is selected. Each iteration we sample 15,000 points from the original point cloud and sample 15,000 points uniformly randomly in a bounding box. We train for 10,000 iterations with a learning rate of 5e-5.

We report the values for DGP~\cite{williams2019DGP}, NSP~\cite{williams2020neural_splines} and PHASE/PHASE+FF~\cite{lipman2021phase} from their respective works, and report the results for FFN \cite{tancik2020fourfeat} from Williams~\etal~\cite{williams2020neural_splines} and IGR+FF from Lipman~et~al.~\cite{lipman2021phase}. We report the results for SAL~\cite{atzmon2020sal}, IGR/IGR~wo~n~\cite{gropp2020igr} and SIREN/SIREN~wo~n~\cite{sitzmann2020siren} using their code.

\textbf{Additional quantitative results. }
We provide additional quantitative results for surface reconstruction on the Surface Reconstruction Benchmark \cite{berger2013benchmark}. \tabref{tab:results:full_srb} is an extended version of \tabref{tab:results:surf-recon-summary} from the main paper that includes comparison to additional methods with normal vector supervision. The results show that we are able to achieve the best reconstruction compared to other methods without normal supervision. Additionally,we achieve improved performance when using normal vector supervision compared to a vanilla SIREN and show comparable results to other methods.

\textbf{Additional qualitative results. }
We provide additional qualitative results for surface reconstruction on the Surface Reconstruction Benchmark \cite{berger2013benchmark}. \figref{fig:appx:surface_recon_qual_results} shows visualizations of the output reconstruction of different methods that use normal vector as ground truth as well as methods that do not. It is clear that whenever ground truth normal vectors are not available, DiGS presents a significant improvement and yields comparable results to normal based methods. 

\vspace{0.05in}

\begin{table}[!t]\scriptsize
    \centering
    \setlength\tabcolsep{4.5pt} % default value: 6pt
    \begin{tabular}{l c c c c}
         \toprule
           & \multicolumn{2}{c}{GT} & \multicolumn{2}{c}{Scans}\\
           Method & $d_{C}$ & $d_H$ & $d_{\vec{C}}$ & $d_{\vec{H}}$\\
            \hline
              SIREN wo n & 0.42 & 7.67  & 0.08 & \textbf{1.42} \\
              DiGS $L_1$ & 0.20 & 4.47 & \textbf{0.07} & 1.77\\ 
              DiGS $L_2$ & 0.25 & 5.18 & \textbf{0.07} & 1.67\\ 
              DiGS no-decay & 0.21 & 4.45 & 0.09 & 2.58\\ 
              DiGS lin-decay & 0.20 & 4.41 & \textbf{0.07} & 1.64\\ 
              DiGS curv & 0.30 & 4.83 & 0.10 & 1.57\\ 
              DiGS $L_1$ MFGI & \textbf{0.19} & \textbf{3.52} & 0.08 & 1.47\\ 

         \bottomrule
    \end{tabular}
    \caption{DiGS ablation study on the Surface Reconstruction Benchmark \cite{berger2013benchmark}. We compare design choices such as annealing method (linear, step and no annealing), divergence term penalties ($L_1$ vs $L_2$), initialization method (MFGI vs SIREN) and curvature ground truth. We use the Chamfer $d_{C}$ and Hausdorff distance $d_H$ metrics.}
    \label{tab:results:surface reconstructon ablations_avg}
\end{table}

% \begin{table}[!t]\scriptsize
%     \centering
%     \setlength\tabcolsep{4.5pt} % default value: 6pt
%     \begin{tabular}{l  c c c c | c c }
%          \toprule
%           & \multicolumn{2}{c}{GT} & \multicolumn{2}{c}{Scans}  & &  \\
%           Method & $d_{C}$ & $d_H$ & $d_{\vec{C}}$ & $d_{\vec{H}}$  
%           & mLCV$_{k=10}$ & mLCV$_{k=50}$ \\
%             \hline
%               DiGS $L_1$ & 0.2 & 4.47 & 0.07 & 1.77 & 0.66 & 2.4 \\ 
%               DiGS $L_2$ & 0.25 & 5.18 & 0.07 & 1.67 & 0.86 & 3.49 \\ 
%               DiGS nodecay & 0.21 & 4.45 & 0.09 & 2.58 & \textbf{0.61} & \textbf{2.07} \\ 
%               DiGS lindecay & 0.2 & 4.41 & 0.07 & 1.64 & 0.65 & 2.38 \\ 
%               DiGS curv & 0.3 & 4.83 & 0.1 & 1.57 & 0.8 & 2.98 \\ 
%               DiGS $L_1$ MFGI & \textbf{0.19} & \textbf{3.84} & 0.08 & \textbf{1.43} & 0.62 & 2.09 \\ 

%          \bottomrule
%     \end{tabular}
%     \caption{DiGS ablation study on the Surface Reconstruction Benchmark \cite{berger2013benchmark}. We compare design choices such as annealing method (linear, step and no annealing), divergence term penalties ($L_1$ vs $L_2$), initialization method (MFGI vs SIREN) and curvature ground truth. We use the Chamfer $d_{C}$ and Hausdorff distance $d_H$ metrics and also the proposed mLCV at 10 and 50 neighboring points and report average over all shapes.}
%     \label{tab:results:surface reconstructon ablations_avg}
% \end{table}

\begin{table}[]\scriptsize
    \centering
     \setlength\tabcolsep{3.0pt} % default value: 6pt
    \begin{tabular}{l c c c c c}
         \toprule
           & & \multicolumn{2}{c}{GT} & \multicolumn{2}{c}{Scans}\\
          & Method & $d_{C}$ & $d_H$ & $d_{\vec{C}}$ & $d_{\vec{H}}$\\
            \hline
            \multirow{5}{*}{anchor}  
                                     & SIREN wo n & 0.72 & 10.98 & 0.11 & 1.27\\
                                    & DiGS l1 & 0.33 & 8.71 & 0.11 & 2.52\\ 
                                     & DiGS l2 & 0.43 & 8.24 & \textbf{0.10} & 2.23\\ 
                                     & DiGS no-decay & 0.35 & 8.84 & 0.12 & 4.38\\ 
                                     & DiGS lin-decay & 0.33 & 8.71 & \textbf{0.10} & 2.05\\ 
                                     & DiGS curv & 0.78 & 11.63 & 0.13 & 1.21\\ 
                                     & DiGS l1 MFGI & \textbf{0.29} & \textbf{7.19} & 0.11 & \textbf{1.17}\\ 
            \hline
            \multirow{5}{*}{daratech}     
                                         & SIREN wo n & 0.21 & 4.37 & 0.09 & 1.78 \\
                                        & DiGS l1 & 0.21 & 3.50 & \textbf{0.07} & 1.81\\ 
                                         & DiGS l2 & 0.20 & 3.47 & \textbf{0.07} & 1.78\\ 
                                         & DiGS no-decay & 0.20 & 3.35 & 0.09 & 1.79\\ 
                                         & DiGS lin-decay & \textbf{0.19} & \textbf{2.97} & \textbf{0.07} & 1.77\\  
                                         & DiGS curv & 0.22 & 4.27 & 0.12 & \textbf{1.76}\\ 
                                         & DiGS l1 MFGI & 0.20 & 3.72 & 0.09 & 1.80\\ 
            \hline
            \multirow{5}{*}{dc}   
                                 & SIREN wo n & 0.34 & 6.27 & \textbf{0.06} & \textbf{2.71} \\
                                 & DiGS l1 & \textbf{0.15} & 2.25 & \textbf{0.06} & 2.76\\ 
                                 & DiGS l2 & 0.19 & 4.04 & \textbf{0.06} & 2.76\\ 
                                 & DiGS no-decay & 0.16 & 2.58 & 0.07 & 2.78\\ 
                                 & DiGS lin-decay & 0.16 & 2.72 & \textbf{0.06} & \textbf{2.71}\\ 
                                  & DiGS curv & 0.16 & 1.76 & 0.08 & 2.82\\ 
                                 & DiGS l1 MFGI & \textbf{0.15} & \textbf{1.70} & 0.07 & 2.75\\ 
            \hline
            \multirow{5}{*}{gargoyle}    
                                    & SIREN wo n & 0.46 & 7.76 & \textbf{0.08} & \textbf{0.68} \\
                                    & DiGS l1 & \textbf{0.17} & 5.12 & \textbf{0.08} & 0.81\\ 
                                         & DiGS l2 & \textbf{0.17} & 5.02 & \textbf{0.08} & 0.75\\ 
                                         & DiGS no-decay & 0.18 & 5.18 & 0.09 & 3.10\\ 
                                         & DiGS lin-decay & \textbf{0.17} & 5.09 & \textbf{0.08} & 0.81\\ 
                                          & DiGS curv & 0.19 & \textbf{3.90} & 0.12 & 1.30\\ 
                                         & DiGS l1 MFGI & \textbf{0.17} & 4.10 & 0.09 & 0.92\\ 
            \hline
            \multirow{5}{*}{lord quas}    
                                         & SIREN wo n & 0.35 & 8.96 & \textbf{0.06} & \textbf{0.65} \\
                                        & DiGS l1 & \textbf{0.12} & 2.77 & \textbf{0.06} & 0.94\\ 
                                         & DiGS l2 & 0.25 & 5.12 & \textbf{0.06} & 0.85\\ 
                                         & DiGS no-decay & 0.13 & 2.30 & \textbf{0.06} & 0.85\\ 
                                         & DiGS lin-decay & 0.13 & 2.59 & \textbf{0.06} & 0.88\\ 
                                          & DiGS curv & 0.14 & 2.61 & 0.07 & 0.75\\ 
                                         & DiGS l1 MFGI & \textbf{0.12} & \textbf{0.91} & \textbf{0.06} & 0.70\\

         \bottomrule
    \end{tabular}
    \caption{DiGS ablation study on the Surface Reconstruction Benchmark. We compare design choices such as annealing method (linear, step and no annealing), divergence term penalties ($L_1$ vs $L_2$), initialization method (MFGI vs SIREN) and curvature ground truth. We use the Chamfer $d_{C}$ and Hausdorff distance $d_H$ metrics.}
    \label{tab:results:surface reconstructon ablations}
\end{table}

\textbf{Ablation study. }
We investigate the effects of several design choices made for DiGS and report the averages over all shapes in the dataset in \tabref{tab:results:surface reconstructon ablations_avg} and individual shapes in \tabref{tab:results:surface reconstructon ablations}.
First we investigate the influence of the annealing function $\tau$. We compare between the case of 
\begin{enumerate}
    \item No annealing:
    \begin{equation}
        \tau = 1.
    \end{equation}
    \item Linear annealing: 
    % $\tau_{Linear}(t_0, t_1, \tau_1)=1$ if  $t < t_0$, $\tau_{Linear}(t_0, t_1, \tau_1)=1 + \left(\tau_1 - 1\right)\frac{(t - t_0)}{t_1 - t_0}$ if $t_0 \leq t \leq t_1$, $\tau_{Linear}(t_0, t_1, \tau_1)=\tau_1$ if  $t > t_1$.
    \begin{equation}
        \tau_{Lin}(t_0, t_1, \tau_1) = \begin{cases} 
                1 &  t < t_0  \\
                1 + \left(\tau_1 - 1\right)\frac{(t - t_0)}{t_1 - t_0} & t_0 \leq t \leq t_1 \\
                \tau_1 &  t > t_1\\
            \end{cases}
    \end{equation}
    \item Step annealing: 
        % $\tau_{Step}(t_0, \tau_1)= 1$ if $t < t_0 $, $\tau_{Step}=\tau_1$, $t \geq t_0$
        \begin{equation}
        \tau_{Step}(t_0, \tau_1)= \begin{cases} 
                1 &  t < t_0  \\
                \tau_1 & t \geq t_0 \\
            \end{cases}
    \end{equation}

\end{enumerate}
% (1) no annealing $\tau = 1$, (2) linear annealing $\tau_{Linear}(t_0, t_1, \tau_1)=1$ if  $t < t_0$, $\tau_{Linear}(t_0, t_1, \tau_1)=1 + \left(\tau_1 - 1\right)\frac{(t - t_0)}{t_1 - t_0}$ if $t_0 \leq t \leq t_1$, $\tau_{Linear}(t_0, t_1, \tau_1)=\tau_1$ if  $t > t_1$, 
% $\tau_{Linear}(t_0, t_1, \tau_1) = \begin{cases} 
%             1 &  t < t_0    \\
%             1 + \left(\tau_1 - 1\right)\frac{(t - t_0)}{t_1 - t_0} & t_0 \leq t \geq t_1 \\
%             \tau_1 &  t > t_1 \\
%         \end{cases}$
% and (3) step annealing $\tau_{Step}(t_0, \tau_1)= 1$ if $t < t_0 $, $\tau_{Step}=\tau_1$, $t \geq t_0$.
% $\tau_{Step}(t_0, \tau_1) = \begin{cases} 
%             1 & t < t_0    \\
%             \tau_1 & t \geq t_0
%         \end{cases}$.
Here $t_0$ and $t_1$ are expressed as a fraction of training iterations. In our experiments we used the following values: $(t_0, t_1, \tau_1)=(0.5, 0.75, 0)$.
Results show that annealing improves performance with little difference between linear and step annealing. 

We also investigate the choice of penalty function over the divergence term (with step decay) and compare between $L_1$ and $L_2$. Here, $L_1$ allows for sparse spatial locations to have high divergence (which is desired because there may be some source or sink point in the gradient vector field) while $L_2$ provides an average low value over the volume. As expected, the results show an advantage for using $L_1$ penalty on the divergence. 

Furthermore, we evaluate the second order supervision approach presented in \secref{appx:sec:second_order_supervision}. For that, we estimate the mean curvature using  DeepFit~\cite{ben2020deepfit} (a recent state-of-the-art method for estimating normals and curvatures) and introduce the mean curvature as ground truth information during training.  The results show that, surprisingly, curvature supervision does not provide any benefit. This is due to the noisiness and local support dependency of curvature estimation which may have over smoothed or jittery values that make training inconsistent.

Finally, we investigate the effect of the multi-frequency geometric initialization (with step decay and $L_1$ penalty) and show that it produces consistently better performance than all other ablations. This variant is denoted throughout the paper as DiGS, \ie, $L_1$ penalty on the divergence term with step decay and MFGI (\secref{Sec:approach_init}). Note that all ablations were trained without normal vector information. 

\textbf{Use of DiGS Loss with other activation functions. }
Our divergence loss is general and can be applied to any network that has second order derivatives defined everywhere. Note that this means it cannot be used with ReLU activations, though we can use smooth approximations such as the SoftPlus activation function (which IGR \cite{gropp2020igr} uses). \\
However, our loss is targeted at high frequency architectures such as SIRENs. A drawback of such architectures is the trade off between high fidelity detail and ghost geometries, especially without normal vector supervision.  Our loss is particularly beneficial to alleviate this trade-off.\\
We experimented with using our loss paired with SoftPlus and found that it does not provide a boost in performance since the network is already biased towards low frequency solutions (performance reduction of 0.99 $d_C$, and 6.34 $d_H$ on SRB) which \textit{SIREN wo n} outperforms even without the divergence loss performance boost.

\subsubsection{Surface reconstruction on ShapeNet}
\textbf{Implementation details.}
We use the preprocessing and evaluation method from Williams~et~al.~\cite{williams2020neural_splines}. They first preprocess using the method from Mescheder~et~al.~\cite{mescheder2019occupancy}, then report on the first 20 shapes of the test set for each shape class. The preprocessing extracts ground truth surface points from the shapes of ShapeNet v1 \cite{chang2015shapenet}, and extracts random samples within the space with their labelled occupancy values. The evaluation method uses the ground truth points to calculate squared Chamfer distance, and uses the labelled random samples to calculate IoU. Note that the initial ShapeNet data has inconsistent normal orientation and many non-manifold surfaces due to its nature as CAD models, and this preprocessing helps orient the normals and remove most of the non-manifold surfaces.

For our method, given the input point cloud from the preprocessing, we continue as we did for SRB (\secref{sec:appx:srb-details}). We report the values for SPSR~\cite{kazhdan2013screened_poisson}, IGR~\cite{gropp2020igr}, SIREN~\cite{sitzmann2020siren}, FFN~\cite{tancik2020fourfeat} and NSP~\cite{williams2020neural_splines} from Williams~et~al.~\cite{williams2020neural_splines}. We report the results for SIREN~wo~n~\cite{sitzmann2020siren} and SAL~\cite{atzmon2020sal} using their code. 

\textbf{Additional quantitative results. }
We give the breakdown of squared Chamfer distance and IoU per shape class in \tabref{tab:appx:shapenet-extended}. As with the summary over all shape classes, for the individual breakdowns we can see that with squared Chamfer distance we get better means and medians without normals, and with normals we get better medians but often not better means. A particular class to note is loudspeaker, it is the only class that we do not get better medians for when comparing with normals, and we can see that it does much worse with normals than without. As stated in the main paper, we find that this is due to there being significant internal ghost geometry due to trying to match the internal parts of the loudspeakers, and having normal vector supervision causes even more of such ghost geometry. For IoU we do similarly: we get better means means and medians without normals, but with normals we get better medians but sometimes not better means. 

\textbf{Additional qualitative results. }
We show additional visualisations in \figref{fig:appx:shapenet-vis} and \figref{fig:appx:shapenet-vis-pt2}. A comparison between the ground truth mesh, DiGS, SIREN wo n, SAL and NSP is done for a shape from each shape class. SIREN wo n and SAL, which do not use normal information, are able to reconstruct the shapes, but the former has a lot of ghost geometry while the latter is overly smoothened or missing thin surfaces. On the other hand, DiGS is able to perform on par with the best method with normal information, NSP, where failure cases are usually extra thin surfaces.

%%%%%%%%%%%%%%% Scene Recon %%%%%%%%%%%%%%%%

\subsubsection{Scene Reconstruction}

In \secref{sec:results:surface_recon} we presented qualitative results for scene reconstruction. Attached to this supplemental material, we provide a low resolution video depicting the scene from multiple angles. A high resolution video is available an external \href{https://drive.google.com/file/d/1CuP8KN93JpzWY-597g0945r5VCIEOi3D/view?usp=sharing}{URL}. 
For this task we use eight layers with 512 units and train on the scene from Sitzmann~\etal.~\cite{sitzmann2020siren} which includes 10M oriented points. We train for $100K$ iterations, sampling random $15K$ points in each iteration. 

%%%%%%%%%%%%%%% Shape Space %%%%%%%%%%%%%%%%
% \input{assets/figures/extended-dfaust-qualitative}

\subsubsection{Shape Space Learning}

We use the setup and evaluation procedure of 
Gropp~et~al.~\cite{gropp2020igr}. The DFaust dataset~\cite{bogo2017dynamic} has scans of 10 humans, each of which do multiple action sequences and have a scan at different points during the sequence. The data set include both the scan done, which are high-resolution triangle soups, and their own registration for the complete mesh of the human which they gained from using extra data (e.g., colour) and the temporal information during the action sequence.

We use the random split setting of Gropp~et~al.~\cite{gropp2020igr}, a 75\%-25\% split between a significant subset of all scans (8566 scans). For the shapespace experiment, a single model is trained on all training scans. During training, a latent code is assigned to each action sequence of each human, which is allowed to be trained. During test time the latent code for the test shape is optimised using the input point cloud data, after which the optimised latent code is used to evaluate the full shape on a $512\times 512\times 512$ grid.

Following the evaluation procedure of 
Gropp~et~al.~\cite{gropp2020igr} we report the mean and median of the total one-sided Chamfer distances between the reconstruction and the input scans, and the reconstructions and the ground truth.

For our network we use a 8 layer SIREN with hidden layers of dimension 256. We also use 256 dimensional latent codes. At test time we optimise for the latent vectors for 800 iterations using Adam with a learning rate of $10^{-3}$. The latent code and the original input ($(x,y,z)$ coordinates) are concatenated together for a 259 dimensional input, and the latent codes are intialised to zero. As is standard in an autodecoder we use latent code regularization \cite{park2019deepsdf}, with a scaling weight of 1.0. 

For IGR we use their model (autodecoder with 8 layers, hidden dimension 512, latent dimension 256) and use the same latent optimisation procedure. For the version with normals we use their trained weights, for the version without normals we train their network using their provided training script.

\begin{figure*}
     \centering
     \resizebox{\textwidth}{!}{
     \begin{tabular}{m{0.2em} c c c c c}
     \setlength\tabcolsep{0pt}
        \rotatebox[origin=c]{90}{airplane} &\raisebox{-0.5\height}{
         \includegraphics[scale=0.5, trim=16pt 16pt 16pt 16pt, clip]{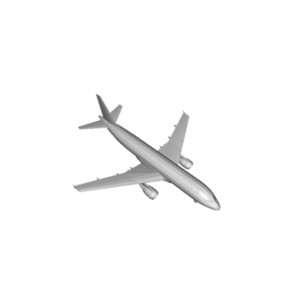}}
         &
         \raisebox{-0.5\height}{\includegraphics[scale=0.5, trim=16pt 16pt 16pt 16pt, clip]{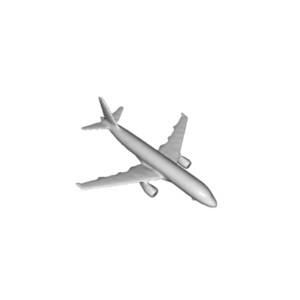}}
         &
         \raisebox{-0.5\height}{\includegraphics[scale=0.5, trim=16pt 16pt 16pt 16pt, clip]{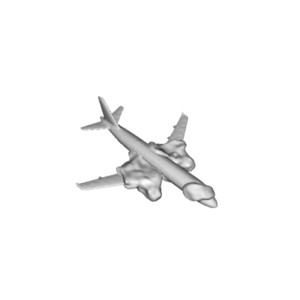}}
         &
         \raisebox{-0.5\height}{\includegraphics[scale=0.5, trim=16pt 16pt 16pt 16pt, clip]{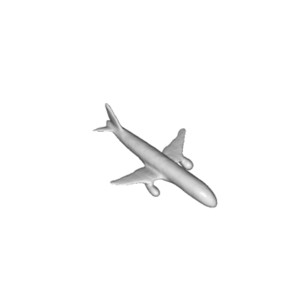}}
         &
         \raisebox{-0.5\height}{\includegraphics[scale=0.5, trim=16pt 16pt 16pt 16pt, clip]{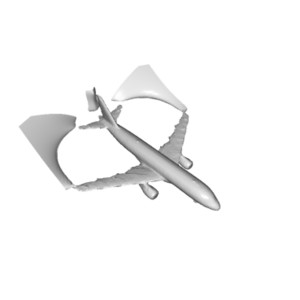}}
         \\ %%%%%%%%%%%%%%%%%%%%%%%%%%%%%
        \rotatebox[origin=c]{90}{bench} &
         \raisebox{-0.5\height}{\includegraphics[scale=0.5, trim=16pt 16pt 16pt 16pt, clip]{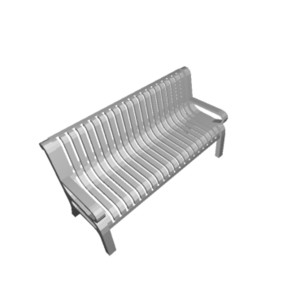}}
         &
         \raisebox{-0.5\height}{\includegraphics[scale=0.5, trim=16pt 16pt 16pt 16pt, clip]{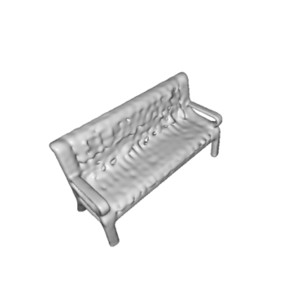}}
         &
         \raisebox{-0.5\height}{\includegraphics[scale=0.5, trim=16pt 16pt 16pt 16pt, clip]{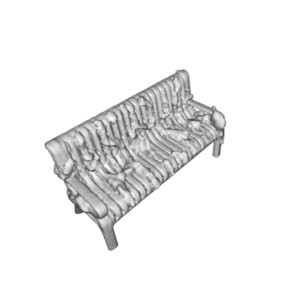}}
         &
         \raisebox{-0.5\height}{\includegraphics[scale=0.5, trim=16pt 16pt 16pt 16pt, clip]{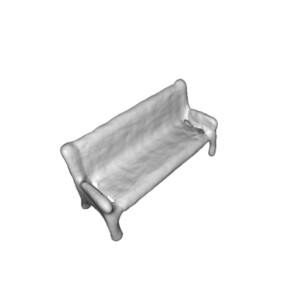}}
         &
         \raisebox{-0.5\height}{\includegraphics[scale=0.5, trim=16pt 16pt 16pt 16pt, clip]{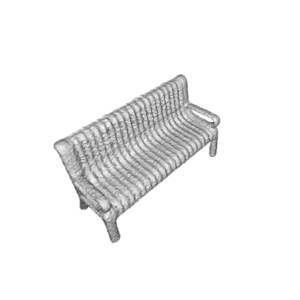}}
         \\ %%%%%%%%%%%%%%%%%%%%%%%%%%%%%
        \rotatebox[origin=c]{90}{ cablinet} &
         \raisebox{-0.5\height}{\includegraphics[scale=0.5, trim=16pt 16pt 16pt 16pt, clip]{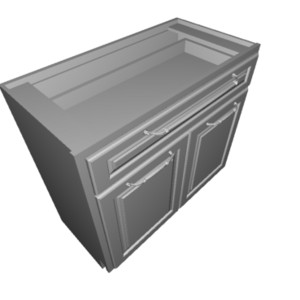}}
         &
         \raisebox{-0.5\height}{\includegraphics[scale=0.5, trim=16pt 16pt 16pt 16pt, clip]{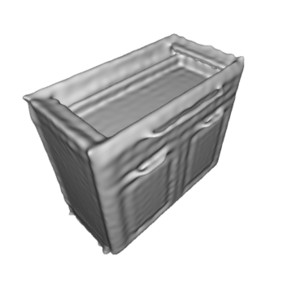}}
         &
         \raisebox{-0.5\height}{\includegraphics[scale=0.5, trim=16pt 16pt 16pt 16pt, clip]{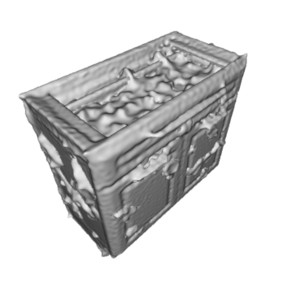}}
         &
         \raisebox{-0.5\height}{\includegraphics[scale=0.5, trim=16pt 16pt 16pt 16pt, clip]{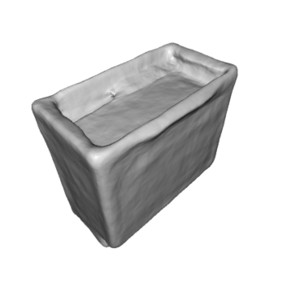}}
         &
         \raisebox{-0.5\height}{\includegraphics[scale=0.5, trim=16pt 16pt 16pt 16pt, clip]{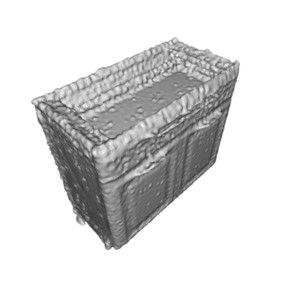}}
         \\ %%%%%%%%%%%%%%%%%%%%%%%%%%%%%
        \rotatebox[origin=c]{90}{car} &
         \raisebox{-0.5\height}{\includegraphics[scale=0.5, trim=16pt 16pt 16pt 16pt, clip]{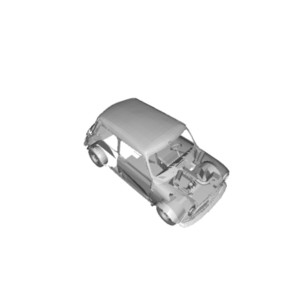}}
         &
         \raisebox{-0.5\height}{\includegraphics[scale=0.5, trim=16pt 16pt 16pt 16pt, clip]{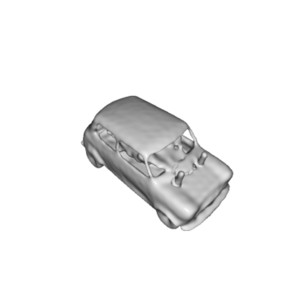}}
         &
         \raisebox{-0.5\height}{\includegraphics[scale=0.5, trim=16pt 16pt 16pt 16pt, clip]{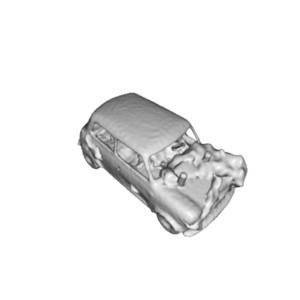}}
         &
         \raisebox{-0.5\height}{\includegraphics[scale=0.5, trim=16pt 16pt 16pt 16pt, clip]{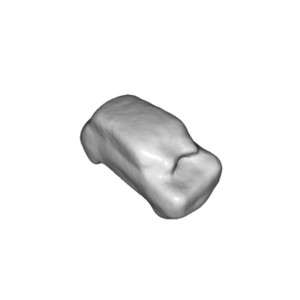}}
         &
         \raisebox{-0.5\height}{\includegraphics[scale=0.5, trim=16pt 16pt 16pt 16pt, clip]{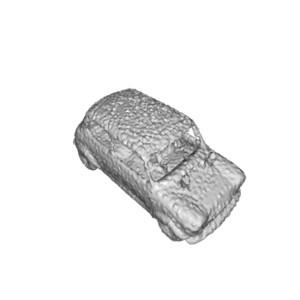}}
         \\ %%%%%%%%%%%%%%%%%%%%%%%%%%%%%
        \rotatebox[origin=c]{90}{chair} &
         \raisebox{-0.5\height}{\includegraphics[scale=0.5, trim=16pt 16pt 16pt 16pt, clip]{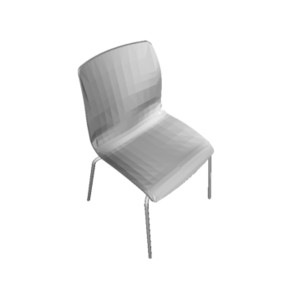}}
         &
         \raisebox{-0.5\height}{\includegraphics[scale=0.5, trim=16pt 16pt 16pt 16pt, clip]{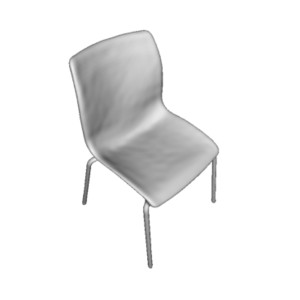}}
         &
         \raisebox{-0.5\height}{\includegraphics[scale=0.5, trim=16pt 16pt 16pt 16pt, clip]{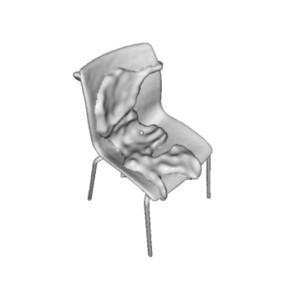}}
         &
         \raisebox{-0.5\height}{\includegraphics[scale=0.5, trim=16pt 16pt 16pt 16pt, clip]{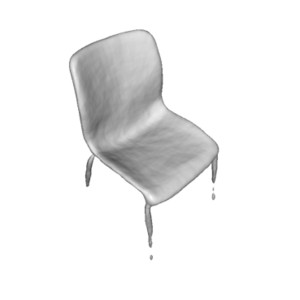}}
         &
         \raisebox{-0.5\height}{\includegraphics[scale=0.5, trim=16pt 16pt 16pt 16pt, clip]{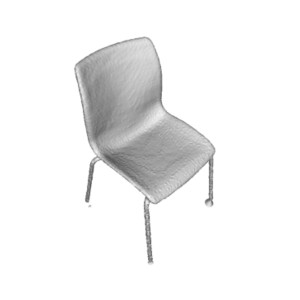}}
         \\ %%%%%%%%%%%%%%%%%%%%%%%%%%%%%
        \rotatebox[origin=c]{90}{ display} &
         \raisebox{-0.5\height}{\includegraphics[scale=0.5, trim=16pt 16pt 16pt 16pt, clip]{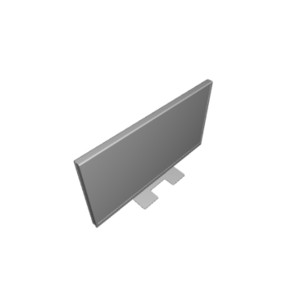}}
         &
         \raisebox{-0.5\height}{\includegraphics[scale=0.5, trim=16pt 16pt 16pt 16pt, clip]{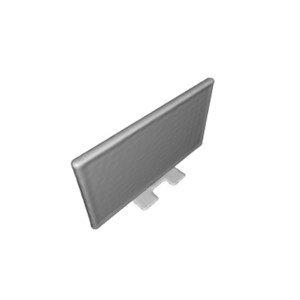}}
         &
         \raisebox{-0.5\height}{\includegraphics[scale=0.5, trim=16pt 16pt 16pt 16pt, clip]{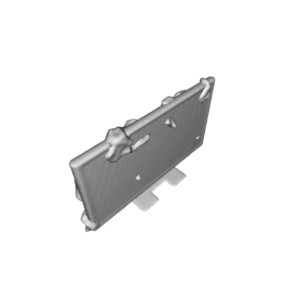}}
         &
         \raisebox{-0.5\height}{\includegraphics[scale=0.5, trim=16pt 16pt 16pt 16pt, clip]{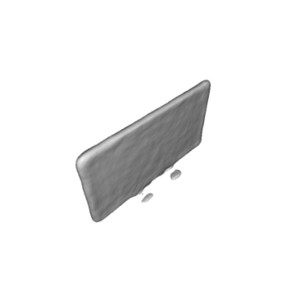}}
         &
         \raisebox{-0.5\height}{\includegraphics[scale=0.5, trim=16pt 16pt 16pt 16pt, clip]{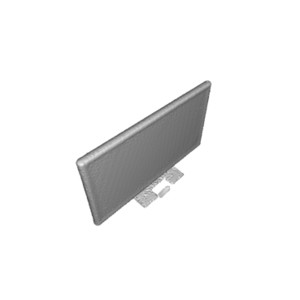}}
         \\ %%%%%%%%%%%%%%%%%%%%%%%%%%%%%
        \rotatebox[origin=c]{90}{lamp} &
         \raisebox{-0.5\height}{\includegraphics[scale=0.5, trim=16pt 16pt 16pt 16pt, clip]{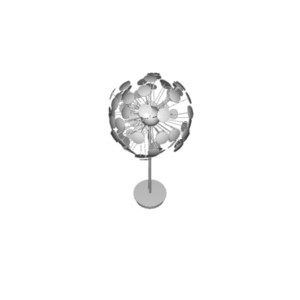}}
         &
         \raisebox{-0.5\height}{\includegraphics[scale=0.5, trim=16pt 16pt 16pt 16pt, clip]{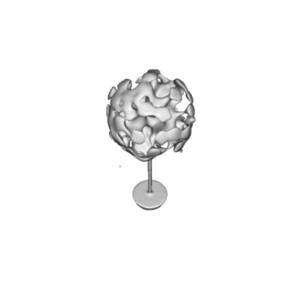}}
         &
         \raisebox{-0.5\height}{\includegraphics[scale=0.5, trim=16pt 16pt 16pt 16pt, clip]{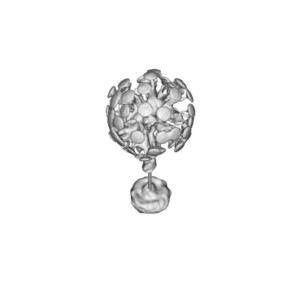}}
         &
         \raisebox{-0.5\height}{\includegraphics[scale=0.5, trim=16pt 16pt 16pt 16pt, clip]{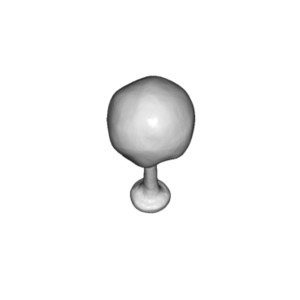}}
         &
         \raisebox{-0.5\height}{\includegraphics[scale=0.5, trim=16pt 16pt 16pt 16pt, clip]{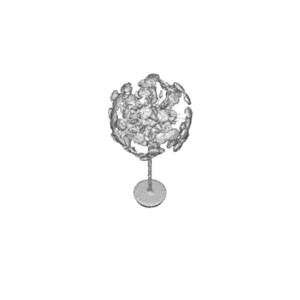}}
         \\ %%%%%%%%%%%%%%%%%%%%%%%%%%%%%
         & Ground Truth & \textbf{Our DiGS} & SIREN wo n & SAL & NSP
     \end{tabular}
     }
        \caption{Qualitative results on ShapeNet. One shape from each class is shown compared to other methods.}
        \label{fig:appx:shapenet-vis}
\end{figure*}

\begin{figure*}
     \centering
     \resizebox{\textwidth}{!}{
     \begin{tabular}{c c c c c c}
     \setlength\tabcolsep{0pt}
        \rotatebox[origin=c]{90}{loudspeaker} &
         \raisebox{-0.5\height}{\includegraphics[scale=0.5, trim=16pt 16pt 16pt 16pt, clip]{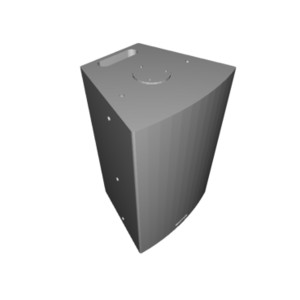}}
         &
         \raisebox{-0.5\height}{\includegraphics[scale=0.5, trim=16pt 16pt 16pt 16pt, clip]{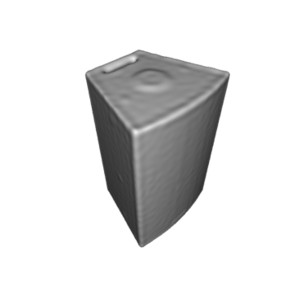}}
         &
         \raisebox{-0.5\height}{\includegraphics[scale=0.5, trim=16pt 16pt 16pt 16pt, clip]{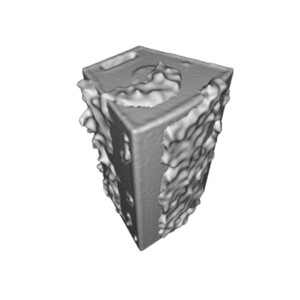}}
         &
         \raisebox{-0.5\height}{\includegraphics[scale=0.5, trim=16pt 16pt 16pt 16pt, clip]{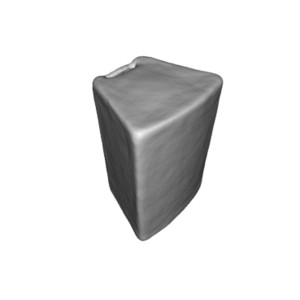}}
         &
         \raisebox{-0.5\height}{\includegraphics[scale=0.5, trim=16pt 16pt 16pt 16pt, clip]{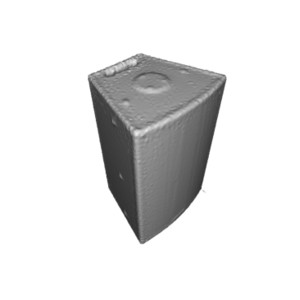}}
         \\ %%%%%%%%%%%%%%%%%%%%%%%%%%%%%
        \rotatebox[origin=c]{90}{rifle} &
         \raisebox{-0.5\height}{\includegraphics[scale=0.5, trim=16pt 16pt 16pt 16pt, clip]{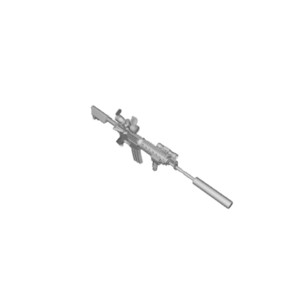}}
         &
         \raisebox{-0.5\height}{\includegraphics[scale=0.5, trim=16pt 16pt 16pt 16pt, clip]{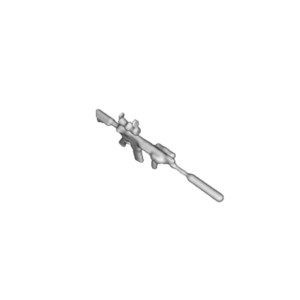}}
         &
         \raisebox{-0.5\height}{\includegraphics[scale=0.5, trim=16pt 16pt 16pt 16pt, clip]{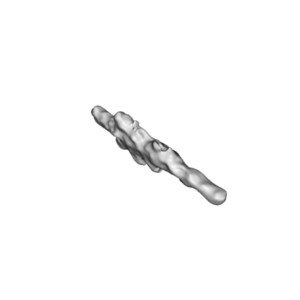}}
         &
         \raisebox{-0.5\height}{\includegraphics[scale=0.5, trim=16pt 16pt 16pt 16pt, clip]{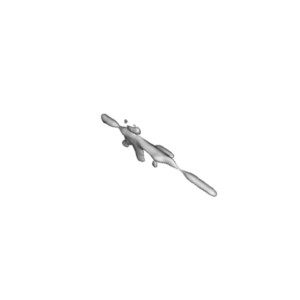}}
         &
         \raisebox{-0.5\height}{\includegraphics[scale=0.5, trim=16pt 16pt 16pt 16pt, clip]{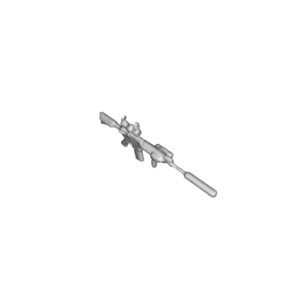}}
         \\ %%%%%%%%%%%%%%%%%%%%%%%%%%%%%
        \rotatebox[origin=c]{90}{sofa} &
         \raisebox{-0.5\height}{\includegraphics[scale=0.5, trim=16pt 16pt 16pt 16pt, clip]{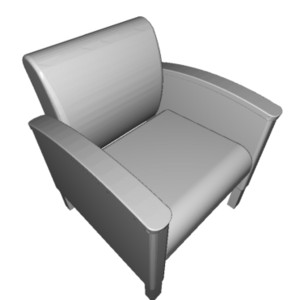}}
         &
         \raisebox{-0.5\height}{\includegraphics[scale=0.5, trim=16pt 16pt 16pt 16pt, clip]{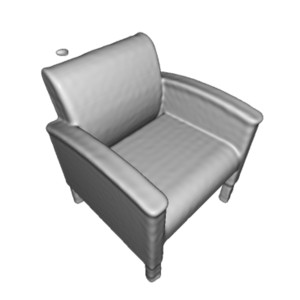}}
         &
         \raisebox{-0.5\height}{\includegraphics[scale=0.5, trim=16pt 16pt 16pt 16pt, clip]{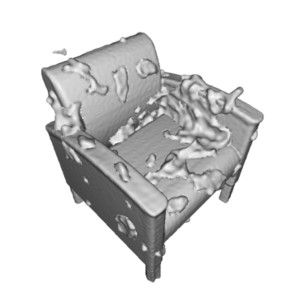}}
         &
         \raisebox{-0.5\height}{\includegraphics[scale=0.5, trim=16pt 16pt 16pt 16pt, clip]{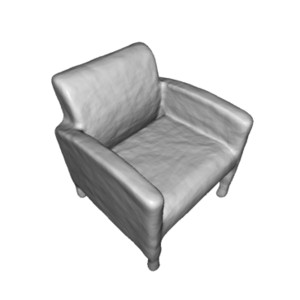}}
         &
         \raisebox{-0.5\height}{\includegraphics[scale=0.5, trim=16pt 16pt 16pt 16pt, clip]{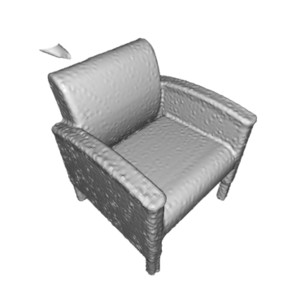}}
         \\ %%%%%%%%%%%%%%%%%%%%%%%%%%%%%
        \rotatebox[origin=c]{90}{ table} &
         \raisebox{-0.5\height}{\includegraphics[scale=0.5, trim=16pt 16pt 16pt 16pt, clip]{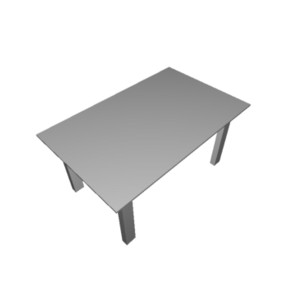}}
         &
         \raisebox{-0.5\height}{\includegraphics[scale=0.5, trim=16pt 16pt 16pt 16pt, clip]{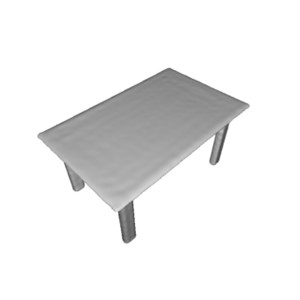}}
         &
         \raisebox{-0.5\height}{\includegraphics[scale=0.5, trim=16pt 16pt 16pt 16pt, clip]{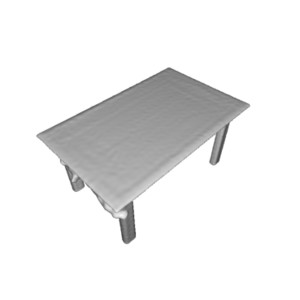}}
         &
         \raisebox{-0.5\height}{\includegraphics[scale=0.5, trim=16pt 16pt 16pt 16pt, clip]{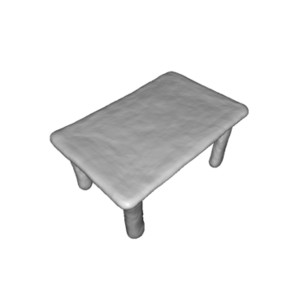}}
         &
         \raisebox{-0.5\height}{\includegraphics[scale=0.5, trim=16pt 16pt 16pt 16pt, clip]{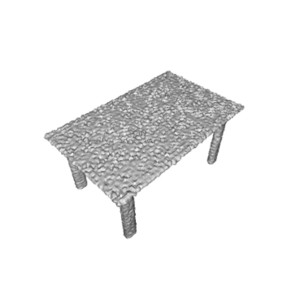}}
         \\ %%%%%%%%%%%%%%%%%%%%%%%%%%%%%
        \rotatebox[origin=c]{90}{telephone} &
         \raisebox{-0.5\height}{\includegraphics[scale=0.5, trim=16pt 16pt 16pt 16pt, clip]{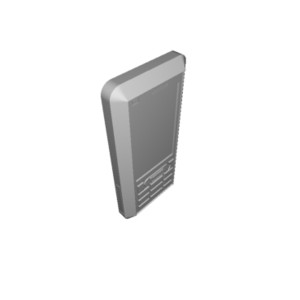}}
         &
         \raisebox{-0.5\height}{\includegraphics[scale=0.5, trim=16pt 16pt 16pt 16pt, clip]{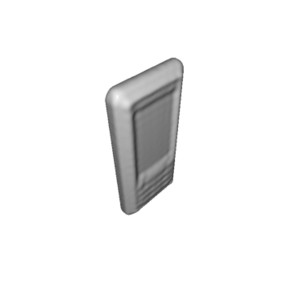}}
         &
         \raisebox{-0.5\height}{\includegraphics[scale=0.5, trim=16pt 16pt 16pt 16pt, clip]{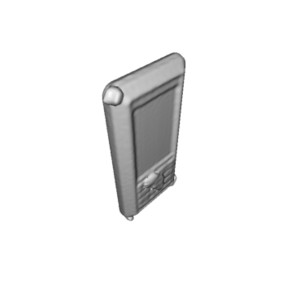}}
         &
         \raisebox{-0.5\height}{\includegraphics[scale=0.5, trim=16pt 16pt 16pt 16pt, clip]{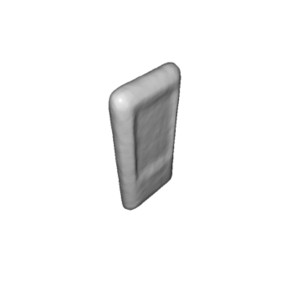}}
         &
         \raisebox{-0.5\height}{\includegraphics[scale=0.5, trim=16pt 16pt 16pt 16pt, clip]{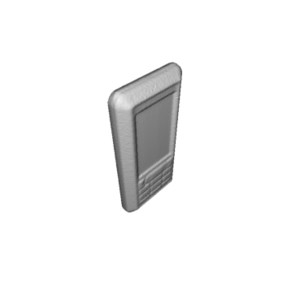}}
         \\ %%%%%%%%%%%%%%%%%%%%%%%%%%%%%
        \rotatebox[origin=c]{90}{watercraft} &
         \raisebox{-0.5\height}{\includegraphics[scale=0.5, trim=16pt 16pt 16pt 16pt, clip]{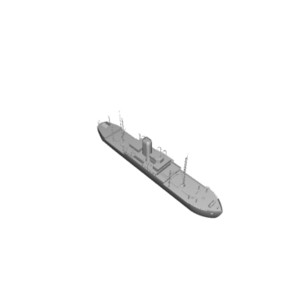}}
         &
         \raisebox{-0.5\height}{\includegraphics[scale=0.5, trim=16pt 16pt 16pt 16pt, clip]{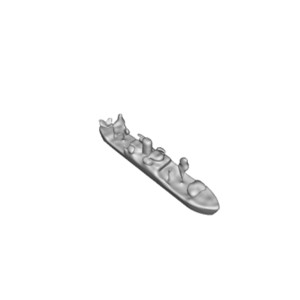}}
         &
         \raisebox{-0.5\height}{\includegraphics[scale=0.5, trim=16pt 16pt 16pt 16pt, clip]{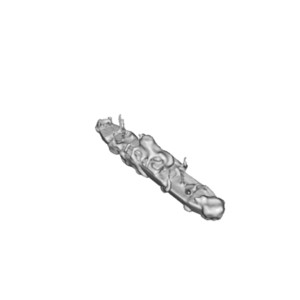}}
         &
         \raisebox{-0.5\height}{\includegraphics[scale=0.5, trim=16pt 16pt 16pt 16pt, clip]{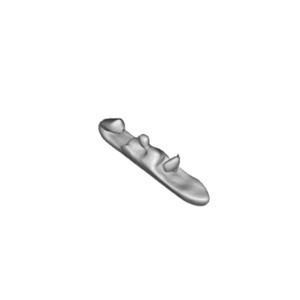}}
         &
         \raisebox{-0.5\height}{\includegraphics[scale=0.5, trim=16pt 16pt 16pt 16pt, clip]{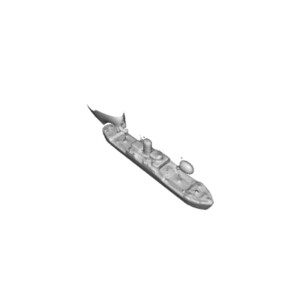}}
         \\ %%%%%%%%%%%%%%%%%%%%%%%%%%%%%
         & Ground Truth & \textbf{Our DiGS} & SIREN wo n & SAL & NSP
     \end{tabular}
     }
        \caption{Qualitative results on ShapeNet. One shape from each class is shown compared to other methods.}
        \label{fig:appx:shapenet-vis-pt2}
\end{figure*}
\begin{table*}[]\scriptsize
    \centering
    % \rotatebox{-90}{
    % \resizebox{0.9\textheight}{!}{
    \rotatebox{0}{
    \resizebox{\textwidth}{!}{
    \begin{tabular}{c}
    Squared Chamfer \\
    \begin{tabular}{l c c c | c c c | c c c | c c c | c c c}
        \toprule
         & \multicolumn{3}{c}{All}
            & \multicolumn{3}{c}{airplane} & \multicolumn{3}{c}{bench} 
            & \multicolumn{3}{c}{cabinet} & \multicolumn{3}{c}{car}\\
        Method & Mean & Median & Std
            & Mean & Median & Std & Mean & Median & Std 
            & Mean & Median & Std & Mean & Median & Std \\
            \hline
            IGR     
                    & 6.66e-4 & 1.07e-4 & 4.69e-3
                    & 3.04e-4 & 1.74e-4 & 3.47e-4 
                    & 4.48e-4 & 2.58e-4 & 4.33e-4 
                    & 1.56e-4 & 9.39e-5 & 1.23e-4 
                    & 2.60e-4 & 2.82e-4 & 9.80e-5 
                    \\ 
            SIREN   
                    & 1.03e-4 & 5.28e-5 & 1.93e-4
                    & 4.15e-5 & 3.87e-5 & 8.57e-6
                    & 9.63e-5 & 8.12e-5 & 5.41e-5
                    & 1.51e-4 & 6.69e-5 & 1.77e-4
                    & 1.39e-4 & 9.07e-5 & 1.03e-4
                    \\
            NSP
                    & \textbf{5.36e-5} & 4.06e-5 & 3.64e-5
                    & 3.55e-5 & 3.44e-5 & 2.45e-6
                    & 5.66e-5 & 4.82e-5 & 2.09e-5
                    & \textbf{6.98e-5} & 4.69e-5 & 4.34e-5
                    & \textbf{8.21e-5} & 7.18e-5 & 3.60e-5
                    \\
            DiGS + n
                    & 2.74e-4 & \textbf{2.32e-5} & 9.90e-4
                    & \textbf{1.05e-5} & \textbf{9.29e-6} & 3.93e-6
                    & \textbf{3.11e-5} & \textbf{2.17e-5} & 3.98e-5
                    & 6.92e-4 & \textbf{4.28e-5} & 1.10e-3
                    & 3.96e-4 & \textbf{3.87e-5} & 1.52e-3
                    \\

            \hline
            SIREN wo n
                    & 3.08e-4 & 2.58e-4 & 3.26e-4
                    & 2.42e-4 & 2.50e-4 & 5.92e-5
                    & 1.93e-4 & 1.67e-4 & 9.09e-5
                    & 3.16e-4 & 2.72e-4 & 1.72e-4
                    & 2.67e-4 & 2.58e-4 & 4.78e-5
                    \\
            SAL
                    & 1.14e-3 & 2.11e-4 & 3.63e-3
                    & 5.98e-4 & 2.38e-4 & 9.22e-4
                    & 3.55e-4 & 1.71e-4 & 4.26e-4
                    & \textbf{2.81e-4} & 1.86e-4 & 1.81e-4
                    & 4.51e-4 & 2.74e-4 & 4.36e-4
                    \\
            Our DiGS
                    & \textbf{1.32e-4} & \textbf{2.55e-5} & 4.73e-4
                    & \textbf{1.32e-5} & \textbf{1.01e-5} & 7.56e-6
                    & \textbf{7.26e-5} & \textbf{2.21e-5} & 1.74e-4
                    & 4.07e-4 & \textbf{4.45e-5} & 9.25e-4
                    & \textbf{7.89e-5} & \textbf{3.97e-5} & 1.10e-4
                    \\

         \bottomrule
    \end{tabular}
    \\~\\
    \begin{tabular}{l c c c | c c c | c c c | c c c | c c c}
        \toprule
         & \multicolumn{3}{c}{chair} 
            & \multicolumn{3}{c}{display} & \multicolumn{3}{c}{lamp} 
            & \multicolumn{3}{c}{loudspeaker} & \multicolumn{3}{c}{rifle}  \\
        Method & Mean & Median & Std
            & Mean & Median & Std & Mean & Median & Std 
            & Mean & Median & Std & Mean & Median & Std \\
            \hline
            IGR     
                    & 9.25e-4 & 9.88e-5 & 3.11e-3 
                    & 9.99e-5 & 7.49e-5 & 8.44e-5 
                    & 1.72e-3 & 1.28e-4 & 6.24e-3 
                    & 3.77e-3 & 1.15e-4 & 1.49e-2 
                    & 9.62e-5 & 5.29e-5 & 1.25e-4 
                    \\ 
            SIREN   
                    & 1.05e-4 & 6.34e-5 & 1.18e-4
                    & 6.98e-5 & 5.68e-5 & 3.86e-5
                    & 6.26e-5 & 5.07e-5 & 3.35e-5
                    & 2.77e-4 & 6.88e-5 & 5.54e-4
                    & 3.62e-5 & 3.50e-5 & 4.03e-6
                    \\
            NSP
                    & \textbf{5.62e-5} & 4.21e-5 & 4.32e-5
                    & \textbf{4.36e-5} & 3.99e-5 & 1.28e-5
                    & 4.19e-5 & 3.91e-5 & 1.00e-5
                    & \textbf{8.41e-5} & \textbf{4.54e-5} & 7.54e-5
                    & 3.26e-5 & 3.15e-5 & 2.79e-6
                    \\
            DiGS + n
                    & 8.55e-5 & \textbf{2.43e-5} & 1.43e-4
                    & 8.67e-4 & \textbf{2.52e-5} & 2.45e-3
                    & \textbf{3.34e-5} & \textbf{1.70e-5} & 4.80e-5
                    & 1.05e-3 & 7.13e-4 & 1.14e-3
                    & \textbf{4.80e-6} & \textbf{4.73e-6} & 1.74e-6
                    \\

            \hline
            SIREN wo n
                    & \textbf{2.63e-4} & 2.60e-4 & 1.31e-04
                    & 2.49e-4 & 2.20e-4 & 8.45e-05
                    & 6.10e-4 & 3.49e-4 & 1.04e-03
                    & 3.29e-4 & 3.04e-4 & 1.31e-04
                    & 5.44e-4 & 5.56e-4 & 1.44e-04
                    \\
            SAL
                    & 1.28e-3 & 2.92e-4 & 2.05-3
                    & 2.56e-4 & 8.86e-5 & 4.99-4
                    & 5.86e-3 & 1.29e-3 & 9.35-3
                    & 4.04e-4 & 2.63e-4 & 4.50-4
                    & 2.18e-3 & 1.15e-4 & 5.17-3
                    \\
            Our DiGS
                    & 3.72e-4 & \textbf{2.73e-5} & 1.05e-3
                    & \textbf{3.16e-5} & \textbf{2.53e-5} & 2.32e-5
                    & \textbf{1.70e-4} & \textbf{2.18e-5} & 3.96e-4
                    & \textbf{1.18e-4} & \textbf{6.18e-5} & 2.15e-4
                    & \textbf{9.10e-6} & \textbf{5.26e-6} & 1.03e-5
                    \\

         \bottomrule
    \end{tabular}
    \\~\\
    \begin{tabular}{l c c c | c c c | c c c | c c c}
        \toprule
         & \multicolumn{3}{c}{sofa} & \multicolumn{3}{c}{table} 
            & \multicolumn{3}{c}{telephone} & \multicolumn{3}{c}{watercraft} \\
        Method
            & Mean & Median & Std & Mean & Median & Std 
            & Mean & Median & Std & Mean & Median & Std \\
            \hline
            IGR     
                    & 2.86e-4 & 1.02e-4 & 5.30e-4 
                    & 3.40e-4 & 1.95e-4 & 3.33e-4 
                    & 1.03e-4 & 4.43e-5 & 1.54e-4 
                    & 1.47e-4 & 1.12e-4 & 1.23e-4 
                    \\ 
            SIREN   
                    & 7.88e-5 & 6.99e-5 & 3.90e-5
                    & 1.92e-4 & 8.32e-5 & 2.32e-4
                    & 3.88e-5 & 3.58e-5 & 9.64e-6
                    & 5.57e-5 & 4.21e-5 & 2.95e-5
                    \\
            NSP
                    & \textbf{5.11e-5} & 4.80e-5 & 1.24e-5
                    & \textbf{6.60e-5} & 4.88e-5 & 4.17e-5
                    & \textbf{3.34e-5} & 3.19e-5 & 3.60e-6
                    & 4.41e-5 & 3.84e-5 & 1.42e-5
                    \\
            DiGS + n
                    & 6.83e-5 & \textbf{2.77e-5} & 9.39e-5
                    & 1.68e-4 & \textbf{3.26e-5} & 3.50e-4
                    & 1.15e-4 & \textbf{1.75e-5} & 3.05e-4
                    & \textbf{2.77e-5} & \textbf{1.57e-5} & 3.30e-5
                    \\

            \hline
            SIREN wo n
                    & 2.72e-4 & 2.66e-4 & 6.74e-05
                    & \textbf{2.29e-4} & 2.38e-4 & 8.40e-05
                    & 2.10e-4 & 1.86e-4 & 6.60e-05
                    & 2.97e-4 & 2.43e-4 & 1.26e-04
                    \\
            SAL
                    & 3.75e-4 & 1.93e-4 & 4.31-4
                    & 1.82e-3 & 5.10e-4 & 4.31-3
                    & 1.04e-4 & 6.81e-5 & 7.99-5
                    & 8.08e-4 & 2.06e-4 & 1.75-3
                    \\
            Our DiGS
                    & \textbf{5.76e-5} & \textbf{3.27e-5} & 5.39e-5
                    & 2.94e-4 & \textbf{2.98e-5} & 6.76e-4
                    & \textbf{1.77e-5} & \textbf{1.74e-5} & 4.49e-6
                    & \textbf{6.10e-5} & \textbf{2.43e-5} & 9.03e-5
                    \\

         \bottomrule
    \end{tabular}
    \\~\\
    IoU
    \\
    \begin{tabular}{l c c c | c c c | c c c | c c c | c c c}
        \toprule
         & \multicolumn{3}{c}{All}
            & \multicolumn{3}{c}{airplane} & \multicolumn{3}{c}{bench} 
            & \multicolumn{3}{c}{cabinet} & \multicolumn{3}{c}{car}\\
        Method & Mean & Median & Std
            & Mean & Median & Std & Mean & Median & Std 
            & Mean & Median & Std & Mean & Median & Std \\
            \hline
            IGR     
                    & 0.8102 & 0.8480 & 0.1519
                    & 0.7851 & 0.8193 & 0.0977 
                    & 0.5812 & 0.5923 & 0.2487 
                    & 0.8709 & 0.8857 & 0.0924 
                    & 0.8026 & 0.8664 & 0.1300 
                    \\ 
            SIREN   
                    & 0.8268 & 0.9097 & 0.2329
                    & 0.8045 & 0.9080 & 0.2696 
                    & 0.6109 & 0.7442 & 0.3258 
                    & 0.8706 & 0.9263 & 0.1621 
                    & 0.8036 & 0.9241 & 0.2753
                    \\
            NSP
                    & 0.8973 & 0.9230 & 0.0871
                    & 0.8165 & 0.8998 & 0.1551
                    & 0.7872 & 0.8370 & 0.1236
                    & \textbf{0.9274} & 0.9291 & 0.0422
                    & 0.8954 & 0.9288 & 0.0740
                    \\
            DiGS + n
                    & \textbf{0.9200} & \textbf{0.9774} & 0.1992
                    & \textbf{0.9693} & \textbf{0.9718} & 0.0151
                    & \textbf{0.9428} & \textbf{0.9655} & 0.0644
                    & 0.8323 & \textbf{0.9867} & 0.3076
                    & \textbf{0.9147} & \textbf{0.9754} & 0.2126
                    \\

            \hline
            SIREN wo n
                    & 0.3085 & 0.2952 & 0.2014
                    & 0.2248 & 0.1735 & 0.1103
                    & 0.4020 & 0.4231 & 0.1953
                    & 0.3014 & 0.2564 & 0.1275
                    & 0.3336 & 0.3030 & 0.0997
                    \\
            SAL
                    & 0.4030 & 0.3944 & 0.2722
                    & 0.1908 & 0.1693 & 0.0955
                    & 0.2260 & 0.2311 & 0.1401
                    & 0.6923 & 0.7224 & 0.1637
                    & 0.6261 & 0.6526 & 0.1525
                    \\
            Our DiGS
                    & \textbf{0.9390} & \textbf{0.9764} & 0.1262
                    & \textbf{0.9613} & \textbf{0.9577} & 0.0164
                    & \textbf{0.9061} & \textbf{0.9536} & 0.1413
                    & \textbf{0.9261} & \textbf{0.9853} & 0.2137
                    & \textbf{0.9455} & \textbf{0.9765} & 0.0699
                    \\

         \bottomrule
    \end{tabular}
    \\~\\
    \begin{tabular}{l c c c | c c c | c c c | c c c | c c c}
        \toprule
         & \multicolumn{3}{c}{chair} 
            & \multicolumn{3}{c}{display} & \multicolumn{3}{c}{lamp} 
            & \multicolumn{3}{c}{loudspeaker} & \multicolumn{3}{c}{rifle}  \\
        Method & Mean & Median & Std
            & Mean & Median & Std & Mean & Median & Std 
            & Mean & Median & Std & Mean & Median & Std \\
            \hline
            IGR     
                    & 0.8049 & 0.8320 & 0.1022 
                    & 0.8741 & 0.8917 & 0.0533 
                    & 0.7865 & 0.8259 & 0.1318 
                    & 0.8867 & 0.9324 & 0.1017 
                    & 0.8279 & 0.8267 & 0.0542  
                    \\ 
            SIREN   
                    & 0.8721 & 0.8807 & 0.0495 
                    & 0.9014 & 0.9146 & 0.0440 
                    & 0.8392 & 0.8995 & 0.2025 
                    & 0.8458 & 0.9618 & 0.2404 
                    & 0.7329 & 0.9132 & 0.3662 
                    \\
            NSP
                    & 0.8841 & 0.9034 & 0.0825
                    & \textbf{0.9309} & 0.9316 & 0.0251
                    & \textbf{0.9037} & 0.9178 & 0.0512
                    & \textbf{0.9323} & 0.9627 & 0.0599
                    & 0.9299 & 0.9313 & 0.0215
                    \\
            DiGS + n
                    & \textbf{0.9719} & \textbf{0.9759} & 0.0140
                    & 0.8367 & \textbf{0.9855} & 0.3485
                    & 0.9024 & \textbf{0.9637} & 0.1991
                    & 0.8798 & \textbf{0.9747} & 0.2424
                    & \textbf{0.9569} & \textbf{0.9571} & 0.0207
                    \\

            \hline
            SIREN wo n
                    & 0.4208 & 0.3748 & 0.2322
                    & 0.3566 & 0.3123 & 0.1790
                    & 0.3055 & 0.2573 & 0.2598
                    & 0.2229 & 0.1724 & 0.1575
                    & 0.0265 & 0.0092 & 0.0554
                    \\
            SAL
                    & 0.2589 & 0.1491 & 0.2213
                    & 0.5067 & 0.5801 & 0.2474
                    & 0.1689 & 0.0698 & 0.1994
                    & 0.6702 & 0.7264 & 0.1976
                    & 0.2835 & 0.2821 & 0.1530
                    \\
            Our DiGS
                    & \textbf{0.9082} & \textbf{0.9650} & 0.1523
                    & \textbf{0.9839} & \textbf{0.9886} & 0.0102
                    & \textbf{0.8776} & \textbf{0.9646} & 0.1943
                    & \textbf{0.9632} & \textbf{0.9851} & 0.0978
                    & \textbf{0.9486} & \textbf{0.9567} & 0.0281
                    \\

         \bottomrule
    \end{tabular}
    \\~\\
    \begin{tabular}{l c c c | c c c | c c c | c c c}
        \toprule
         & \multicolumn{3}{c}{sofa} & \multicolumn{3}{c}{table} 
            & \multicolumn{3}{c}{telephone} & \multicolumn{3}{c}{watercraft} \\
        Method
            & Mean & Median & Std & Mean & Median & Std 
            & Mean & Median & Std & Mean & Median & Std \\
            \hline
            IGR     
                    & 0.8891 & 0.9139 & 0.0708 
                    & 0.6852 & 0.7260 & 0.2004 
                    & 0.9148 & 0.9372 & 0.0639 
                    & 0.8146 & 0.8445 & 0.0931 
                    \\ 
            SIREN   
                    & 0.9251 & 0.9411 & 0.0390 
                    & 0.7280 & 0.8058 & 0.2089 
                    & 0.9427 & 0.9514 & 0.0310 
                    & 0.8722 & 0.9279 & 0.1990 
                    \\
            NSP
                    & 0.9387 & 0.9473 & 0.0264
                    & 0.8414 & 0.8427 & 0.0534
                    & \textbf{0.9569} & 0.9625 & 0.0260
                    & 0.9207 & 0.9231 & 0.0402
                    \\
            DiGS + n
                    & \textbf{0.9624} & \textbf{0.9859} & 0.0696
                    & \textbf{0.9284} & \textbf{0.9784} & 0.1743
                    & 0.8880 & \textbf{0.9855} & 0.2935
                    & \textbf{0.9747} & \textbf{0.9789} & 0.0168
                    \\

            \hline
            SIREN wo n
                    & 0.3397 & 0.3444 & 0.1206
                    & 0.3797 & 0.3603 & 0.1528
                    & 0.3778 & 0.3806 & 0.2590
                    & 0.3190 & 0.3007 & 0.1877
                    \\
            SAL
                    & 0.4844 & 0.4530 & 0.1404
                    & 0.0965 & 0.0320 & 0.1502
                    & 0.6025 & 0.6704 & 0.2203
                    & 0.4170 & 0.4728 & 0.2367
                    \\
            Our DiGS
                    & \textbf{0.9572} & \textbf{0.9807} & 0.0896
                    & \textbf{0.8943} & \textbf{0.9720} & 0.1996
                    & \textbf{0.9854} & \textbf{0.9876} & 0.0071
                    & \textbf{0.9522} & \textbf{0.9735} & 0.0504
                    \\

         \bottomrule
    \end{tabular}
    
    \end{tabular}
    }
    }
    \caption{Extended results for surface reconstruction on ShapeNet~\cite{chang2015shapenet}. For each shape class, and all shapes together, we report the squared Chamfer distance (first three tables) and the IoU (last three tables) to the ground truth mesh. Methods above the line use normal supervison, and methods below do not.}
    \label{tab:appx:shapenet-extended}
\end{table*}

\begin{figure*}
     \centering
     \resizebox{\textwidth}{!}{
     \begin{tabular}{c c c c c c c c}
         \rotatebox[origin=c]{90}{GT}
         &
         \raisebox{-.5\height}{\includegraphics[width=0.1\linewidth, trim=94pt 50pt 62pt 41pt, clip]{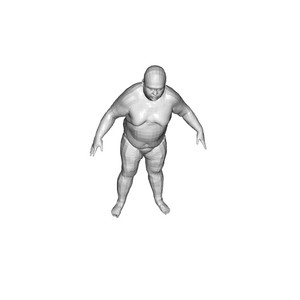} }
         &
         \raisebox{-.5\height}{\includegraphics[width=0.1\linewidth, trim=94pt 50pt 62pt 41pt, clip]{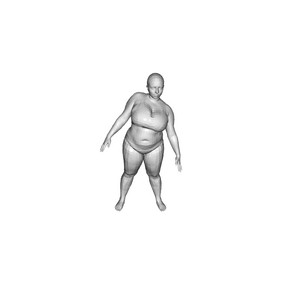} }
         &
         \raisebox{-.5\height}{\includegraphics[width=0.1\linewidth, trim=94pt 50pt 62pt 41pt, clip]{assets/figures/dfaust-shapespace/gt_3.jpg} }
         &
         \raisebox{-.5\height}{\includegraphics[width=0.1\linewidth, trim=94pt 50pt 62pt 41pt, clip]{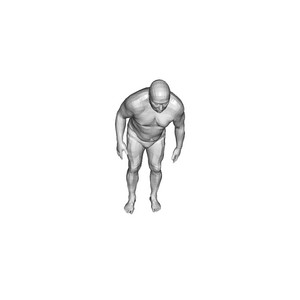} }
         &
         \raisebox{-.5\height}{\includegraphics[width=0.1\linewidth, trim=94pt 50pt 62pt 41pt, clip]{assets/figures/dfaust-shapespace/gt_5.jpg} }
         &
         \raisebox{-.5\height}{\includegraphics[width=0.1\linewidth, trim=94pt 50pt 62pt 41pt, clip]{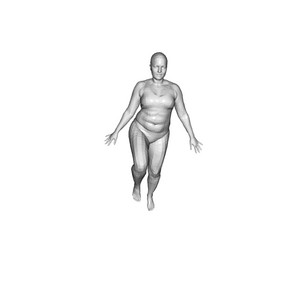} }
         &
         \raisebox{-.5\height}{\includegraphics[width=0.1\linewidth, trim=94pt 50pt 62pt 41pt, clip]{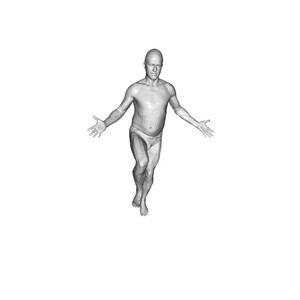} }\\ \hline
         \rule{0pt}{0.05ex} &&& \\  
         \rotatebox[origin=c]{90}{IGR}
         &
         \raisebox{-.5\height}{\includegraphics[width=0.1\linewidth, trim=94pt 50pt 62pt 41pt, clip]{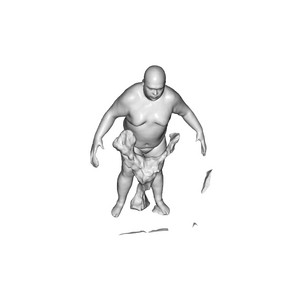} }
         &
         \raisebox{-.5\height}{\includegraphics[width=0.1\linewidth, trim=94pt 50pt 62pt 41pt, clip]{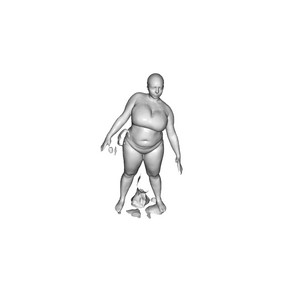} }
         &
         \raisebox{-.5\height}{\includegraphics[width=0.1\linewidth, trim=94pt 50pt 62pt 41pt, clip]{assets/figures/dfaust-shapespace/igr_3.jpg} }
         &
         \raisebox{-.5\height}{\includegraphics[width=0.1\linewidth, trim=94pt 50pt 62pt 41pt, clip]{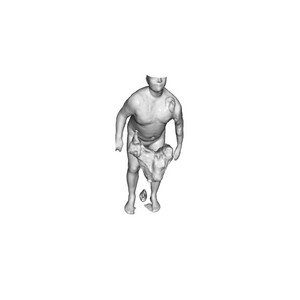} }
         &
         \raisebox{-.5\height}{\includegraphics[width=0.1\linewidth, trim=94pt 50pt 62pt 41pt, clip]{assets/figures/dfaust-shapespace/igr_5.jpg} }
         &
         \raisebox{-.5\height}{\includegraphics[width=0.1\linewidth, trim=94pt 50pt 62pt 41pt, clip]{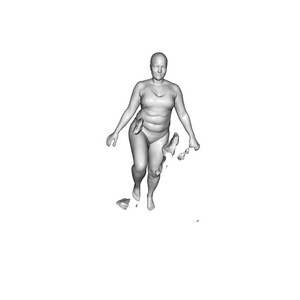} }
         &
         \raisebox{-.5\height}{\includegraphics[width=0.1\linewidth, trim=94pt 50pt 62pt 41pt, clip]{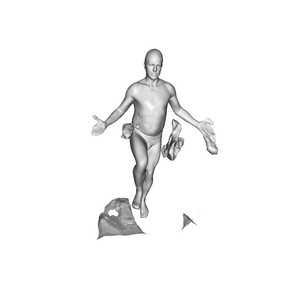} }\\
         \rule{0pt}{0.05ex} &&& \\  
         \rotatebox[origin=c]{90}{\textbf{Our DiGS}+n}
         &
         \raisebox{-.5\height}{\includegraphics[width=0.1\linewidth, trim=94pt 50pt 62pt 41pt, clip]{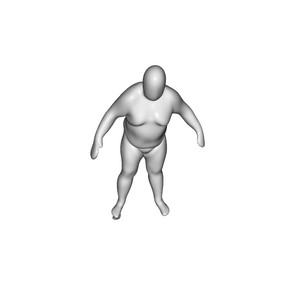} }
         &
         \raisebox{-.5\height}{\includegraphics[width=0.1\linewidth, trim=94pt 50pt 62pt 41pt, clip]{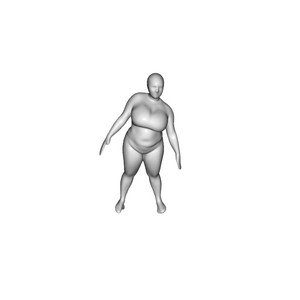} }
         &
         \raisebox{-.5\height}{\includegraphics[width=0.1\linewidth, trim=94pt 50pt 62pt 41pt, clip]{assets/figures/dfaust-shapespace/digs_n_3.jpg} }
         &
         \raisebox{-.5\height}{\includegraphics[width=0.1\linewidth, trim=94pt 50pt 62pt 41pt, clip]{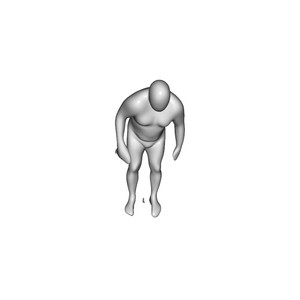} }
         &
         \raisebox{-.5\height}{\includegraphics[width=0.1\linewidth, trim=94pt 50pt 62pt 41pt, clip]{assets/figures/dfaust-shapespace/digs_n_5.jpg} }
         &
         \raisebox{-.5\height}{\includegraphics[width=0.1\linewidth, trim=94pt 50pt 62pt 41pt, clip]{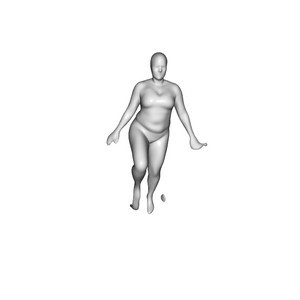} }
         &
         \raisebox{-.5\height}{\includegraphics[width=0.1\linewidth, trim=94pt 50pt 62pt 41pt, clip]{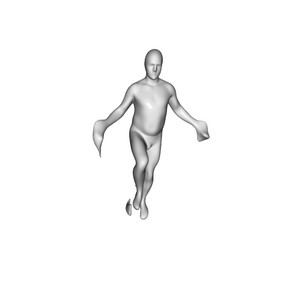} }\\ \hline   
         \rule{0pt}{0.05ex} &&& \\
    
         \rotatebox[origin=c]{90}{IGR wo n}
         &
         \raisebox{-.5\height}{\includegraphics[width=0.1\linewidth, trim=94pt 50pt 62pt 41pt, clip]{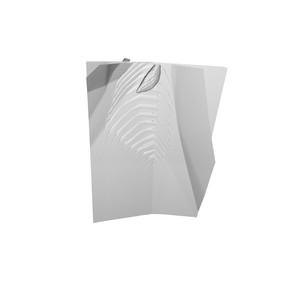}}
         &
         \raisebox{-.5\height}{\includegraphics[width=0.1\linewidth, trim=94pt 50pt 62pt 41pt, clip]{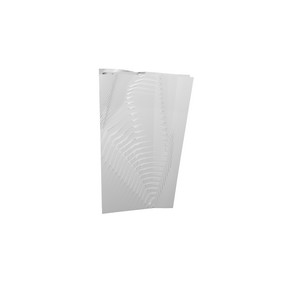}}
         &
         \raisebox{-.5\height}{\includegraphics[width=0.1\linewidth, trim=94pt 50pt 62pt 41pt, clip]{assets/figures/dfaust-shapespace/igr_wo_n_3.jpg}}
         &
         \raisebox{-.5\height}{\includegraphics[width=0.1\linewidth, trim=94pt 50pt 62pt 41pt, clip]{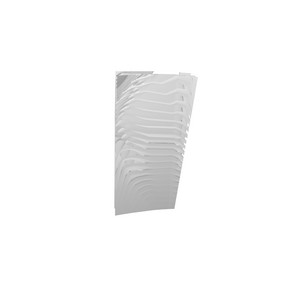}}
         &
         \raisebox{-.5\height}{\includegraphics[width=0.1\linewidth, trim=94pt 50pt 62pt 41pt, clip]{assets/figures/dfaust-shapespace/igr_wo_n_5.jpg}}
         &
         \raisebox{-.5\height}{\includegraphics[width=0.1\linewidth, trim=94pt 50pt 62pt 41pt, clip]{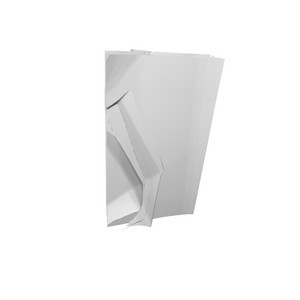}}
         &
         \raisebox{-.5\height}{\includegraphics[width=0.1\linewidth, trim=94pt 50pt 62pt 41pt, clip]{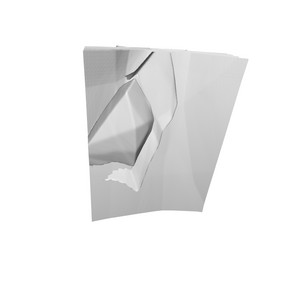}}\\
         \rule{0pt}{0.05ex} &&& \\
         
         \rotatebox[origin=c]{90}{\textbf{Our DiGS}}
         &
         \raisebox{-.5\height}{\includegraphics[width=0.1\linewidth, trim=94pt 50pt 62pt 41pt, clip]{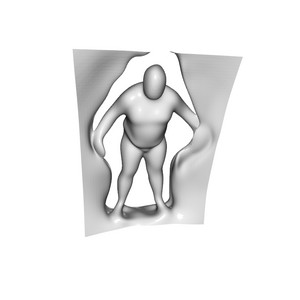}}
         &
         \raisebox{-.5\height}{\includegraphics[width=0.1\linewidth, trim=94pt 50pt 62pt 41pt, clip]{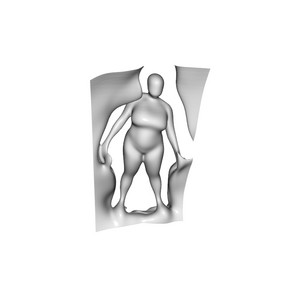}}
         &
         \raisebox{-.5\height}{\includegraphics[width=0.1\linewidth, trim=94pt 50pt 62pt 41pt, clip]{assets/figures/dfaust-shapespace/digs_3.jpg}}
         &
         \raisebox{-.5\height}{\includegraphics[width=0.1\linewidth, trim=94pt 50pt 62pt 41pt, clip]{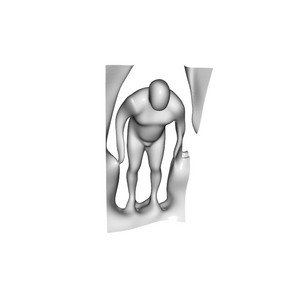}}
         &
         \raisebox{-.5\height}{\includegraphics[width=0.1\linewidth, trim=94pt 50pt 62pt 41pt, clip]{assets/figures/dfaust-shapespace/digs_5.jpg}}
         &
         \raisebox{-.5\height}{\includegraphics[width=0.1\linewidth, trim=94pt 50pt 62pt 41pt, clip]{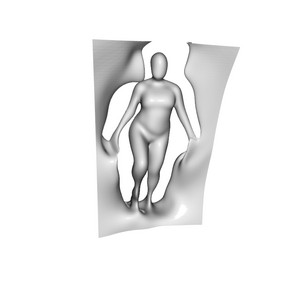}}
         &
         \raisebox{-.5\height}{\includegraphics[width=0.1\linewidth, trim=94pt 50pt 62pt 41pt, clip]{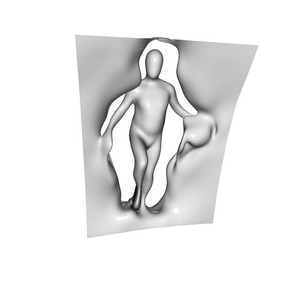}}\\
        Col & 1 & 2 & 3 & 4 & 5 & 6 & 7
     \end{tabular}}
        \caption{Qualitative results for the shape space experiment on the DFaust dataset \cite{bogo2017dynamic}. Both methods with normals, IGR and DiGS+n, manage to capture the shape of the humans. IGR has more detail, but has a lot of ghost geometries (all columns), sometimes changes orientation that is far away from the mean of the dataset (column 4) and often misses large surfaces such as forearms (columns 3 and 5). DiGS+n captures the correct surface with minimal ghost geometry and few missing regions, but oversmooths fine details (e.g. facial features). On the other hand, for methods without normals (DiGS and IGR wo n), only DiGS is able to learn multiple human shapes whereas IGR wo n is not able to learn at all. DiGS also has large ghost geometries and oversmooths the human surfaces, but manages to capture more key regions compared to IGR. }
        \label{fig:dfaust_shapespace_comparisons}
\end{figure*}

\end{document}